%% file: main.tex
\PassOptionsToPackage{table}{xcolor}
\documentclass{article}

\usepackage{iclr2027_conference,times}
\iclrfinalcopy
\usepackage[utf8]{inputenc} 
\usepackage[T1]{fontenc}    
\usepackage{hyperref}       
\usepackage{url}            
\hypersetup{
    hidelinks,
    pdftitle={ZO-COSMO: Index-Free One-Hop Mixing for Decentralized Zeroth-Order Optimization},
    pdfsubject={Decentralized zeroth-order optimization with index-free communication},
    pdfkeywords={zeroth-order optimization, decentralized optimization, communication efficiency, sparse consensus},
    pdfauthor={Shengjun Zhang, Tingyi Liu, Heng Zhang, Dong Xie}
}
\usepackage{booktabs}       
\usepackage{amsfonts}       
\usepackage{nicefrac}       
\usepackage{microtype}      
\usepackage{graphicx}
\usepackage{placeins}
\graphicspath{{./figures/}}
\usepackage{algorithm}
\usepackage{algpseudocode}
\usepackage{subcaption}
\usepackage{adjustbox}
\usepackage{booktabs}
\usepackage{multirow}
\usepackage{makecell}
\usepackage{xcolor}
\usepackage{colortbl}
\usepackage[cmex10]{amsmath}
\usepackage{amssymb}
\usepackage{amsthm}

\usepackage{tikz}
\usepackage{ifthen}
\usetikzlibrary{calc, positioning, fit, arrows.meta, backgrounds, shapes.geometric, decorations.pathreplacing}

\DeclareMathOperator{\diag}{diag}

\newcommand{\Lift}{\mathsf{Lift}}
\newcommand{\Res}{\mathsf{Res}}
\newcommand{\corrnorm}[1]{\left\|#1\right\|}

\newtheorem{theorem}{Theorem}
\newtheorem{assumption}{Assumption}
\newtheorem{proposition}{Proposition}
\newtheorem{lemma}{Lemma}
\newtheorem{definition}{Definition}
\newtheorem{remark}{Remark}

\definecolor{rqgray}{gray}{0.94}
\definecolor{rqborder}{gray}{0.55}
\newcommand{\researchquestionbox}[1]{%
\begin{center}
\setlength{\fboxsep}{7pt}%
\fcolorbox{rqborder}{rqgray}{%
\begin{minipage}{0.88\linewidth}
\centering
\textbf{\emph{#1}}
\end{minipage}}%
\end{center}}

\def\R{\mathbb{R}}

\title{ZO-COSMO: Index-Free One-Hop Mixing\\for Decentralized Zeroth-Order Optimization}

\newcommand{\arxivauthorblock}{%
\begin{minipage}{0.96\textwidth}
\centering\normalfont
\textbf{Shengjun Zhang}\textsuperscript{1,*}\quad
\textbf{Tingyi Liu}\textsuperscript{2,*}\quad
\textbf{Heng Zhang}\textsuperscript{3,\ensuremath{\dagger}}\quad
\textbf{Dong Xie}\textsuperscript{4,\ensuremath{\dagger}}\par
\vspace{5pt}
{\small
\textsuperscript{1}School of Artificial Intelligence, Hubei University, Wuhan, China\par
\textsuperscript{2}School of Economics and Management, Wuhan University, Wuhan, China\par
\textsuperscript{3}School of Electrical Engineering, Shanghai Jiao Tong University, Shanghai, China\par
\textsuperscript{4}Baidu, Beijing, China\par
\vspace{4pt}
\href{mailto:sj.zhang@hubu.edu.cn}{\texttt{sj.zhang@hubu.edu.cn}}\quad
\href{mailto:tingyiL@whu.edu.cn}{\texttt{tingyiL@whu.edu.cn}}\par
\href{mailto:zhangheng_sjtu@sjtu.edu.cn}{\texttt{zhangheng\_sjtu@sjtu.edu.cn}}\quad
\href{mailto:xiedong04@baidu.com}{\texttt{xiedong04@baidu.com}}\par
\vspace{4pt}
\textsuperscript{*}Equal contribution.\qquad
\textsuperscript{\ensuremath{\dagger}}Corresponding authors.\par
}
\end{minipage}
}
\author{\arxivauthorblock}

\begin{document}

\maketitle
\lhead{Preprint}

\begin{abstract}
Sparse communication in decentralized zeroth-order learning requires compatible peer-state coordinates. We characterize this one-hop condition and develop \textsf{ZO-COSMO}, coupling two-query estimation with average-preserving masked consensus using $q$ values per active link. Global supports serve all-neighbor mixing; matching updates require agreement only within each pair. We derive a sharp contraction-per-scalar bound within the matching class and convergence guarantees for the core and sparse-momentum updates. At fixed matching, exact moment identities characterize how shared directions preserve gradient-heterogeneity cancellation and redistribute estimation error and disagreement. Mechanism experiments cover unequal curvatures, noise, and sparse momentum. Further tests span $64$ synthetic agents and eight logical Qwen LoRA workers. At matched payload budgets, Qwen2-7B QNLI gains $3.65$ accuracy points over explicit-index Rand-$k$; edge-local updates gain $3.42$ and $2.53$ points over all-neighbor mixing on eight-worker complete and ring graphs. A matched-first-step ablation gives a $3.92$-point momentum benefit. Seed-aware and same-matching controls distinguish encoding, scheduling, and query correlation.
\end{abstract}

\section{Introduction}\label{sec:intro}

In decentralized black-box learning, agents can query local function values but cannot centralize data or gradients. Let $x\in\mathbb{R}^d$ denote the shared decision vector. We study
\begin{equation}\label{nonconvex:eqn:xopt}
    \min_{x\in\mathbb{R}^d} F(x),
    \qquad
    F(x)=\frac{1}{N}\sum_{i=1}^N F_i(x)
\end{equation}
over a graph, where each $F_i$ is a smooth nonconvex black box. Zeroth-order methods replace gradients with finite differences \citep{ghadimi2013stochastic,nesterov2017random,shamir2017optimal}, but high dimension makes both estimation and communication expensive.

Sparse communication helps only if the support is also cheap. A coordinate-list encoding sends $qB_{\mathrm{val}}+q\lceil\log_2 d\rceil$ bits. At billion-parameter scale, each index is about as wide as a single-precision value, nearly doubling the value-only payload.

We ask:

\researchquestionbox{Can a decentralized zeroth-order method communicate sparse updates without communicating sparse indices?}

Public seeds already reconstruct ZO updates at a server \citep{li2025decomfl,cai2026feedsign,ran2026meerkat} or through graph-wide dissemination \citep{kim2026seedflood}. One-hop mixing also requires the corresponding peer-state values. A message carrying values on $\{2,3\}$ leaves the value at coordinate $1$ unavailable for an average on $\{1,2\}$, even when its support is public (Figure \ref{fig:shared_randomness_n_nodes}). Theorem \ref{thm:one_hop_compatibility} gives the necessary and sufficient support condition for mixing arbitrary current peer states without stored peer histories.

ZO-COSMO meets this condition by coordinating supports with the mixing graph. A global support permits masked Metropolis mixing; a matching permits separate supports on disjoint edges and one $q$-value message per active node. Uniform perfect matchings on a complete graph attain the best contraction per scalar within this operator class (Theorem \ref{thm:mixing_per_bit}). The query and mixing steps share a mask, so their dependence also enters the optimization analysis. We show how a shared direction preserves cancellation between local gradients and how independent directions redistribute estimation error and disagreement.

\subsection{Contributions}

\begin{itemize}
    \item \textbf{Exact index-free mixing.} We characterize the supports needed to mix arbitrary peer states from one-hop coordinate messages, with a sharp missing-coordinate error bound. This leads to global and matching-based protocols with $q$ values per active link.

    \item \textbf{Mixing and optimization mechanisms.} We derive the optimal contraction per scalar in the matching class and separate average-estimation error from disagreement injection. Pair invariance yields a heterogeneity-sensitive core bound, and exact conditional moments account for oracle noise and stored momentum; separate results establish bit-budget and sparse-momentum convergence.

    \item \textbf{Mechanism and model-scale evidence.} Fixed-matching experiments test correlation under unequal curvatures, noise, and momentum. Qwen LoRA shows fixed-budget gains over explicit-index and all-neighbor baselines and a matched-step momentum benefit. Seed-aware and same-matching controls separate these effects; synthetic tests include adaptive Top-$k$ compression and $64$ agents.
\end{itemize}

\section{Related Work}
\label{sec:related}

\paragraph{Zeroth-order and decentralized optimization.}
ZO methods make optimization possible when gradients are unavailable, using one- or two-point function queries \citep{conn2009introduction,flaxman2005online,ghadimi2013stochastic,nesterov2017random,shamir2017optimal,larson2019derivative}. In high dimension, much of the work has focused on better directions through sparsity, subspaces, momentum, and variance reduction \citep{wang2018stochastic,golovin2019gradientless,liu2018zeroth,liu2018stochastic,cai2020zeroth,kozak2020stochastic,chen2019zo}. Decentralized methods must also keep locally updated states in agreement over a graph \citep{yuan2014randomized,yuan2015gradient,sahu2018distributed,hajinezhad2019zone,tang2019distributed,zhang2021zo,veprikov2024zero}. We retain two-point estimation and study which sparse messages allow that mixing step to be executed exactly.

\paragraph{Compressed decentralized learning.}
Randomized gossip relates pairwise averaging to graph spectra \citep{boyd2006randomizedgossip}; compressed learning uses quantization, sparsification, sign compression, and error compensation \citep{alistarh2017qsgd,lin2018deep,stich2018sparsified,lan2018communication,tang2018communication,Koloskova2019decentralized,koloskova2020decentralized,pmlr-v80-bernstein18a,basu2019qsparse}. Coordinate lists add a representation cost \citep{lin2018deep,shi2020communication}. PermK analyzes how correlated compressors change aggregate estimation error \citep{szlendak2022permutation}. Our query direction also selects the peer-state coordinates that are mixed. We therefore study both aggregate error and the disagreement injected by that same update. Complementary supports may reduce aggregate variance without supplying the coordinates required for exact current-state mixing.

DeComFL and MEERKAT reconstruct seed--scalar updates at a server, while FeedSign aggregates seed--sign votes \citep{li2025decomfl,ran2026meerkat,cai2026feedsign}. SeedFlood disseminates updates across the graph and batches their application in a shared subspace \citep{kim2026seedflood}. These methods exploit compact update representations. ZO-COSMO sends $q$ current-state values to complete one neighbor-mixing stage; Appendix \ref{appendix:public_randomness_operators} compares the payload and dissemination requirements. Recent decentralized ZO methods combine estimators with general compressors \citep{hua2026compressedzo,wang2025compressedzo,chen2025compressedmomentumzo}. We include Com-DSZO with adaptive Top-$k$ and a periodic-refresh comparator motivated by MARINA \citep{gorbunov2021marina}.

\paragraph{Forward-only model adaptation.}
Language-model adaptation can replace backpropagation with loss queries \citep{malladi2023mezo,zhang2024revisiting,liu2024sparsemezo,yu2024subzero}. Sparse and low-dimensional directions reduce the parameters touched by each query, but peers must still agree on how received values update their local states. Our Qwen experiments examine this communication step with separate worker states and data shards.

\section{Algorithm}\label{sec:algorithm}

Algorithm~\ref{alg:zo_cosmo} reconstructs a public support, takes two sparse queries, and mixes selected coordinates. Messages contain $q$ ordered values; unsent coordinates retain their local values. The implementation derives $R_t=\mathcal{H}(S_{\mathrm{master}}\mathbin\Vert t)$ from synchronized seeds and round counters. Decoding is pathwise exact; stochastic rates use fresh public coins (Appendix \ref{appendix:public_seed_protocol}).

Let $X_t$ and $Y_t$ stack the node states and pre-mixing updates as rows. The network step has the compact form
\begin{equation}\label{eq:matrix_masked_consensus}
    X_{t+1}=Y_t(I-P_t)+WY_tP_t.
\end{equation}
Theorem \ref{thm:one_hop_compatibility} determines which coordinates messages must cover. ZO-COSMO uses this support for both queries and mixing, coupling local progress to the communicated coordinates.

A global support is sufficient but not necessary. In Algorithm \ref{alg:edge_local_cosmo}, a public matching assigns each edge its own support. Endpoints query and average on this support, using one $q$-value message per active node; disjoint edges may use different coordinates. Theorem \ref{thm:edge_local_contraction} gives the contraction.

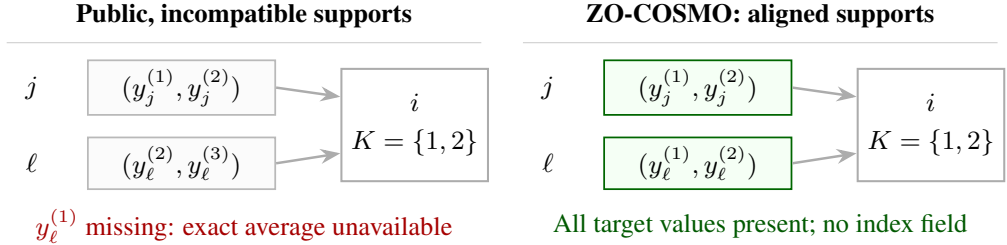
\begin{figure}[H]
\centering
\resizebox{0.96\linewidth}{!}{%
\input{figures/one_hop_compatibility_v20.tex}
}
\caption{Both protocols know the supports. Only the right messages suffice for exact mixing on $K=\{1,2\}$.}
\label{fig:shared_randomness_n_nodes}
\end{figure}

\begin{algorithm}[H]
\caption{ZO-COSMO Core Update}
\label{alg:zo_cosmo}
\small
\begin{algorithmic}[1]
\State \textbf{Input:} $W,\eta,\mu,q$ and public random tape $\{R_t\}_{t<T}$; initialize $x_{i,0}$.
\For{$t=0,\ldots,T-1$}
    \State From $R_t$, reconstruct $\mathcal{S}_t$ and signs $\epsilon_t$; set $P_t=\diag(\mathbf{1}_{\mathcal{S}_t})$ and $u_t=P_t\epsilon_t$.
    \ForAll{agents $i$ in parallel}
        \State $\hat g_{i,t}=\frac{d}{q}\frac{f_i(x_{i,t}+\mu u_t;\xi_{i,t})-f_i(x_{i,t}-\mu u_t;\xi_{i,t})}{2\mu}u_t$.
        \State $y_{i,t}=x_{i,t}-\eta\hat g_{i,t}$.
        \State Broadcast ordered values $\{y_{i,t}^{(j)}:j\in\mathcal{S}_t\}$, with no coordinate indices.
    \EndFor
    \ForAll{agents $i$ in parallel}
        \State $x_{i,t+1}=(I-P_t)y_{i,t}+P_t\sum_{\ell=1}^N w_{i\ell}y_{\ell,t}$.
    \EndFor
\EndFor
\State \textbf{Output:} $x_{i,T}$ or the network average $\bar{x}_T$.
\end{algorithmic}
\end{algorithm}

The sparse-momentum variant changes only the local direction in Algorithm \ref{alg:zo_cosmo}. It replaces $\hat g_{i,t}$ by $P_tm_{i,t+1}$, where
\begin{equation}
    m_{i,t+1}=(I-P_t)m_{i,t}+P_t\bigl(\beta m_{i,t}+(1-\beta)\hat g_{i,t}\bigr).
    \label{eq:optional_subspace_momentum}
\end{equation}
The public support, masked mixing rule, and transmitted message are unchanged. Theorem \ref{thm:convergence} gives the sharper core rate; Appendix Theorem \ref{thm:momentum_convergence} proves convergence with sparse momentum for global and perfect-matching updates. Experiments report $\beta$ explicitly.

\section{Theoretical Analysis}
\label{sec:theory}

The analysis proceeds from decodability to mixing efficiency and then optimization. Exactness refers to coordinate recovery and the ideal arithmetic update; $B_{\mathrm{val}}$ accounts for a fixed-width scalar representation, without a separate quantization-error guarantee.

\subsection{Notation}

Let $F_i(x)=\mathbb{E}_{\xi_i}[f_i(x;\xi_i)]$, $F=N^{-1}\sum_iF_i$, and $\bar{x}_t=N^{-1}\sum_i x_{i,t}$. Before round $t$, an independent public coin $R_t$ draws a uniform $q$-subset $\mathcal{S}_t$ and independent Rademacher signs. Define $P_t=\diag(\mathbf{1}_{\mathcal{S}_t})$ and $u_t=P_t\epsilon_t$. Then $\|u_t\|^2=q$ and
\begin{equation}\label{eq:uut}
    \mathbb{E}[u_tu_t^\top]=(q/d)I_d.
\end{equation}
The local two-query estimator is
\begin{equation}\label{eq:theory_estimator}
    \hat{g}_{i,t}=\frac{d}{q}\frac{f_i(x_{i,t}+\mu u_t;\xi_{i,t})-f_i(x_{i,t}-\mu u_t;\xi_{i,t})}{2\mu}u_t .
\end{equation}
For an ordered support $S=\{s_1<\cdots<s_q\}$, define
\begin{equation}\label{eq:public_support_lift}
    \Res_S z=(z_{s_1},\ldots,z_{s_q}),
    \qquad
    \Lift_S(v)=\sum_{k=1}^q v_ke_{s_k}.
\end{equation}
Then $\Res_S\Lift_S=I_q$, $\Lift_S\Res_S=P_S$, and $\|\Lift_S(v)\|=\|v\|$. In the protocol, a sender transmits $v=\Res_{S_t}z$ and the receiver reconstructs $\Lift_{S_t}(v)$ from the public support. 

\subsection{Assumptions}

\begin{assumption}
\label{assum:smooth_lower}
Each $F_i$ is $L$-smooth, and $F(x)\ge F^*>-\infty$ for all $x$.
\end{assumption}

\begin{assumption}
\label{assum:mixing}
 The mixing matrix $W$ is nonnegative, symmetric, and doubly stochastic, with $\rho=\|W-N^{-1}\mathbf{1}\mathbf{1}^\top\|_2<1$.
\end{assumption}

\begin{assumption}
\label{assum:oracle}
For each node $i$, $\mathbb{E}_{\xi_i}[\nabla f_i(x;\xi_i)]=\nabla F_i(x)$, $\mathbb{E}_{\xi_i}\|\nabla f_i(x;\xi_i)-\nabla F_i(x)\|^2\le\sigma^2$, and along the iterates $\mathbb{E}_{\xi_i}\|\nabla f_i(x;\xi_i)\|^2\le G^2+\sigma^2$. At every round, the oracle samples are fresh and conditionally independent across nodes given the iterates and public direction.
\end{assumption}

\begin{assumption}
\label{assum:third}
For every node $i$, point $x$, oracle sample $\xi$, and sparse direction $u$ used by the algorithm,
$|f_i(x+\mu u;\xi)-f_i(x-\mu u;\xi)-2\mu\nabla f_i(x;\xi)^\top u|\le (L_3/3)\mu^3\|u\|^3$.
\end{assumption}

For later use, define the finite-difference bias scale and the per-node estimator moment bound as
\begin{equation}\label{eq:bmu_mnode_def}
    B_\mu=\frac{L_3}{6}\mu^2dq,
    \qquad
    M_{\mathrm{node}}^2=2d(G^2+\sigma^2)+2B_\mu^2 .
\end{equation}

\subsection{Communication Accounting}

\begin{definition}
\label{def:sparse_message_model}
 A one-hop $q$-sparse current-state message from $j$ to $i$ contains ordered values $\Res_{S_{ji}}(y_j)$. The decoder knows its own state, $W$, and the public transcript, but receives no other part of $y_j$ and does not replay earlier updates. In a private-support protocol the sender encodes $S_{ji}$; in a public-coin protocol both endpoints reconstruct it before communication.
\end{definition}

\begin{theorem}
\label{thm:one_hop_compatibility}
Fix receiver $i$, target support $K\subset[d]$, and message supports $S_{ji}$ independently of the peer states, with $P_K=\diag(\mathbf1_K)$. From one ordered message $\Res_{S_{ji}}(y_j)$ per remote neighbor, a decoder can recover
\begin{equation}\label{eq:one_hop_target}
    P_K\sum_{j=1}^N w_{ij}y_j
\end{equation}
for every collection of peer states if and only if $K\subseteq S_{ji}$ for every $j\ne i$ with $w_{ij}>0$. If $|K|=|S_{ji}|=q$, this requires $S_{ji}=K$. Moreover, if $k\in K\setminus S_{ji}$ and $\|y_j\|_\infty\le R$, every decoder has worst-case coordinate error at least $w_{ij}R$.
\end{theorem}

\begin{proof}
See Appendix \ref{appendix:support_compatibility}.
\end{proof}

A zero-error support field costs at least $H(S_t\mid R_t)$ bits in expectation (Appendix Proposition \ref{prop:conditional_support_cost}). Public supports have zero conditional entropy; a private uniform $q$-subset costs $\log_2\binom dq$ bits. Reusing a data-dependent support over $K$ rounds amortizes its description to $H(S\mid R)/K$, but fixes those coordinates throughout the block.

Consequently, with $B_{\mathrm{val}}$ value bits, each directed ZO-COSMO link transmits $qB_{\mathrm{val}}$ bits. An arbitrary worst-case support, or a uniform private support in expectation, needs $qB_{\mathrm{val}}+\log_2\binom{d}{q}$ bits, while coordinate-list sparse encoding sends
\begin{equation}\label{eq:coordinate_list_cost}
    q(B_{\mathrm{val}}+\lceil\log_2 d\rceil).
\end{equation}
The binomial bounds in Appendix \ref{appendix:conditional_support_cost_proof} quantify the private-support cost. Reusing a full-model Qwen mask for ten rounds reduces coordinate-list overhead from $2.03\times$ to $1.10\times$, with the same coordinates used throughout the block. Public node-specific seeds also regenerate independent random supports without recurring metadata; we include this encoding control in the experiments.

\subsection{Edge-Local Mixing}

Theorem \ref{thm:one_hop_compatibility} applies locally: supports must agree across an active mixing edge, but disjoint edges need not share a support. Let $M_t$ be a public matching, $L_e=(e_i-e_j)(e_i-e_j)^\top$ for $e=\{i,j\}$, and let each $P_{e,t}$ be an independent uniform rank-$q$ coordinate projector. The support-labeled mixing step is
\begin{equation}\label{eq:edge_local_operator}
    X_{t+1}=Y_t-\frac12\sum_{e\in M_t}L_eY_tP_{e,t}.
\end{equation}

\begin{theorem}
\label{thm:edge_local_contraction}
Let $\Pi=I-N^{-1}\mathbf1\mathbf1^\top$ and $\bar L=\mathbb{E}[\sum_{e\in M_t}L_e]$. For any fixed input $Y_t$ independent of the current matching and supports, \eqref{eq:edge_local_operator} preserves the network average and satisfies
\begin{align}
\mathbb{E}_{P}\|\Pi X_{t+1}\|_F^2
&=\|\Pi Y_t\|_F^2-\frac{q}{2d}\sum_{\{i,j\}\in M_t}\|y_{i,t}-y_{j,t}\|^2, \label{eq:edge_exact_contraction}\\
\mathbb{E}_{M,P}\|\Pi X_{t+1}\|_F^2
&\le\left(1-\frac{q\lambda_2(\bar L)}{2d}\right)\|\Pi Y_t\|_F^2. \label{eq:edge_spectral_contraction}
\end{align}
The identities require uniform marginal supports, not independence across disjoint edges. Each active endpoint sends $qB_{\mathrm{val}}$ bits and no indices.
\end{theorem}

\begin{proof}
See Appendix \ref{appendix:edge_local_proof}.
\end{proof}

The optimization update $Y_t$ shares the current query mask with mixing. Appendix \ref{appendix:momentum_convergence} retains this dependence by applying contraction to the past state and controlling the current update separately.

To compare matching schedules independently of a particular state, let
\begin{equation}\label{eq:coordinate_labeled_laplacian}
    \bar L_k=\mathbb{E}\!\left[\sum_{e\in M_t}\mathbf{1}\{k\in S_{e,t}\}L_e\right],
    \qquad
    C_{\mathrm{dir}}=2\mathbb{E}\!\left[\sum_{e\in M_t}|S_{e,t}|\right],
\end{equation}
where $C_{\mathrm{dir}}$ is the expected number of directed scalar values sent in one round.

\begin{theorem}
\label{thm:mixing_per_bit}
For any matching-based exact one-hop operator of the form \eqref{eq:edge_local_operator}, allowing arbitrary state-independent public coordinate supports, the largest $\gamma$ such that
\begin{equation}\label{eq:uniform_matching_contraction}
    \mathbb{E}\|\Pi X_{t+1}\|_F^2\le(1-\gamma)\|\Pi Y_t\|_F^2
\end{equation}
for every fixed input $Y_t$ independent of the sampled operator is
\begin{equation}\label{eq:mixing_per_bit_converse}
    \gamma_\star=\frac12\min_{k\in[d]}\lambda_2(\bar L_k)
    \le \frac{C_{\mathrm{dir}}}{2d(N-1)}.
\end{equation}
With $B_{\mathrm{val}}$-bit scalars, $B_{\mathrm{net}}\ge2B_{\mathrm{val}}d(N-1)\gamma_\star$. For even $N$, uniform rank-$q$ supports on uniformly random perfect matchings of $K_N$ attain equality with $C_{\mathrm{dir}}=Nq$, using either a common support or independent edge supports.
\end{theorem}

\begin{proof}
See Appendix \ref{appendix:mixing_per_bit_proof}.
\end{proof}

Equality requires the same isotropic expected mixing on every coordinate (Appendix \ref{appendix:mixing_per_bit_proof}). Uniform perfect matchings on complete graphs satisfy this condition with either common or independent edge supports. Both attain the pure-mixing bound; Section \ref{sec:correlation_mechanism} analyzes their different query statistics.

Appendix Proposition \ref{prop:estimator} gives the estimator bias $B_\mu$ and moment bound $M_{\mathrm{node}}^2$ used below; a coordinate is refreshed after $d/q$ rounds in expectation.

\subsection{Convergence Bound}

Let
\begin{equation}
    \mathcal{E}_t
    =
    \frac{1}{N}\sum_{i=1}^N\mathbb{E}\|x_{i,t}-\bar{x}_t\|^2
    \label{eq:consensus_error_def}
\end{equation}
be the parameter consensus error.

We use the conservative stability condition
\begin{equation}\label{eq:stepsize_condition}
    0<\eta\le \min\{L^{-1},\,c(Ld)^{-1}\},
\end{equation}
for a universal constant $c>0$, and the canonical smoothing scale
\begin{equation}\label{eq:canonical_mu}
    \mu=\Theta((dq)^{-1/2}T^{-1/4}),
\end{equation}
which balances the displayed finite-difference bias terms in the proof.

\begin{theorem}
\label{thm:convergence}
\label{thm:budget_convergence}
Let $\Delta_0=F(\bar{x}_0)-F^*$ and suppose Assumptions \ref{assum:smooth_lower}--\ref{assum:third} hold. Run Algorithm \ref{alg:zo_cosmo} from a common initialization with $\mu=\Theta((dq)^{-1/2}T^{-1/4})$ and $\eta=c_\eta T^{-1/2}$. For any fixed $c_\eta>0$ and every horizon $T$ large enough that \eqref{eq:stepsize_condition} holds,
\begin{equation}\label{eq:simplified_bigo_rate}
    \frac1T\sum_{t<T}\mathbb{E}\|\nabla F(\bar{x}_t)\|^2
    \le
    \mathcal{O}\!\left(
    \frac{\Delta_0}{\sqrt{T}}
    +
    \frac{d\sigma^2}{N\sqrt{T}}
    +
    \frac1T
    +
    \frac{d^2(G^2+\sigma^2)}{q(1-\rho^2)^2T}
    \right).
\end{equation}
\end{theorem}

\begin{proof}
The descent proof is in Appendix \ref{appendix:main_theorem_proof}; the budget algebra is in Appendix \ref{appendix:comm_budget_rate}.
\end{proof}

\begin{remark}
The leading terms in \eqref{eq:simplified_bigo_rate} describe stochastic ZO descent; the last accounts for sparse mixing. Under a directed-link budget $\mathcal B$, $T=\lfloor\mathcal B/(qB_{\mathrm{val}})\rfloor$ converts the network term to $\mathcal O(d^2(G^2+\sigma^2)B_{\mathrm{val}}/((1-\rho^2)^2\mathcal B))$: its explicit $1/q$ factor cancels, while the spectral penalty remains. Smaller $q$ buys more rounds and queries. Appendix \ref{appendix:comm_budget_rate} gives the full bound, and Appendix \ref{appendix:sparse_consensus} proves the unavoidable $q/d$ one-step mixing slowdown.
\end{remark}

For the momentum update, set $a_\beta=\beta/(1-\beta)$ and $\bar m_t=N^{-1}\sum_i m_{i,t}$. Average preservation gives the exact auxiliary recursion
\begin{equation}\label{eq:main_momentum_identity}
 z_t=\bar x_t-\eta a_\beta\bar m_t,
 \qquad z_{t+1}=z_t-\eta\bar g_t.
\end{equation}
Appendix Theorem \ref{thm:momentum_convergence} bounds the distance from $z_t$ to $\bar x_t$ and handles the shared query--mixing mask, proving convergence for both global and perfect-matching momentum updates.

\input{v28_correlation_main}

\section{Experiments}\label{sec:exp}

We charge value and index fields per directed link. Topology comparisons fix per-node bits because the active edge sets differ; total network traffic is also recorded. Dense and refresh methods pay for every full-state exchange. Seed setup is separate from recurring payloads.

The controls isolate encoding, mixing, and local optimization. Explicit-index and seed-aware encodings measure the value of removing recurring metadata. Dense ZO-DSGD provides a dense-update reference; Com-DSZO tests adaptive compression. ZO-MARINA-type uses periodic dense refreshes rather than MARINA's gradient-difference estimator (Appendix \ref{appendix:experimental_details}).

\subsection{Synthetic Experiments}
On 20-dimensional Rosenbrock with $N=10$, the indexed control uses identical ZO-COSMO updates but redundantly sends each support. Com-DSZO \citep{hua2026compressedzo} combines dense two-point updates with Top-$k$ innovation compression and error-compensated consensus. All methods use two function evaluations per node and round.

\begin{figure}[!htbp]
    \centering
    \includegraphics[width=0.98\linewidth]{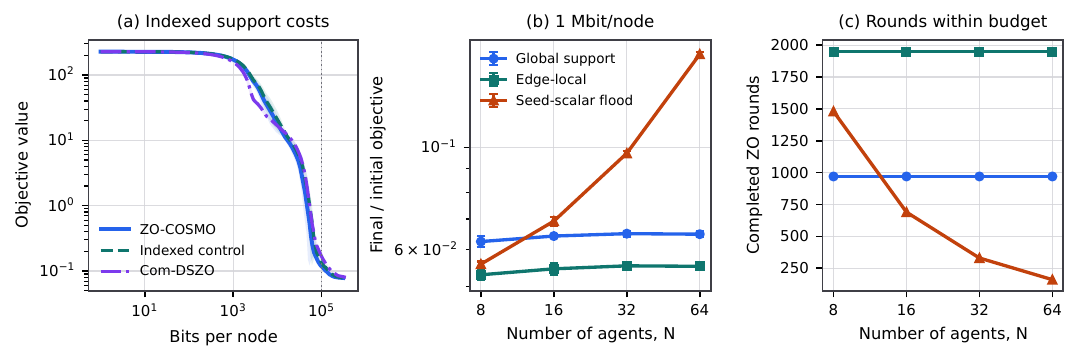}
    \caption{Communication-limited Rosenbrock. Bands and bars show one standard deviation over five seeds.}
    \label{fig:comm_efficiency}
\end{figure}

\begin{table}[!htbp]
    \centering
    \caption{Rosenbrock at $0.1$M bits per node with $q=1$ (five seeds). Lower is better.}
    \label{tab:synthetic_ablation}
    \renewcommand{\arraystretch}{1.12}
    \small
    \begin{tabular}{l c c c}
        \toprule
        \textbf{Method} & \textbf{Ring} & \textbf{Erd\H{o}s-R\'enyi} & \textbf{Grid} \\
        \midrule
        \rowcolor{gray!10}
        \textbf{ZO-COSMO (Ours)} & $\mathbf{0.083 \pm 0.006}$ & $\mathbf{0.119 \pm 0.011}$ & $\mathbf{0.093 \pm 0.006}$ \\
        Same updates + indices & $0.088 \pm 0.006$ & $0.136 \pm 0.015$ & $0.103 \pm 0.006$ \\
        Com-DSZO (Top-$k$) & $0.097 \pm 0.007$ & $0.171 \pm 0.018$ & $0.115 \pm 0.007$ \\
        \bottomrule
    \end{tabular}
\end{table}

At the marked $0.1$M-bit budget, removing repeated indices buys $15$--$17\%$ more rounds. ZO-COSMO has the lowest objective in all 15 paired seed--topology comparisons. At $0.3$M bits the methods have nearly converged, so the gap shrinks below $0.004$ on the Erd\H{o}s-R\'enyi graph. The identical-update control attributes this gain to the extra optimization rounds purchased by value-only messages.

Panels (b)--(c) compare one-hop global support, edge-local support, and an optimistic seed--scalar flood on rings. At $1$ Mbit per node, edge-local ZO-COSMO completes $1950$ logged rounds, versus $970$ global-support rounds, and has the lowest objective for every $N\in\{8,16,32,64\}$. Flooding is competitive at $N=8$, but its all-node dissemination cost leaves only $160$ rounds at $N=64$, where its relative objective is $0.159\pm0.002$ versus $0.055\pm0.0004$ for edge-local mixing. The separate topology stress test reaches $64$ agents on complete, grid, and ring graphs (Appendix Figure \ref{fig:appendix_network_stress}).

\input{v23_correlation_experiment}

\subsection{Zeroth-Order Fine-Tuning for Large Language Models}

We study Qwen2-7B \citep{yang2024qwen2} on QNLI with $2.52$M LoRA coordinates \citep{hu2022lora}. Workers maintain separate float32 adapter and momentum states on disjoint data shards; complete and ring overlays use $N\in\{2,4,8\}$ logical workers. QNLI was selected before training by a fixed task-screening rule (Appendix \ref{appendix:experimental_details}). We evaluate all $5463$ validation examples.

Value-only communication extends the affordable training horizon. Each explicit coordinate costs $22$ bits in addition to its $32$-bit value; at $B=26.214$ Mbit per directed link, ZO-COSMO completes $100$ rounds versus $59$ for explicit-index Rand-$k$. Over ten paired seeds, QNLI accuracy increases from $61.75\pm3.41\%$ to $65.40\pm3.57\%$: a $3.65$-point gain (95\% bootstrap CI $[0.49,6.86]$). Mean query loss decreases by $0.182$ (CI $[0.093,0.273]$). Figure \ref{fig:qwen_two_worker_qnli} compares the resulting models at the same communication budget.

Node-specific public seeds also permit value-only encoding of independent Rand-$k$. This post-hoc lossless re-encoding retains its $100$-round accuracy of $66.48\pm7.36\%$ at the same budget. ZO-COSMO's paired accuracy difference is $-1.08$ points (CI $[-5.85,3.67]$), with mean query loss lower by $0.091$ (CI $[0.006,0.175]$). The control separates the benefit of removing indices from that of coordinating queries and peer-state mixing.

Sparse momentum provides a separate optimization gain. With support, communication, and first-update scale $(1-\beta)\eta$ matched, it raises accuracy from $61.48\pm3.31\%$ to $65.40\pm3.57\%$ (paired $+3.92$ points, CI $[1.67,6.08]$). Matching the initial scale isolates the contribution of stored momentum; Appendix Table \ref{tab:qwen_qnli_momentum_ablation} reports both step-size controls.

The representation saving grows with model dimension. In the single-worker full-model study, sparse updates over $6.53$B trainable coordinates reach $95.03\pm0.13\%$ SST-2 accuracy. Encoding these same updates takes $262.1$K bits per step with public support, compared with $532.5$K bits using explicit indices, a $2.03\times$ reduction.

\begin{figure}[!htbp]
    \centering
    \includegraphics[width=0.80\linewidth]{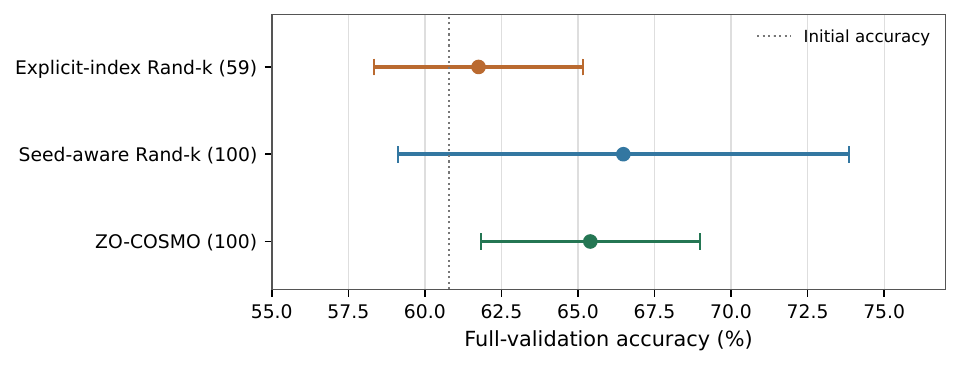}
    \caption{QNLI at $26.214$ Mbit per link (ten-seed mean $\pm$ std). Seed-aware Rand-$k$ is a lossless re-encoding of existing runs.}
    \label{fig:qwen_two_worker_qnli}
\end{figure}

Edge-local mixing extends the budget advantage to larger networks. At $26.214$ Mbit per node, matching permits $100$ rounds, versus $14$ all-neighbor rounds on the eight-worker complete graph and $50$ on the ring. Five-seed accuracy rises from $61.44\pm0.24\%$ to $64.85\pm3.03\%$ on the complete graph and from $61.09\pm0.63\%$ to $63.62\pm0.85\%$ on the ring: gains of $3.42$ points (CI $[1.06,5.63]$) and $2.53$ points (CI $[1.87,3.07]$). With query correlation fixed on the four-worker ring, matching lowers same-budget label NLL by $0.057$ (CI $[0.017,0.105]$). Figure \ref{fig:qwen_edge_local} includes all four topology--size settings and equal-round controls.

\input{v28_qwen_long_main}

\section{Conclusion}

ZO-COSMO performs decentralized two-query updates with $q$ values per active link and no recurring coordinate indices. The analysis identifies the support compatibility needed for one-hop mixing, quantifies contraction per communicated scalar, and explains how direction correlation trades estimation error against disagreement. Synthetic and Qwen experiments show fixed-budget gains from compact messages and matching schedules, together with separate benefits from sparse momentum and edge-sign independence.

\bibliographystyle{iclr2027_conference}
\bibliography{zeroth_order_v20}

\appendix

\section{Proofs}
\label{appendix:heavyweight_proof}

We give the details behind Section \ref{sec:theory}. Constants denoted by $C$ are numerical constants independent of $T,d,q,N,\eta,$ and $\mu$.

\subsection{Proof Structure}
\label{appendix:proof_roadmap}

We first establish the support-coding and decoding results. The optimization proof then controls the finite-difference estimate along $u_t$ and disagreement after mixing on the same support, retaining their dependence in the conditional expectations.

Smoothness at the network average gives
\begin{equation}
    F(\bar{x}_{t+1})
    \le
    F(\bar{x}_t)
    -
    \eta\langle \nabla F(\bar{x}_t),h_t\rangle
    +
    \frac{L\eta^2}{2}\|h_t\|^2,
    \label{eq:appendix_roadmap_descent}
\end{equation}
For Algorithm \ref{alg:zo_cosmo}, $h_t=\bar g_t$, and the direction decomposes as
\begin{equation}
    h_t
    =
    \nabla F(\bar{x}_t)
    +
    \underbrace{(\bar{g}_t-\mathbb{E}_t\bar{g}_t)}_{\text{oracle and direction noise}}
    +
    \underbrace{(\mathbb{E}_t\bar{g}_t-\nabla F(\bar{x}_t))}_{\text{finite-difference and disagreement bias}} .
    \label{eq:appendix_roadmap_decomposition}
\end{equation}
The estimator lemmas control the two error terms in \eqref{eq:appendix_roadmap_decomposition}. The masked disagreement recursion gives the price of mixing a coordinate only when it belongs to the public support. Substitution into \eqref{eq:appendix_roadmap_descent}, followed by telescoping, proves Theorem \ref{thm:convergence}.

The filtration is fixed as follows. Let $\mathcal{F}_t$ contain the iterates, algorithm state, graph, past public coins $R_0,\ldots,R_{t-1}$, and oracle samples through round $t-1$. The fresh public coin $R_t$, and hence the support and signs at round $t$, is independent of $\mathcal{F}_t$. Conditional on $\mathcal{F}_t$ and $R_t$, the oracle variables $\xi_{i,t}$ are independent across nodes. Worker averaging reduces independent oracle noise, while shared direction randomness remains in the conditional moments. Appendix \ref{appendix:public_seed_protocol} specifies the seeded implementation and the role of the ideal public-coin assumption.

We write $\mathbb{E}_S[\cdot]$ for expectation over the support, $\mathbb{E}_u[\cdot]$ for expectation over support and signs, and $\mathbb{E}_\xi[\cdot]$ for expectation over oracle samples. An omitted subscript includes all randomness. Here $P_t$ projects onto $\mathcal{S}_t$, $u_t$ is the sparse Rademacher direction, and $\rho$ is the spectral radius of $W-\frac{1}{N}\mathbf{1}\mathbf{1}^{\top}$. Public information determines the support at both endpoints; its random selection still contributes to the estimator variance.

\subsection{Conditional Support Cost}
\label{appendix:conditional_support_cost_proof}

\begin{proposition}
\label{prop:conditional_support_cost}
Any zero-error uniquely decodable support field needs at least $H(S_t\mid R_t)$ bits in expectation; with $B_{\mathrm{val}}$ value bits, the total is at least
\begin{equation}\label{eq:conditional_total_sparse_cost}
    qB_{\mathrm{val}}+H(S_t\mid R_t).
\end{equation}
If one data-dependent support $S$ is announced and reused for a block of $K$ rounds, its amortized per-round cost is at least
\begin{equation}\label{eq:periodic_support_cost}
    qB_{\mathrm{val}}+\frac{H(S\mid R)}{K}.
\end{equation}
Zero expected support length requires the support to be determined by public information. An unreconstructable uniform $q$-subset costs $\log_2\binom{d}{q}$ bits, of order $q\log(d/q)$ for $q\le d/2$. Public supports remove this field; reused adaptive masks amortize it as in \eqref{eq:periodic_support_cost}.
\end{proposition}

The elementary binomial estimates used in the main text are
\begin{equation}\label{eq:support_binomial_bounds}
q\log_2(d/q)\le\log_2\binom{d}{q}\le q\log_2(ed/q).
\end{equation}

Fix a communication round and condition on the public information $R_t=r$. A zero-error decoder must recover the realized support $S_t$ from the transmitted support codeword and $r$. Thus, for each $r$, the code must be uniquely decodable for the conditional source $S_t\mid R_t=r$. The source-coding converse gives
\begin{equation}
    \mathbb{E}[\ell_t\mid R_t=r]
    \ge
    H(S_t\mid R_t=r),
    \label{eq:conditional_code_converse_fixed_r}
\end{equation}
where $\ell_t$ is the number of transmitted support bits in that round. Taking expectation over $R_t$ yields
\begin{equation}
    \mathbb{E}\ell_t
    \ge
    \mathbb{E}_{R_t}H(S_t\mid R_t)
    =
    H(S_t\mid R_t).
    \label{eq:conditional_code_converse}
\end{equation}
The value payload contributes $qB_{\mathrm{val}}$ additional bits, giving \eqref{eq:conditional_total_sparse_cost}. Because codeword lengths are nonnegative, zero expected support length forces $H(S_t\mid R_t)=0$. For a discrete support this is equivalent to $S_t$ being a function of $R_t$ almost surely. If $S_t=\Phi(R_t,t)$, this condition holds. If instead $S_t$ is uniform over $\binom{[d]}{q}$ and independent of $R_t$, then $H(S_t\mid R_t)=H(S_t)=\log_2\binom{d}{q}$; Stirling's bounds give $\Theta(q\log(d/q))$ when $1\le q\le d/2$.

For a block-reused support, condition on the public information available when the block begins. Zero-error recovery of the private support still requires expected length at least $H(S\mid R)$. The block then carries $KqB_{\mathrm{val}}$ value bits and at least $H(S\mid R)$ support bits. Dividing by $K$ proves \eqref{eq:periodic_support_cost}. This is a lower bound on describing a chosen mask; any communication used to agree on a data-dependent mask across decentralized nodes is additional.

\subsection{Private-Support Code Length}
\label{appendix:support_coding_proof}

Let
\begin{equation}
    \mathcal{U}_{d,q}:=\{S\subset[d]: |S|=q\},
    \qquad
    |\mathcal{U}_{d,q}|=\binom{d}{q}.
    \label{eq:support_universe}
\end{equation}
For a fixed-length zero-error code of length $L$, distinct supports require distinct codewords, so $\binom{d}{q}\le 2^L$. For a prefix-free code whose maximum length is $L$, Kraft's inequality gives
\begin{equation}
    1
    \ge
    \sum_{S\in\mathcal{U}_{d,q}}2^{-\ell(S)}
    \ge
    \binom{d}{q}2^{-L}.
    \label{eq:support_kraft_worst_case}
\end{equation}
Either model therefore requires
\begin{equation}
    \left\lceil \log_2|\mathcal{U}_{d,q}|\right\rceil
    =
    \left\lceil \log_2\binom{d}{q}\right\rceil
    \label{eq:counting_lower_bound}
\end{equation}
bits in the worst case.

For the expected-length statement, assume $S$ is uniform over $\mathcal{U}_{d,q}$. Any uniquely decodable zero-error code has expected length at least the entropy of the source, so
\begin{equation}
    \mathbb{E}[\ell(S)]
    \ge
    H(S)
    =
    \log_2|\mathcal{U}_{d,q}|
    =
    \log_2\binom{d}{q}.
    \label{eq:entropy_lower_bound}
\end{equation}
The standard source-coding converse gives the claim.

The displayed $\Theta(q\log(d/q))$ order follows from the standard binomial estimates
\begin{equation}
    \left(\frac{d}{q}\right)^q
    \le
    \binom{d}{q}
    \le
    \left(\frac{ed}{q}\right)^q.
    \label{eq:appendix_binomial_support_bounds}
\end{equation}
Taking base-two logarithms gives \eqref{eq:support_binomial_bounds} for $q\le d/2$. Coordinate lists cost $q\lceil\log_2 d\rceil$ bits; optimal private-support codes may be shorter. Publicly reconstructible supports require zero support bits. These are the three encodings compared in the communication analysis.

In the public-coin model used by ZO-COSMO, the support is generated as
\begin{equation}
    S_t = \Phi(R_t,t),
    \label{eq:public_coin_support_function}
\end{equation}
where $R_t$ is known to both sender and receiver and $\Phi$ is a deterministic rule. Therefore $S_t$ is measurable with respect to $R_t$, and
\begin{equation}
    H(S_t\mid R_t)=0.
    \label{eq:conditional_entropy_zero_appendix}
\end{equation}
Thus no support field is needed. This coding observation applies to any publicly reconstructible support, including independent Rand-$k$ supports generated from node-specific seeds.

\subsection{Public-Support Lifting}
\label{appendix:public_seed_protocol}

The implementation reconstructs each round from the master seed $s_0$ and public identifiers. In the experiment code, BLAKE2b folds these identifiers into a NumPy generator seed:
\begin{equation}
    r_t = \mathrm{Hash}(s_0,t),
    \qquad
    (S_t,\varepsilon_t)=\mathrm{SampleSparseDirection}(r_t,d,q),
    \label{eq:seed_round_generation}
\end{equation}
where $S_t\subset[d]$ has cardinality $q$ and $\varepsilon_t\in\{\pm1\}^{S_t}$. The sparse direction is $u_{t,j}=\varepsilon_{t,j}\mathbf{1}\{j\in S_t\}$. Node identifiers or canonically ordered edge endpoints select local streams when required. With the same generator, ordering convention, and public identifiers, endpoints reconstruct the same support and signs directly.

The communicated message is the ordered value vector
\begin{equation}
    m_{i,t}
    =
    \big(
        v_{i,t,1},\ldots,v_{i,t,q}
    \big),
    \label{eq:ordered_value_message}
\end{equation}
where $v_{i,t,k}$ is attached to the $k$-th coordinate in the deterministic ordering of $S_t$. No coordinate identifier appears in \eqref{eq:ordered_value_message}. If $S_t=\{s_{t,1}<\cdots<s_{t,q}\}$, then decoding is exactly the public-support lift
\begin{equation}
    \mathrm{Decode}(m_{i,t};s_0,t)
    =
    \Lift_{S_t}(m_{i,t})
    =
    \sum_{k=1}^q v_{i,t,k}e_{s_{t,k}},
    \label{eq:ordered_decode}
\end{equation}
with all other coordinates equal to zero.

Decoding and average preservation hold pathwise for every support-compatible sequence, including seeded sequences. The contraction and convergence rates assume fresh uniform draws in the ideal public-coin model. The finite-seed implementation reproduces the message identities exactly; extending the stochastic rates to its pseudorandom generator would require a separate randomness approximation or cryptographic argument.

Endpoints share only the generator specification, seed, round counter, $d$, and $q$; data, losses, gradients, and momentum remain local. With a $\kappa$-bit master seed, initialization contributes $\kappa+O(\log d+\log q+\log T)$ bits per neighbor when charged explicitly, rather than a recurring support field.

\begin{lemma}
\label{lemma:seed_decode_correctness}
Assume two neighboring agents use the same tuple $(s_0,t,d,q)$ and the same deterministic sampling rule in \eqref{eq:seed_round_generation}. Then the decoder \eqref{eq:ordered_decode} reconstructs the sender's sparse vector from ordered values alone. Consequently, exact sparse-message decoding requires no per-round coordinate metadata.
\end{lemma}

\begin{proof}
Once the sampling rule is fixed, $S_t$ and $\varepsilon_t$ are deterministic functions of $(s_0,t,d,q)$. Both endpoints therefore obtain the same ordered support. The map $\Lift_{S_t}$ places the $k$-th received value on the $k$-th public coordinate, so the decoded vector is the one formed by the sender.
\end{proof}

\subsection{Consequences of One-Hop Support Compatibility}
\label{appendix:support_compatibility}

\begin{proof}[Proof of Theorem \ref{thm:one_hop_compatibility}]
If $K\subseteq S_{ji}$ for every positive-weight remote neighbor, receiver $i$ lifts each ordered message, reads its coordinates in $K$, and adds the known self term. This proves sufficiency. For necessity, suppose $w_{ij}>0$ and $k\in K\setminus S_{ji}$. Hold every state fixed except $y_j^{(k)}$, setting that entry to $R$ in one instance and $-R$ in another. The receiver's state, public transcript, supports, and messages are identical, while the correct $k$-th outputs differ by $2w_{ij}R$. Any common output has error at least half this separation on one instance. Taking arbitrary $R>0$ rules out zero-error recovery; under the stated bound it gives the minimax error $w_{ij}R$. If both supports have cardinality $q$, containment implies equality.
\end{proof}

Theorem \ref{thm:one_hop_compatibility} concerns exact current-state mixing. PermK instead coordinates compressors to control the error of an aggregate estimate \citep{szlendak2022permutation}; the inputs at different workers need not be identical. Complementary supports can serve that purpose without exposing every peer on a common target coordinate. Likewise, a seed--scalar pair reconstructs an update, but not an arbitrary state to which past updates have been applied. For randomized decoders, the same two indistinguishable states give an expected absolute-error lower bound: $|a-w_{ij}R|+|a+w_{ij}R|\ge2w_{ij}R$ for every output $a$, and taking expectations preserves the inequality. Setting the unknown contribution to zero attains this coordinatewise bound when all other contributions are known.

The condition is local to an active communication group. If all neighbors of $i$ participate in one Metropolis step, they need a common target support at $i$. If communication is scheduled as a matching, only the two endpoints of each active edge must agree. This observation leads to the edge-local method below.

\subsection{Edge-Local ZO-COSMO}
\label{appendix:edge_local_cosmo}

Algorithm \ref{alg:edge_local_cosmo} uses public perfect matchings. For a maximal matching, unmatched nodes retain their states. A positive contraction in Theorem \ref{thm:edge_local_contraction} requires a connected expected union graph, equivalently $\lambda_2(\bar L)>0$.

\begin{algorithm}[htbp]
\caption{Edge-Local ZO-COSMO}
\label{alg:edge_local_cosmo}
\small
\begin{algorithmic}[1]
\State \textbf{Input:} graph $G$, $\eta,\mu,q$, and public tape; initialize $x_{i,0}$.
\For{$t=0,\ldots,T-1$}
    \State Reconstruct a public matching $M_t$.
    \ForAll{$e=\{i,j\}\in M_t$ in parallel}
        \State From $(R_t,e)$, reconstruct support $S_{e,t}$, signs $\epsilon_{e,t}$, and $u_{e,t}=P_{e,t}\epsilon_{e,t}$.
        \ForAll{$r\in\{i,j\}$ in parallel}
            \State $\hat g_{r,t}=\frac{d}{q}\frac{f_r(x_{r,t}+\mu u_{e,t};\xi_{r,t})-f_r(x_{r,t}-\mu u_{e,t};\xi_{r,t})}{2\mu}u_{e,t}$; $y_{r,t}=x_{r,t}-\eta\hat g_{r,t}$.
            \State Send $\Res_{S_{e,t}}(y_{r,t})$ to the other endpoint.
        \EndFor
        \State $x_{i,t+1}=(I-P_{e,t})y_{i,t}+\frac12P_{e,t}(y_{i,t}+y_{j,t})$; update $x_{j,t+1}$ symmetrically.
    \EndFor
\EndFor
\end{algorithmic}
\end{algorithm}

Every edge support is keyed by the unordered edge id, so its endpoints reconstruct the same ordered coordinates. Different matched edges may reconstruct different supports because they have disjoint receivers. The code alternates the two perfect matchings of an even ring; on a complete graph it publicly permutes the nodes and pairs adjacent entries. The momentum variant replaces $\hat g_{r,t}$ in the local step by $P_{e,t}m_{r,t+1}$ using \eqref{eq:optional_subspace_momentum} with $P_t=P_{e,t}$.

\subsection{Proof of Theorem \ref{thm:edge_local_contraction}}
\label{appendix:edge_local_proof}

For $e=\{i,j\}$, write $Q_e(Z)=\frac12L_eZP_{e,t}$. Since $L_e^2=2L_e$ and $P_{e,t}^2=P_{e,t}$, $Q_e$ is an orthogonal projector under the Frobenius inner product. Distinct edges in a matching have disjoint endpoints, hence $L_eL_{e'}=0$ and the projectors $Q_e,Q_{e'}$ are orthogonal even when their coordinate supports overlap. Thus $Q=\sum_{e\in M_t}Q_e$ is itself an orthogonal projector and \eqref{eq:edge_local_operator} is $X_{t+1}=(I-Q)Y_t$.

Because $\mathbf1^\top L_e=0$, the network average is unchanged. Moreover, $QY_t$ lies in the disagreement subspace and $Q\Pi Y_t=QY_t$. Pythagoras therefore gives
\begin{align}
\|\Pi X_{t+1}\|_F^2
&=\|(I-Q)\Pi Y_t\|_F^2 \\
&=\|\Pi Y_t\|_F^2-\|QY_t\|_F^2 \\
&=\|\Pi Y_t\|_F^2-\frac12\sum_{\{i,j\}\in M_t}\|P_{e,t}(y_{i,t}-y_{j,t})\|^2.
\label{eq:edge_pythagoras}
\end{align}
For a uniform $q$-coordinate projector and a fixed vector $v$ independent of it, $\mathbb{E}\|P_{e,t}v\|^2=(q/d)\|v\|^2$. Holding $Y_t$ fixed while taking the support expectation in \eqref{eq:edge_pythagoras} proves \eqref{eq:edge_exact_contraction}. No cross-edge independence is used. Averaging also over the matching yields
\begin{equation}
\mathbb{E}_{M,P}\|\Pi X_{t+1}\|_F^2
=\|\Pi Y_t\|_F^2-\frac{q}{2d}\operatorname{tr}(Y_t^\top\bar L Y_t).
\label{eq:edge_laplacian_energy}
\end{equation}
Since $\bar L\mathbf1=0$ and $\bar L\succeq0$, the variational definition of $\lambda_2$ gives $\operatorname{tr}(Y_t^\top\bar L Y_t)\ge\lambda_2(\bar L)\|\Pi Y_t\|_F^2$, proving \eqref{eq:edge_spectral_contraction}.

For a uniformly random perfect matching of the complete graph, $\bar L=L_{K_N}/(N-1)$ and $\lambda_2(\bar L)=N/(N-1)$. Drawing either perfect-matching phase of an even ring uniformly gives $\bar L=L_{\mathrm{ring}}/2$ and
\begin{equation}
    \lambda_2(\bar L)=1-\cos(2\pi/N).
    \label{eq:edge_ring_gap}
\end{equation}
These constants concern independently sampled matching phases and a fixed input. Deterministic alternation, as used by the ring implementation, instead satisfies the two-round bound \eqref{eq:alternating_ring_block_gap}. The factor $q/d$ is the cost of activating only a random coordinate subspace.

\subsection{Proof of Theorem \ref{thm:mixing_per_bit}}
\label{appendix:mixing_per_bit_proof}

Let $z_k=\Pi Y_t e_k$ be the centered network state on coordinate $k$. Expanding \eqref{eq:edge_pythagoras}, then averaging over the matching and its supports, gives
\begin{equation}
\mathbb{E}\|\Pi X_{t+1}\|_F^2
=
\|\Pi Y_t\|_F^2
-\frac12\sum_{k=1}^d z_k^\top\bar L_k z_k.
\label{eq:coordinate_labeled_energy}
\end{equation}
Each $\bar L_k$ is a graph Laplacian. Hence
\begin{equation}
\sum_{k=1}^d z_k^\top\bar L_k z_k
\ge
\left(\min_{k\in[d]}\lambda_2(\bar L_k)\right)
\sum_{k=1}^d\|z_k\|^2.
\label{eq:coordinate_labeled_gap}
\end{equation}
The coefficient is sharp: choose one nonzero column of $Y_t$ to be a unit eigenvector associated with $\lambda_2(\bar L_k)$ for a minimizing coordinate $k$. This proves the equality in \eqref{eq:mixing_per_bit_converse}.

It remains to bound this coefficient by the scalar payload. Since $\operatorname{tr}(L_e)=2$,
\begin{equation}
\sum_{k=1}^d\operatorname{tr}(\bar L_k)
=
2\mathbb{E}\!\left[\sum_{e\in M_t}|S_{e,t}|\right]
=C_{\mathrm{dir}}.
\label{eq:labeled_laplacian_trace_budget}
\end{equation}
For any $N$-node Laplacian $L$, $\lambda_2(L)\le\operatorname{tr}(L)/(N-1)$. Therefore
\begin{equation}
2\gamma_\star
=\min_k\lambda_2(\bar L_k)
\le\frac1d\sum_{k=1}^d\lambda_2(\bar L_k)
\le\frac{C_{\mathrm{dir}}}{d(N-1)},
\label{eq:mixing_trace_converse}
\end{equation}
which proves the converse and, after multiplying by $B_{\mathrm{val}}$, the bit lower bound.

For a uniformly random perfect matching of $K_N$, every edge is selected with probability $1/(N-1)$, so
\begin{equation}
\mathbb{E}\!\left[\sum_{e\in M_t}L_e\right]
=\frac{L_{K_N}}{N-1}
=\frac{N}{N-1}\Pi.
\label{eq:complete_matching_laplacian}
\end{equation}
A uniform rank-$q$ support contains coordinate $k$ with probability $q/d$, and a perfect matching sends $C_{\mathrm{dir}}=Nq$ directed values. Thus $\bar L_k=qN\Pi/[d(N-1)]$ for every $k$, and
\begin{equation}
\gamma_\star=\frac{qN}{2d(N-1)}
=\frac{Nq}{2d(N-1)}
=\frac{C_{\mathrm{dir}}}{2d(N-1)},
\label{eq:complete_matching_tightness}
\end{equation}

To characterize equality, put $c=C_{\mathrm{dir}}/[d(N-1)]$. Equality in \eqref{eq:mixing_trace_converse} requires all coordinate gaps to be equal, and all $N-1$ eigenvalues on the disagreement subspace of each $\bar L_k$ to equal its smallest one. Since $\bar L_k\mathbf1=0$, these conditions are equivalent to
\begin{equation}\label{eq:matching_equality_characterization}
    \bar L_k=c\Pi \quad\text{for every }k\in[d].
\end{equation}
Conversely, \eqref{eq:matching_equality_characterization} gives $\gamma_\star=c/2$, attaining the bound. Unequal coordinate coverage or an anisotropic expected Laplacian makes at least one inequality strict. Uniform perfect matchings with uniform marginal supports satisfy this pure-mixing equality criterion for either common or independent edge supports.

\subsection{Balanced Public Schedule}
\label{appendix:balanced_schedule_proof}

Let $K=\lceil d/q\rceil$. At epoch $e$, all agents use the shared seed to generate a public permutation $\pi_e$ and take
\begin{equation}
    \mathcal{S}_{eK+r}=\{\pi_e(rq+1),\ldots,\pi_e(\min\{(r+1)q,d\})\},\qquad r=0,\ldots,K-1.
    \label{eq:balanced_schedule}
\end{equation}
The permutation $\pi_e$ is a bijection from $[d]$ to $[d]$. The blocks in \eqref{eq:balanced_schedule} form a partition of the ordered list
\begin{equation}
    \pi_e(1),\pi_e(2),\ldots,\pi_e(d).
    \label{eq:permutation_partition}
\end{equation}
Hence every coordinate appears in at least one block of epoch $e$. If $q$ divides $d$, every block has size exactly $q$, the blocks are disjoint, and every coordinate appears exactly once.

Since every coordinate appears at least once in every epoch, a coordinate activated in round $eK+r$ must be activated again by the end of the next epoch. The largest possible gap happens when the coordinate appears in the first round of epoch $e$ and in the last round of epoch $e+1$, giving a gap of at most
\begin{equation}
    (K-1)+K=2K-1.
    \label{eq:max_balanced_gap}
\end{equation}
The average refresh time is at most the epoch length $K$ because each epoch contains at least one activation of the coordinate. The padding rule, when used, only adds extra activations and cannot increase the maximum gap or the average refresh time.

\subsection{Masked Matrix Form of ZO-COSMO}
\label{appendix:masked_updates}

Let $\mathbf{X}_t,\mathbf{G}_t\in\mathbb{R}^{N\times d}$ stack the row vectors $x_{i,t}$ and $\hat{g}_{i,t}$. Since every estimator is supported on $\mathcal{S}_t$, $\mathbf{G}_tP_t=\mathbf{G}_t$. The local update is
\begin{equation}
    \mathbf{Y}_t
    =
    \mathbf{X}_t-\eta \mathbf{G}_t,
    \label{eq:matrix_local_y}
\end{equation}
and the partial consensus step is
\begin{equation}
    \mathbf{X}_{t+1}
    =
    \mathbf{Y}_t(I-P_t)+W\mathbf{Y}_tP_t.
    \label{eq:matrix_partial_consensus}
\end{equation}
Multiplying \eqref{eq:matrix_partial_consensus} by $N^{-1}\mathbf{1}^\top$ from the left and using $\mathbf{1}^\top W=\mathbf{1}^\top$ gives the exact averaged update
\begin{equation}
    \bar{x}_{t+1}
    =
    \bar{x}_t-\eta\bar g_t,
    \qquad
    \bar g_t:=\frac{1}{N}\sum_{i=1}^N\hat g_{i,t}.
    \label{eq:avg_masked_update}
\end{equation}
Thus the exact descent direction of the network average is
\begin{equation}
    h_t:=\bar g_t.
    \label{eq:effective_direction}
\end{equation}

\subsection{Sparse Direction Identities}
\label{appendix:sparse_direction_identities}

\begin{lemma}
\label{lemma:sparse_direction_identities}
Let $\mathcal{S}$ be sampled uniformly from all $q$-subsets of $[d]$, let $\epsilon_j$ be independent Rademacher signs on $\mathcal{S}$, and define $u_j=\epsilon_j\mathbf{1}\{j\in\mathcal{S}\}$. Then
\begin{equation}
    \mathbb{E}[uu^\top]=\frac{q}{d}I_d,
    \qquad
    \|u\|^2=q\quad\text{almost surely}.
    \label{eq:sparse_direction_identities}
\end{equation}
Consequently, for every deterministic vector $v\in\mathbb{R}^d$,
\begin{equation}
    \mathbb{E}\left\|
    \frac{d}{q}uu^\top v
    \right\|^2
    =
    d\|v\|^2.
    \label{eq:sparse_direction_second_moment_lemma}
\end{equation}
\end{lemma}

\begin{proof}
For $j\neq k$, $\mathbb{E}[u_ju_k]=0$ because the signs are independent and mean-zero conditional on the support. For $j=k$,
\begin{equation}
    \mathbb{E}[u_j^2]
    =
    \mathbb{P}(j\in\mathcal{S})
    =
    \frac{\binom{d-1}{q-1}}{\binom{d}{q}}
    =
    \frac{q}{d}.
\end{equation}
This proves $\mathbb{E}[uu^\top]=(q/d)I_d$. Since exactly $q$ entries of $u$ are nonzero and each nonzero entry has magnitude one, $\|u\|^2=q$. Therefore,
\begin{equation}
\begin{aligned}
    \mathbb{E}\left\|
    \frac{d}{q}uu^\top v
    \right\|^2
    &=
    \left(\frac{d}{q}\right)^2
    \mathbb{E}\big[(u^\top v)^2\|u\|^2\big]  \\
    &=
    \left(\frac{d}{q}\right)^2
    q\,v^\top\mathbb{E}[uu^\top]v
    =
    d\|v\|^2.
\end{aligned}
\end{equation}
\end{proof}

\begin{lemma}
\label{lemma:sparse_direction_fourth}
Under the same sampling rule as Lemma \ref{lemma:sparse_direction_identities}, for any deterministic $v\in\mathbb{R}^d$,
\begin{equation}
    \mathbb{E}\big[(u^\top v)^4\big]
    \le
    3\frac{q}{d}\|v\|^4,
    \label{eq:sparse_fourth_direction}
\end{equation}
Moreover,
\begin{equation}
    \mathbb{E}\left\|
        \frac{d}{q}uu^\top v
    \right\|^4
    \le
    3\frac{d^3}{q}\|v\|^4.
    \label{eq:sparse_fourth_projected}
\end{equation}
\end{lemma}

\begin{proof}
Condition on the support $\mathcal{S}$. The random variable $u^\top v=\sum_{j\in\mathcal{S}}\varepsilon_jv_j$ is a Rademacher sum, so by the fourth-moment identity for independent signs,
\begin{equation}
    \mathbb{E}_{\varepsilon}\big[(u^\top v)^4\mid \mathcal{S}\big]
    =
    3\left(\sum_{j\in\mathcal{S}}v_j^2\right)^2
    -
    2\sum_{j\in\mathcal{S}}v_j^4
    \le
    3\left(\sum_{j\in\mathcal{S}}v_j^2\right)^2.
    \label{eq:fourth_conditioned_support}
\end{equation}
For a uniform $q$-subset, the inclusion probabilities are
$\mathbb{P}(j\in\mathcal{S})=q/d$ and
$\mathbb{P}(j,k\in\mathcal{S})=q(q-1)/(d(d-1))$ for $j\ne k$. Therefore
\begin{equation}
    \begin{aligned}
    \mathbb{E}_{\mathcal{S}}
    \left[
        \left(\sum_{j\in\mathcal{S}}v_j^2\right)^2
    \right]
    &=
    \frac{q}{d}\sum_{j=1}^dv_j^4
    +
    \frac{2q(q-1)}{d(d-1)}
    \sum_{1\le j<k\le d}v_j^2v_k^2  \\
    &\le
    \frac{q}{d}
    \left(
        \sum_{j=1}^dv_j^4
        +
        2\sum_{1\le j<k\le d}v_j^2v_k^2
    \right)  \\
    &=
    \frac{q}{d}\|v\|^4.
    \end{aligned}
    \label{eq:fourth_support_exact}
\end{equation}
Combining \eqref{eq:fourth_conditioned_support} and \eqref{eq:fourth_support_exact} gives
\begin{equation}
    \mathbb{E}\big[(u^\top v)^4\big]
    \le
    3\frac{q}{d}\|v\|^4.
    \label{eq:fourth_support_crude}
\end{equation}
This proves \eqref{eq:sparse_fourth_direction}. Since $\|u\|^2=q$ almost surely,
\begin{equation}
    \left\|
        \frac{d}{q}uu^\top v
    \right\|^4
    =
    \left(\frac{d}{q}\right)^4
    (u^\top v)^4\|u\|^4
    =
    \left(\frac{d}{q}\right)^4q^2(u^\top v)^4.
    \label{eq:fourth_projected_identity}
\end{equation}
Combining \eqref{eq:fourth_projected_identity} with \eqref{eq:sparse_fourth_direction} gives
\begin{equation}
    \mathbb{E}\left\|
        \frac{d}{q}uu^\top v
    \right\|^4
    \le
    \left(\frac{d}{q}\right)^4q^2
    \cdot
    3\frac{q}{d}\|v\|^4
    =
    3\frac{d^3}{q}\|v\|^4,
\end{equation}
which is \eqref{eq:sparse_fourth_projected}.
\end{proof}

\begin{proposition}
\label{prop:estimator}
Under Assumptions \ref{assum:oracle} and \ref{assum:third},
\begin{align}
    \mathbb{E}[\hat{g}_{i,t}\mid x_{i,t}]&=\nabla F_i(x_{i,t})+b_\mu(x_{i,t}),\quad \|b_\mu(x_{i,t})\|\le B_\mu, \label{eq:estimator_bias}\\
    \mathbb{E}\|\hat{g}_{i,t}\|^2&\le M_{\mathrm{node}}^2. \label{eq:estimator_second_moment}
\end{align}
\end{proposition}

\subsection{Proof of Proposition \ref{prop:estimator}}
\label{appendix:estimator_proof}

Fix a node $i$, an iterate $x$, and write $u=u_t$. By Assumption \ref{assum:third},
\begin{equation}
    \frac{f_i(x+\mu u;\xi)-f_i(x-\mu u;\xi)}{2\mu}
    =
    \nabla f_i(x;\xi)^\top u+r_\mu(x,u,\xi),
    \label{eq:central_difference_decomp}
\end{equation}
where
\begin{equation}
    |r_\mu(x,u,\xi)|
    \le
    \frac{L_3}{6}\mu^2\|u\|^3.
    \label{eq:rmu_bound}
\end{equation}
Substituting \eqref{eq:central_difference_decomp} into \eqref{eq:theory_estimator} yields
\begin{equation}
    \hat{g}_{i,t}
    =
    \underbrace{\frac{d}{q}uu^\top\nabla f_i(x;\xi)}_{a(x,u,\xi)}
    +
    \underbrace{\frac{d}{q}r_\mu(x,u,\xi)u}_{e_\mu(x,u,\xi)}.
    \label{eq:g_a_e_decomp}
\end{equation}
By Lemma \ref{lemma:sparse_direction_identities} and $\mathbb{E}_{\xi}[\nabla f_i(x;\xi)]=\nabla F_i(x)$,
\begin{equation}
    \mathbb{E}[a(x,u,\xi)\mid x]
    =
    \frac{d}{q}\mathbb{E}[uu^\top]\nabla F_i(x)
    =
    \nabla F_i(x).
    \label{eq:a_unbiased}
\end{equation}
For the remainder, $\|u\|^2=q$, so \eqref{eq:rmu_bound} gives
\begin{equation}
    \|e_\mu(x,u,\xi)\|
    \le
    \frac{d}{q}\cdot \frac{L_3}{6}\mu^2\|u\|^3\cdot \|u\|
    =
    \frac{L_3}{6}\mu^2dq
    =
    B_\mu.
    \label{eq:emu_norm_bound}
\end{equation}
Combining \eqref{eq:a_unbiased} and \eqref{eq:emu_norm_bound} proves \eqref{eq:estimator_bias}.

For the second moment, first condition on $\xi$ and apply Lemma \ref{lemma:sparse_direction_identities}:
\begin{equation}
\begin{aligned}
    \mathbb{E}_{u}\|a(x,u,\xi)\|^2
    &=
    \left(\frac{d}{q}\right)^2
    \mathbb{E}_{u}\left[(\nabla f_i(x;\xi)^\top u)^2\|u\|^2\right] \\
    &=
    \left(\frac{d}{q}\right)^2 q\,
    \nabla f_i(x;\xi)^\top \mathbb{E}[uu^\top]\nabla f_i(x;\xi) \\
    &=
    d\|\nabla f_i(x;\xi)\|^2.
\end{aligned}
\label{eq:a_second_moment}
\end{equation}
Taking expectation over $\xi$ and using Assumption \ref{assum:oracle},
\begin{equation}
    \mathbb{E}\|a(x,u,\xi)\|^2
    \le
    d(G^2+\sigma^2).
    \label{eq:a_second_moment_bound}
\end{equation}
Finally, $\|\hat{g}_{i,t}\|^2\le 2\|a\|^2+2\|e_\mu\|^2$, so
\begin{equation}
    \mathbb{E}\|\hat{g}_{i,t}\|^2
    \le
    2d(G^2+\sigma^2)+2B_\mu^2
    =
    M_{\mathrm{node}}^2.
    \label{eq:node_second_moment_final}
\end{equation}
This proves \eqref{eq:estimator_second_moment}.

\subsection{Averaged Estimator Bounds}
\label{appendix:avg_estimator}

\begin{lemma}
\label{lemma:avg_estimator_shared_direction}
Define
\begin{equation}
    \bar{g}_t:=\frac{1}{N}\sum_{i=1}^N\hat{g}_{i,t},
    \qquad
    \bar{\nabla}_t:=\frac{1}{N}\sum_{i=1}^N\nabla F_i(x_{i,t}).
    \label{eq:gbar_nablabar}
\end{equation}
Let $\mathcal{F}_t$ denote the sigma-field generated by all iterates and all randomness before the support, direction, and oracle samples at time $t$ are drawn. Under Assumptions \ref{assum:smooth_lower}--\ref{assum:third},
\begin{equation}
    \mathbb{E}\|\bar{g}_t\|^2
    \le
    C d\,\mathbb{E}\|\nabla F(\bar{x}_t)\|^2
    +
    C d L^2\mathcal{E}_t
    +
    C\frac{d\sigma^2}{N}
    +
    C B_\mu^2,
    \label{eq:gbar_second_moment}
\end{equation}
More precisely, before replacing the local average gradient by the gradient at the network average,
\begin{equation}
    \mathbb{E}\!\left[\|\bar{g}_t\|^2\mid\mathcal{F}_t\right]
    \le
    C d\,\|\bar{\nabla}_t\|^2
    +
    C\frac{d\sigma^2}{N}
    +
    C B_\mu^2.
    \label{eq:gbar_shared_direction_decomp}
\end{equation}
Moreover,
\begin{equation}
    \mathbb{E}\big[
    \langle \nabla F(\bar{x}_t),\bar{g}_t\rangle
    \big]
    \ge
    \frac{3}{4}\mathbb{E}\|\nabla F(\bar{x}_t)\|^2
    -
    C L^2\mathcal{E}_t
    -
    C B_\mu^2.
    \label{eq:gbar_alignment}
\end{equation}
The only variance term reduced by $N$ is the independent oracle noise. The sparse direction $u_t$ is public across agents, so its random-direction variance is shared rather than averaged away.
\end{lemma}

\begin{proof}
By Jensen's inequality and $L$-smoothness,
\begin{equation}
\begin{aligned}
    \mathbb E\|\bar{\nabla}_t-\nabla F(\bar{x}_t)\|^2
    &=
    \mathbb E\left\|
    \frac{1}{N}\sum_{i=1}^N
    \big(\nabla F_i(x_{i,t})-\nabla F_i(\bar{x}_t)\big)
    \right\|^2 \\
    &\le
    \frac{1}{N}\sum_{i=1}^N\mathbb E
    \|\nabla F_i(x_{i,t})-\nabla F_i(\bar{x}_t)\|^2 \\
    &\le
    L^2\mathcal{E}_t.
\end{aligned}
\label{eq:local_gradient_disagreement}
\end{equation}
The price of decentralization is visible here: local gradients are evaluated at $x_{i,t}$ rather than at the network average $\bar{x}_t$.

Next write the stochastic gradient noise as
\begin{equation}
    \zeta_{i,t}
    :=
    \nabla f_i(x_{i,t};\xi_{i,t})-\nabla F_i(x_{i,t}),
    \qquad
    \mathbb{E}[\zeta_{i,t}\mid \mathcal{F}_t]=0,
    \qquad
    \mathbb{E}\|\zeta_{i,t}\|^2\le\sigma^2.
    \label{eq:zeta_def}
\end{equation}
The samples $\xi_{i,t}$ are independent across agents, hence
\begin{equation}
    \mathbb{E}\left[
    \left\|\frac{1}{N}\sum_{i=1}^N\zeta_{i,t}\right\|^2
    \middle|\mathcal{F}_t
    \right]
    =
    \frac{1}{N^2}\sum_{i=1}^N
    \mathbb{E}\|\zeta_{i,t}\|^2
    \le
    \frac{\sigma^2}{N}.
    \label{eq:variance_reduction_N}
\end{equation}
Equation \eqref{eq:variance_reduction_N} reduces independent oracle variance by $N$. The common direction $u_t$ remains in the averaged estimator.

Using \eqref{eq:g_a_e_decomp}, the average estimator can be written as
\begin{equation}
    \bar{g}_t
    =
    \frac{d}{q}u_tu_t^\top
    \left(\frac{1}{N}\sum_{i=1}^N\nabla f_i(x_{i,t};\xi_{i,t})\right)
    +
    \frac{1}{N}\sum_{i=1}^N e_{\mu,i,t}.
    \label{eq:gbar_decomp}
\end{equation}
By Lemma \ref{lemma:sparse_direction_identities} and the conditional variance identity \eqref{eq:variance_reduction_N},
\begin{equation}
    \mathbb{E}\!\left[\|\bar{g}_t\|^2\mid\mathcal{F}_t\right]
    \le
    C d\,\|\bar{\nabla}_t\|^2
    +
    C\frac{d\sigma^2}{N}
    +
    C B_\mu^2,
    \label{eq:gbar_shared_direction_decomp_proof}
\end{equation}
which is \eqref{eq:gbar_shared_direction_decomp}. Taking expectation and using \eqref{eq:local_gradient_disagreement} gives
\begin{equation}
\begin{aligned}
    \mathbb{E}\|\bar{g}_t\|^2
    &\le
    C d\,\mathbb{E}\|\bar{\nabla}_t\|^2
    +
    C\frac{d\sigma^2}{N}
    +
    C B_\mu^2 \\
    &\le
    C d\,\mathbb{E}\|\nabla F(\bar{x}_t)\|^2
    +
    C d L^2\mathcal{E}_t
    +
    C\frac{d\sigma^2}{N}
    +
    C B_\mu^2.
\end{aligned}
    \label{eq:gbar_second_moment_detailed}
\end{equation}
This proves \eqref{eq:gbar_second_moment}.

We also need an alignment bound. Conditional on $\mathcal{F}_t$, Proposition \ref{prop:estimator} gives
\begin{equation}
    \mathbb{E}[\bar{g}_t\mid\mathcal{F}_t]
    =
    \bar{\nabla}_t+\bar{b}_{\mu,t},
    \qquad
    \|\bar{b}_{\mu,t}\|\le B_\mu.
    \label{eq:gbar_conditional_bias}
\end{equation}
Using Young's inequality twice,
\begin{equation}
\begin{aligned}
    \mathbb{E}\big[
    \langle \nabla F(\bar{x}_t),\bar{g}_t\rangle
    \big]
    &=
    \mathbb{E}\|\nabla F(\bar{x}_t)\|^2
    +
    \mathbb{E}\langle \nabla F(\bar{x}_t),
    \bar{\nabla}_t-\nabla F(\bar{x}_t)\rangle \\
    &\qquad
    +
    \mathbb{E}\langle \nabla F(\bar{x}_t),\bar{b}_{\mu,t}\rangle \\
    &\ge
    \frac{3}{4}\mathbb{E}\|\nabla F(\bar{x}_t)\|^2
    -
    C L^2\mathcal{E}_t
    -
    C B_\mu^2.
\end{aligned}
\end{equation}
This proves \eqref{eq:gbar_alignment}.
\end{proof}

\subsection{Base Direction for the Convergence Proof}
\label{appendix:base_direction}

By \eqref{eq:avg_masked_update}, the direction used by Algorithm \ref{alg:zo_cosmo} is exactly
\begin{equation}
    h_t=\bar g_t.
    \label{eq:base_direction_identity}
\end{equation}
The second-moment and alignment bounds for the actual descent direction are therefore
\begin{align}
    \mathbb{E}\|h_t\|^2
    &\le
    C d\,\mathbb{E}\|\nabla F(\bar x_t)\|^2
    +C dL^2\mathcal E_t
    +C\frac{d\sigma^2}{N}
    +C B_\mu^2,
    \label{eq:h_second_moment}\\
    \mathbb{E}\langle \nabla F(\bar x_t),h_t\rangle
    &\ge
    \frac34\mathbb{E}\|\nabla F(\bar x_t)\|^2
    -C L^2\mathcal E_t
    -C B_\mu^2.
    \label{eq:h_alignment}
\end{align}
Both displays follow immediately from Lemma \ref{lemma:avg_estimator_shared_direction} and \eqref{eq:base_direction_identity}.
\subsection{Masked Disagreement Energy}
\label{appendix:sparse_consensus}

\begin{lemma}
\label{lemma:masked_disagreement_energy}
For Algorithm \ref{alg:zo_cosmo}, the disagreement sequence satisfies the one-step recursion \eqref{eq:consensus_recursion} and the averaged bound \eqref{eq:consensus_average_final}.
\end{lemma}

\begin{proof}
Let $J=I_N-N^{-1}\mathbf{1}\mathbf{1}^\top$. Then $\mathcal{E}_t=N^{-1}\mathbb{E}\|J\mathbf{X}_t\|_F^2$. Consider coordinate $j$. If $j\notin\mathcal{S}_t$, then $y^{(j)}=x^{(j)}$ and the coordinate is not mixed. If $j\in\mathcal{S}_t$, then $y^{(j)}=x^{(j)}-\eta g^{(j)}$, and the update is mixed by $W$. Hence
\begin{equation}
    \|JW y^{(j)}\|^2
    \le
    \rho^2\|Jy^{(j)}\|^2,
    \label{eq:coordinate_contraction}
\end{equation}
where $x^{(j)},g^{(j)},y^{(j)}$ are the $j$-th columns of $\mathbf{X}_t,\mathbf{G}_t,\mathbf{Y}_t$. On active coordinates, the relaxed triangle inequality gives
\begin{equation}
    \|Jy^{(j)}\|^2
    \le
    (1+\alpha)\|Jx^{(j)}\|^2
    +
    (1+\alpha^{-1})\eta^2\|Jg^{(j)}\|^2.
    \label{eq:y_triangle}
\end{equation}
Condition on $\mathcal{F}_t$ and average over the fresh public coin and oracle samples. Since $g^{(j)}=0$ off the active support and $\mathbb{P}(j\in\mathcal{S}_t)=q/d$,
\begin{equation}
\begin{aligned}
    \mathbb{E}
    \big[\|Jx_{t+1}^{(j)}\|^2\mid\mathcal F_t\big]
    &\le
    \left(1-\frac{q}{d}\right)\|Jx^{(j)}\|^2
    +
    \frac{q}{d}\rho^2(1+\alpha)\|Jx^{(j)}\|^2 \\
    &\qquad
    +
    \rho^2(1+\alpha^{-1})\eta^2
    \mathbb{E}\big[\|Jg^{(j)}\|^2\mid\mathcal F_t\big].
\end{aligned}
    \label{eq:coordinate_support_expectation}
\end{equation}
If $\rho=0$, the active-coordinate mixing term vanishes and the same bound follows by taking the limit below. Assume $\rho>0$ and choose $\alpha=(1-\rho^2)/(2\rho^2)$. Then $\rho^2(1+\alpha)=(1+\rho^2)/2$, and therefore
\begin{equation}
    \left(1-\frac{q}{d}\right)
    +
    \frac{q}{d}\rho^2(1+\alpha)
    =
    1-\frac{q}{d}\frac{1-\rho^2}{2}.
    \label{eq:sparse_contraction_factor}
\end{equation}
With the same choice of $\alpha$,
\begin{equation}
    \rho^2(1+\alpha^{-1})
    =
    \rho^2+\frac{2\rho^4}{1-\rho^2}
    \le
    \frac{C}{1-\rho^2}.
    \label{eq:rho_alpha_payload}
\end{equation}
Summing \eqref{eq:coordinate_support_expectation} over coordinates and taking the remaining expectation yields
\begin{equation}
    \mathcal{E}_{t+1}
    \le
    \left(1-\frac{q}{d}\frac{1-\rho^2}{2}\right)\mathcal{E}_t
    +
    \frac{C\eta^2}{1-\rho^2}
    \mathbb{E}\left[
        \frac{1}{N}\|J\mathbf{G}_t\|_F^2
    \right].
    \label{eq:consensus_recursion_raw}
\end{equation}
By $\|J\|_2\le1$ and Proposition \ref{prop:estimator},
\begin{equation}
    \mathbb{E}\!\left[\frac1N\|J\mathbf G_t\|_F^2\right]
    \le \frac1N\sum_{i=1}^N\mathbb E\|\hat g_{i,t}\|^2
    \le M_{\mathrm{node}}^2.
    \label{eq:disagreement_estimator_moment}
\end{equation}
Therefore,
\begin{equation}
    \mathcal{E}_{t+1}
    \le
    \kappa\mathcal{E}_t
    +
    \frac{C\eta^2}{1-\rho^2}M_{\mathrm{node}}^2,
    \qquad
    \kappa:=1-\frac{q}{d}\frac{1-\rho^2}{2}.
    \label{eq:consensus_recursion}
\end{equation}
Unrolling \eqref{eq:consensus_recursion} gives
\begin{equation}
    \frac{1}{T}\sum_{t=0}^{T-1}\mathcal{E}_t
    \le
    \frac{\mathcal{E}_0}{T(1-\kappa)}
    +
    \frac{C\eta^2M_{\mathrm{node}}^2}{(1-\rho^2)(1-\kappa)}.
    \label{eq:consensus_average_unrolled}
\end{equation}
Since $1-\kappa=(q/d)(1-\rho^2)/2$,
\begin{equation}
    \frac{1}{T}\sum_{t=0}^{T-1}\mathcal{E}_t
    \le
    C\frac{d}{q}\frac{\mathcal{E}_0}{T(1-\rho^2)}
    +
    C\eta^2\frac{d}{q}\frac{M_{\mathrm{node}}^2}{(1-\rho^2)^2}.
    \label{eq:consensus_average_final}
\end{equation}
\end{proof}

The second term in \eqref{eq:consensus_average_final} is the sparse-consensus term used in the main theorem. The first term vanishes when all agents are initialized equally and is otherwise a transient.

\paragraph{Sparse-consensus factor.}
Equation \eqref{eq:consensus_average_final} is where the random subspace design enters the network analysis. Full-coordinate mixing contracts disagreement energy at rate $1-\rho^2$. ZO-COSMO mixes coordinate $j$ only on the event $j\in S_t$, which has probability $q/d$. The effective energy contraction probability is therefore of order $(q/d)(1-\rho^2)$, and unrolling the recursion produces the multiplier $d/q$.

The second spectral-gap factor is already visible in \eqref{eq:rho_alpha_payload}. The active-coordinate descent step perturbs consensus before mixing, and bounding that perturbation costs one factor of $(1-\rho^2)^{-1}$. Unrolling the masked contraction costs another. Together they give the squared energy gap in Theorem \ref{thm:convergence}, along with the sparse activation factor $d/q$.

The $q/d$ factor is intrinsic to this one-step partial-consensus model, as the next lemma shows. The remaining slack enters through Young's inequality in \eqref{eq:rho_alpha_payload} and the bound of the injected disagreement by the full per-node estimator moment. For matching updates, Proposition \ref{prop:corr_moments} instead separates aggregate estimation error from disagreement injection, exposing the effect of direction correlation.

\begin{lemma}
\label{lemma:masked_contraction_tight}
The sparse activation factor $q/d$ in \eqref{eq:consensus_recursion} is tight for partial coordinate mixing. In particular, there are disagreement states for which one pure partial-consensus step with $\eta=0$ contracts the expected disagreement energy by exactly
\begin{equation}
    1-\frac{q}{d}(1-\rho^2).
    \label{eq:masked_contraction_witness}
\end{equation}
\end{lemma}

\begin{proof}
Because $W$ is symmetric and doubly stochastic, there is a unit vector $v\perp\mathbf{1}$ and an eigenvalue $\lambda$ with $Wv=\lambda v$ and $|\lambda|=\rho$. Take a disagreement matrix whose $j$-th column is $x^{(j)}=a_jv$ for arbitrary coefficients $a_j$, and set $\eta=0$ so that the step only applies partial consensus. If $j\notin S_t$, then $x_{+}^{(j)}=x^{(j)}$. If $j\in S_t$, then $x_{+}^{(j)}=Wx^{(j)}=a_j\lambda v$, and hence
\begin{equation}
    \|Jx_{+}^{(j)}\|^2
    =
    \begin{cases}
    \|x^{(j)}\|^2, & j\notin S_t,\\
    \rho^2\|x^{(j)}\|^2, & j\in S_t.
    \end{cases}
\end{equation}
Since $\mathbb{P}(j\in S_t)=q/d$,
\begin{equation}
    \mathbb{E}\|Jx_{+}^{(j)}\|^2
    =
    \left(1-\frac{q}{d}(1-\rho^2)\right)\|x^{(j)}\|^2.
\end{equation}
Summing over coordinates gives \eqref{eq:masked_contraction_witness}. Thus, for algorithms that mix only the active coordinates and leave inactive coordinates unchanged, the $q/d$ slowdown in the one-step consensus contraction cannot be removed by sharper analysis alone.
\end{proof}

\subsection{Proof of Theorem \ref{thm:convergence}}
\label{appendix:main_theorem_proof}

By Assumption \ref{assum:smooth_lower} and the averaged update \eqref{eq:avg_masked_update},
\begin{equation}
    F(\bar{x}_{t+1})
    \le
    F(\bar{x}_t)
    -
    \eta\langle\nabla F(\bar{x}_t),h_t\rangle
    +
    \frac{L\eta^2}{2}\|h_t\|^2.
    \label{eq:smooth_descent_h}
\end{equation}
Taking expectations and substituting \eqref{eq:h_alignment} and \eqref{eq:h_second_moment} gives
\begin{equation}
\begin{aligned}
    \mathbb{E}F(\bar{x}_{t+1})
    \le\;&
    \mathbb{E}F(\bar{x}_t)
    -
    \frac{3\eta}{4}\mathbb{E}\|\nabla F(\bar{x}_t)\|^2
    +
    C\eta L^2\mathcal{E}_t
    +
    C\eta B_\mu^2 \\
    +
    C L\eta^2 d\,\mathbb{E}\|\nabla F(\bar{x}_t)\|^2
    +
    C L\eta^2 dL^2\mathcal{E}_t \\
    &+
    C L\eta^2\frac{d\sigma^2}{N}
    +
    C L\eta^2B_\mu^2.
\end{aligned}
\label{eq:descent_before_absorb}
\end{equation}
Choose the universal constant in \eqref{eq:stepsize_condition} so that $CL\eta d\le1/4$. The positive stationarity term is then at most $(\eta/4)\mathbb{E}\|\nabla F(\bar{x}_t)\|^2$. The same condition absorbs $CL\eta d$ into the disagreement coefficient, and $\eta\le L^{-1}$ absorbs the smaller bias term. Thus
\begin{equation}
    \mathbb{E}F(\bar{x}_{t+1})
    \le
    \mathbb{E}F(\bar{x}_t)
    -
    \frac{\eta}{2}\mathbb{E}\|\nabla F(\bar{x}_t)\|^2
    +
    C\eta L^2\mathcal{E}_t
    +
    C\eta B_\mu^2
    +
    C L\eta^2\frac{d\sigma^2}{N}.
    \label{eq:descent_after_absorb}
\end{equation}
Summing \eqref{eq:descent_after_absorb} from $t=0$ to $T-1$ and using $F(\bar{x}_T)\ge F^*$ gives
\begin{equation}
\begin{aligned}
    \frac{1}{T}\sum_{t=0}^{T-1}
    \mathbb{E}\|\nabla F(\bar{x}_t)\|^2
    \le\;&
    \frac{2(F(\bar{x}_0)-F^*)}{\eta T}
    +
    C L^2
    \frac{1}{T}\sum_{t=0}^{T-1}\mathcal{E}_t \\
    &+
    C B_\mu^2
    +
    C L\eta\frac{d\sigma^2}{N}.
\end{aligned}
\label{eq:before_consensus_substitution}
\end{equation}
Substitute the masked disagreement bound from Lemma \ref{lemma:masked_disagreement_energy}. This gives
\begin{equation}
\begin{aligned}
    \frac{1}{T}\sum_{t=0}^{T-1}
    \mathbb{E}\|\nabla F(\bar{x}_t)\|^2
    \le\;&
    \frac{C(F(\bar{x}_0)-F^*)}{\eta T}
    +
    C L\eta\frac{d\sigma^2}{N}
    +
    C B_\mu^2 \\
    &+
    C L^2\eta^2
    \frac{d}{q}
    \frac{M_{\mathrm{node}}^2}{(1-\rho^2)^2}
    +
    C L^2
    \frac{d}{q}
    \frac{\mathcal{E}_0}{T(1-\rho^2)},
\end{aligned}
\end{equation}
Using $B_\mu^2=\mathcal{O}(\mu^4d^2q^2)$ and
$M_{\mathrm{node}}^2=\mathcal{O}(d(G^2+\sigma^2)+\mu^4d^2q^2)$ gives the explicit finite-\(\eta\) form
\begin{equation}\label{eq:main_convergence}
\begin{aligned}
R_T&:=\frac1T\sum_{t<T}\mathbb{E}\|\nabla F(\bar{x}_t)\|^2\\
&\le\mathcal{O}\!\left(
\frac{\Delta_0}{\eta T}+\frac{\eta d\sigma^2}{N}+\mu^4d^2q^2
\right.\\[-1mm]
&\hspace{18mm}\left.
+\frac{\eta^2 d[d(G^2+\sigma^2)+\mu^4d^2q^2]}{q(1-\rho^2)^2}
+\frac{d\mathcal{E}_0}{qT(1-\rho^2)}
\right).
\end{aligned}
\end{equation}
With common initialization and \eqref{eq:canonical_mu}, for any feasible \(\eta\),
\begin{equation}\label{eq:main_common_eta_rate}
    R_T\le \mathcal{O}\!\left(
    \frac{\Delta_0}{\eta T}
    +
    \eta\frac{d\sigma^2}{N}
    +
    \frac1T
    +
    \eta^2
    \frac{d^2(G^2+\sigma^2)}{q(1-\rho^2)^2}
    \right).
\end{equation}
Taking $\eta=c_\eta T^{-1/2}$ in \eqref{eq:main_common_eta_rate}, for horizons satisfying \eqref{eq:stepsize_condition}, gives \eqref{eq:simplified_bigo_rate}. The directed-link form \eqref{eq:main_budget_convergence} then follows from the round substitution in Appendix \ref{appendix:comm_budget_rate}.
If $\sigma^2>0$ and
\begin{equation}\label{eq:main_eta_worker_scaled}
    \eta_\star=\Theta\!\left(\sqrt{\frac{N\Delta_0}{d\sigma^2T}}\right)
\end{equation}
satisfies \eqref{eq:stepsize_condition}, then
\begin{equation}\label{eq:main_worker_scaled_rate}
    R_T\le
    \mathcal{O}\!\left(
    \sqrt{\frac{\Delta_0 d\sigma^2}{NT}}
    +
    \frac1T
    +
    \frac{N\Delta_0 d(G^2+\sigma^2)}
    {q(1-\rho^2)^2\sigma^2T}
    \right).
\end{equation}

\subsection{Descent Absorption}
\label{appendix:descent_absorption_details}

We spell out the passage from \eqref{eq:descent_before_absorb} to \eqref{eq:descent_after_absorb}. The only positive term containing the stationarity measure is
\begin{equation}
    CL\eta^2d\,\mathbb E\|\nabla F(\bar x_t)\|^2.
    \label{eq:positive_gradient_descent_term}
\end{equation}
Choose the universal constant $c$ in \eqref{eq:stepsize_condition} so that $CL\eta d\le1/4$. Then \eqref{eq:positive_gradient_descent_term} is at most $(\eta/4)\mathbb E\|\nabla F(\bar x_t)\|^2$. Combining it with the coefficient $-3\eta/4$ in \eqref{eq:descent_before_absorb} leaves
\begin{equation}
    -\frac{\eta}{2}\mathbb E\|\nabla F(\bar x_t)\|^2.
    \label{eq:remaining_descent}
\end{equation}

The two disagreement terms factor as
\begin{equation}
    C\eta L^2\mathcal E_t
    +CL\eta^2dL^2\mathcal E_t
    =C\eta L^2\mathcal E_t(1+L\eta d),
    \label{eq:descent_disagreement_terms}
\end{equation}
which is at most $C'\eta L^2\mathcal E_t$ under the same condition. Likewise, since $L\eta\le1$,
\begin{equation}
    C\eta B_\mu^2+CL\eta^2B_\mu^2
    \le C'\eta B_\mu^2.
    \label{eq:bias_absorbed}
\end{equation}
The remaining stochastic term is $CL\eta^2d\sigma^2/N$. After summing over $t$ and dividing by $\eta T$, it becomes $CL\eta d\sigma^2/N$. These are precisely the three positive terms in \eqref{eq:descent_after_absorb}; all graph dependence enters later through \eqref{eq:consensus_average_final}.
\subsection{Worker-Scaled Rate Derivation}
\label{appendix:worker_scaled_derivation}

For common initialization, $\mathcal{E}_0=0$. Equation \eqref{eq:main_worker_scaled_rate} follows by setting
\begin{equation}
    A_q(\mu)
    :=
    \frac{d}{q(1-\rho^2)^2}
    \left[d(G^2+\sigma^2)+\mu^4d^2q^2\right].
    \label{eq:appendix_Aq_def}
\end{equation}
Ignoring only constants depending on $L$, \eqref{eq:main_convergence} can be written as
\begin{equation}
    R_T(\eta,\mu)
    :=
    \frac{1}{T}\sum_{t=0}^{T-1}
    \mathbb{E}\|\nabla F(\bar{x}_t)\|^2
    \le
    \mathcal{O}\!\left(\frac{\Delta_0}{\eta T}\right)
    +
    \mathcal{O}\!\left(\eta\frac{d\sigma^2}{N}\right)
    +
    \mathcal{O}\!\left(\mu^4d^2q^2\right)
    +
    \mathcal{O}\!\left(\eta^2 A_q(\mu)\right).
    \label{eq:appendix_worker_start}
\end{equation}
The first term is the usual descent term, the second term is the averaged stochastic ZO noise, the third term is the finite-difference bias, and the fourth term is the sparse decentralized consensus penalty.

Balancing the first two terms in \eqref{eq:appendix_worker_start} gives
\begin{equation}
    \frac{\Delta_0}{\eta T}
    \asymp
    \eta\frac{d\sigma^2}{N},
    \qquad
    \eta_\star
    \asymp
    \sqrt{\frac{N\Delta_0}{d\sigma^2T}},
    \label{eq:appendix_eta_balance}
\end{equation}
Substituting \eqref{eq:appendix_eta_balance} back into the first two terms yields
\begin{equation}
    \frac{\Delta_0}{\eta_\star T}
    +
    \eta_\star\frac{d\sigma^2}{N}
    =
    \mathcal{O}\!\left(
        \sqrt{\frac{\Delta_0 d\sigma^2}{NT}}
    \right).
    \label{eq:appendix_worker_main_term}
\end{equation}
The sparse-consensus term becomes
\begin{equation}
    \eta_\star^2A_q(\mu)
    =
    \mathcal{O}\!\left(
        \frac{N\Delta_0}{d\sigma^2T}
        \frac{d}{q(1-\rho^2)^2}
        \left[d(G^2+\sigma^2)+\mu^4d^2q^2\right]
    \right).
    \label{eq:appendix_worker_consensus_raw}
\end{equation}
Using the canonical smoothing choice
\begin{equation}
    \mu=\Theta\!\left((dq)^{-1/2}T^{-1/4}\right)
    \label{eq:appendix_mu_choice}
\end{equation}
gives
\begin{equation}
    \mu^4d^2q^2=\mathcal{O}(T^{-1})
    \label{eq:appendix_bias_choice}
\end{equation}
and therefore the finite-difference bias contributes $\mathcal{O}(T^{-1})$. The bias-dependent part inside \eqref{eq:appendix_worker_consensus_raw} is also of smaller order than the displayed consensus term for the regimes used in our experiments. Thus
\begin{equation}
    \eta_\star^2A_q(\mu)
    =
    \mathcal{O}\!\left(
        \frac{N\Delta_0 d(G^2+\sigma^2)}
        {q(1-\rho^2)^2\sigma^2T}
    \right)
    +
    \text{higher-order smoothing terms}.
    \label{eq:appendix_worker_consensus_simplified}
\end{equation}
Combining \eqref{eq:appendix_worker_main_term}, \eqref{eq:appendix_bias_choice}, and \eqref{eq:appendix_worker_consensus_simplified} gives
\begin{equation}
    R_T\le
    \mathcal{O}\!\left(\sqrt{\frac{\Delta_0 d\sigma^2}{NT}}\right)
    +\mathcal{O}\!\left(\frac1T\right)
    +\mathcal{O}\!\left(\frac{N\Delta_0 d(G^2+\sigma^2)}{q(1-\rho^2)^2\sigma^2T}\right).
    \label{eq:worker_speedup_canonical}
\end{equation}
The same public coordinates are used across agents, so worker averaging removes independent oracle noise but not sparse-coordinate sampling itself.

\subsection{Parameter Choice}
\label{appendix:parameter_choice_theory}

The smoothing radius $\mu$ controls finite-difference bias, the stepsize $\eta$ balances descent against stochastic and network error, and the sparsity $q$ controls communication and coordinate refresh.

For a central-difference estimator with a third-order smoothness bound, the bias scales as
\begin{equation}
    B_\mu^2
    =
    \mathcal{O}(\mu^4d^2q^2).
    \label{eq:parameter_bias_scale}
\end{equation}
The dependence on $q$ comes from $\|u_t\|^4=q^2$ and the scaling factor $d/q$ used to keep the sparse estimator unbiased to first order. If $q$ is increased while $\mu$ is held fixed, the finite-difference truncation error can grow. The canonical theoretical choice therefore sets
\begin{equation}
    \mu
    =
    \Theta\!\left(\frac{1}{\sqrt{dq}\,T^{1/4}}\right),
    \label{eq:appendix_mu_choice_expanded}
\end{equation}
which keeps \eqref{eq:parameter_bias_scale} at order $T^{-1}$. The experimental values of $\mu$ are fixed separately to accommodate loss scale, minibatch noise, and numerical precision. The theoretical scaling prescribes a smaller radius as $q$ or $d$ increases.

The stepsize has two competing constraints. The smooth descent step requires
\begin{equation}
    \eta
    \le
    c(Ld)^{-1}
    \label{eq:appendix_eta_stability_constraint}
\end{equation}
for a universal constant $c$. The stochastic rate suggests the larger balancing choice
\begin{equation}
    \eta
    \asymp
    \sqrt{\frac{N\Delta_0}{d\sigma^2T}},
    \label{eq:appendix_eta_balancing_choice}
\end{equation}
whenever this value satisfies \eqref{eq:appendix_eta_stability_constraint}. In high-dimensional full-model tuning, the stability constraint can be much more restrictive than the formal stochastic balancing rule. This is consistent with the Qwen full-model study: the estimator is multiplied by $d/q$, so a learning rate that is harmless in a low-dimensional LoRA subspace can be unstable when the trainable coordinate set has billions of dimensions.

The sparsity level $q$ creates the main design tradeoff. Per round, larger $q$ improves coordinate coverage and reduces the sparse-consensus factor $d/q$, but it also increases the value payload linearly. Under a fixed directed-link budget $\mathcal{B}$, the number of affordable rounds is approximately
\begin{equation}
    T(q)
    \approx
    \frac{\mathcal{B}}{qB_{\mathrm{val}}}.
    \label{eq:Tq_budget_simple}
\end{equation}
Substitution into \eqref{eq:main_budget_convergence} cancels the explicit $1/q$ factor in the network term. The remaining displayed terms increase with $q$ when the other constants and the prescribed horizon-dependent parameters are held fixed. This bound therefore favors smaller supports on communication alone; it does not select an interior optimum. Practical selection also accounts for the two forward queries per round, their latency, and finite-difference precision.

For iid uniform supports at a node, coordinate exposure is measured by $\tau=qT/d$. For any fixed coordinate $j$,
\begin{equation}\label{eq:coordinate_exposure_v24}
 \mathbb E\sum_{t<T}\mathbf1\{j\in S_t\}=\frac{qT}{d},
 \qquad
 \Pr\{j\notin\cup_{t<T}S_t\}=(1-q/d)^T.
\end{equation}
The first identity sums the activation probabilities; the second uses independence across rounds. Hence $d/q$ is the mean waiting time, not a deterministic coverage period. At a fixed value-only budget, $\tau\approx\mathcal B/(dB_{\mathrm{val}})$, independent of $q$ to first order. The support-size sweep in Table \ref{tab:rosen_q_sensitivity} compares the resulting optimization trajectories at matched bits.

The graph parameter $\rho$ enters through the disagreement recursion, producing $(1-\rho^2)^{-2}$ in the bound. The one-step witness in Appendix \ref{appendix:sparse_consensus} establishes the $q/d$ contraction factor; it does not establish sharpness of the full spectral dependence. Sparse momentum retains information between coordinate activations. Theorem \ref{thm:momentum_convergence} gives its convergence guarantee, and Appendix \ref{appendix:extension_experiments_v23} tests the interaction between its step size and query correlation.

\input{v28_momentum_proof}

\input{v28_correlation_appendix}
\FloatBarrier
\input{v28_correlation_extension}
\FloatBarrier

\section{Communication Accounting Details}
\label{appendix:comm_accounting}

An indexed sparse method sends both values and coordinate identifiers; ZO-COSMO communicates only values after the public random seed is fixed. For a $q$-sparse update in dimension $d$, let
\begin{equation}
    \ell_d=\lceil\log_2 d\rceil
    \label{eq:index_bits_def}
\end{equation}
be the number of bits needed to name a coordinate. The per-round directed-link costs are
\begin{equation}
    C_{\mathrm{COSMO}}=qB_{\mathrm{val}},
    \qquad
    C_{\mathrm{indexed}}=q(B_{\mathrm{val}}+\ell_d).
    \label{eq:appendix_sparse_costs}
\end{equation}
The fraction of an indexed sparse message spent only on coordinate identities is therefore
\begin{equation}
    r_{\mathrm{idx}}(d,B_{\mathrm{val}})
    =
    \frac{\ell_d}{B_{\mathrm{val}}+\ell_d}.
    \label{eq:index_fraction_def}
\end{equation}
The public-support design removes exactly this fraction of the sparse payload. The saving is modest for small dense floating-point messages, but it becomes dominant for high-dimensional models and for low-bit value quantization.

\begin{table}[H]
    \centering
    \caption{Coordinate-identity overhead in indexed sparse communication.}
    \label{tab:appendix_index_share}
    \footnotesize
    \setlength{\tabcolsep}{3pt}
    \begin{tabular}{lcccccc}
        \toprule
        Setting & Dimension $d$ & $\ell_d$ & $B_{\mathrm{val}}=32$ & $B_{\mathrm{val}}=16$ & $B_{\mathrm{val}}=8$ & $B_{\mathrm{val}}=1$ \\
        \midrule
        Qwen LoRA-scale experiment & $2.52{\times}10^6$ & 22 & 40.7\% & 57.9\% & 73.3\% & 95.7\% \\
        Qwen full-model study & $6.53{\times}10^9$ & 33 & 50.8\% & 67.3\% & 80.5\% & 97.1\% \\
        \bottomrule
    \end{tabular}
\end{table}

Table \ref{tab:appendix_index_share} quantifies the representation cost removed by index-free sparsity. In the Qwen full-model study, even with 32-bit values, an indexed sparse method spends slightly more than half of each sparse payload on coordinate names. If the value payload is quantized to 8 bits, the coordinate names become more than four fifths of the message. The same effect appears in LoRA-scale training, where the dimension is far smaller than full-model fine-tuning but still large enough that coordinate identifiers dominate low-bit sparse messages.

\begin{table}[H]
    \centering
    \caption{Representative directed-link payloads.}
    \label{tab:appendix_payload_examples}
    \footnotesize
    \setlength{\tabcolsep}{3pt}
    \begin{tabular}{lccc>{\raggedright\arraybackslash}p{5.1cm}}
        \toprule
        Setting & $d$ & $q$ & ZO-COSMO & Private-support / refresh cost \\
        \midrule
        \makecell[l]{Qwen LoRA-scale\\32-bit values} & $2.52$M & $8{,}192$ & $262.1$K bits & $\ge341.6$K entropy bound; $442.4$K coordinate list; $80.6$M dense \\
        \makecell[l]{Qwen full-model\\32-bit values} & $6.53$B & $8{,}192$ & $262.1$K bits & $\ge434.6$K entropy bound; $532.5$K coordinate list; $2.09{\times}10^{11}$ dense \\
        \bottomrule
    \end{tabular}
\end{table}

The entropy entries evaluate the lower bound $qB_{\mathrm{val}}+\log_2\binom{d}{q}$ for a uniform private support. The coordinate-list entries give the implemented format. Together they distinguish the minimum support-description cost from the cost of explicit indices.

\subsection{Public Randomness and the Communication Operator}
\label{appendix:public_randomness_operators}

Shared randomness can encode different mathematical objects. Table \ref{tab:public_randomness_operators} separates three decentralized uses. Here $\bar\delta$ is average degree, $D_G$ is graph diameter, and $b=B_{\mathrm{seed}}+B_{\mathrm{scalar}}$ is the size of one seed--scalar update. Payload is the sum of packet bits crossing edges, divided by $N$; sequential stages count the longest dependency chain within one optimization round.

\begin{table}[H]
    \centering
    \caption{Payload and link stages for public-randomness protocols.}
    \label{tab:public_randomness_operators}
    \footnotesize
    \setlength{\tabcolsep}{3.5pt}
    \begin{adjustbox}{max width=\linewidth}
    \begin{tabular}{lcccc}
        \toprule
        Protocol & Reconstructed object & Destination & Bits/node/round & Link stages \\
        \midrule
        Global-support ZO-COSMO & Current state slice & Neighbors & $\bar\delta qB_{\mathrm{val}}$ & $1$ \\
        Edge-local ZO-COSMO & Current state slice & Matched peer & $qB_{\mathrm{val}}$ & $1$ \\
        Seed-scalar flooding \citep{kim2026seedflood} & Local update & Every node & $\Omega((N-1)b)$ & $\ge D_G$ \\
        \bottomrule
    \end{tabular}
    \end{adjustbox}
\end{table}

Server-mediated methods send updates to a coordinator: DeComFL and MEERKAT reconstruct seed--scalar updates, while FeedSign aggregates seed--sign votes \citep{li2025decomfl,ran2026meerkat,cai2026feedsign}. Table \ref{tab:public_randomness_operators} compares peer-to-peer operators, whose destinations and dependency graphs differ from this server-mediated pattern.

The flooding bound concerns delivery of every source packet to every node. Each of $N$ distinct packets requires at least $N-1$ packet-edge deliveries, giving $N(N-1)b$ network bits, or $(N-1)b$ per node; a diameter endpoint requires at least $D_G$ sequential links. Batching is allowed, but each original packet must reach every node; in-network aggregation is a different operator. SeedFlood uses this dissemination for all-gather-equivalent synchronization. Its SubCGE batches updates in a shared subspace, reducing application cost from $O(Nd)$ to $O(N+rd)$ for subspace rank $r$ \citep{kim2026seedflood}. The table counts packet delivery separately from this application cost.

A seed--scalar pair can be much smaller than $q$ state values, making flooding attractive for small $N$ or inexpensive multi-hop communication. Edge-local ZO-COSMO finishes in one link stage and keeps distinct local states, with constant per-node payload on a perfect matching. The support-locality experiment in Appendix \ref{appendix:support_locality_experiment} measures the crossover at a common bit budget.

\subsection{Sparse-Message Serialization Check}
\label{appendix:serialization_diagnostic}

We measure local encoding and decoding with contiguous 32-bit floating-point values. Indexed packets additionally use \texttt{uint32} coordinates for Qwen LoRA and \texttt{uint64} for full-model Qwen. These timings cover serialization only; network transfer is measured separately below.

\begin{table}[H]
    \centering
    \caption{Sparse-message serialization check.}
    \label{tab:appendix_serialization_microbenchmark}
    \footnotesize
    \setlength{\tabcolsep}{4pt}
    \begin{tabular}{lcccc}
        \toprule
        Setting & Value-only & Indexed impl. & Min. indexed & Time ratio \\
        \midrule
        Qwen LoRA-scale & $32.0$ KiB & $64.0$ KiB & $54.0$ KiB & $3.22\times$ \\
        Qwen full-model & $32.0$ KiB & $96.0$ KiB & $65.0$ KiB & $4.04\times$ \\
        \bottomrule
    \end{tabular}
\end{table}

The ``Min. indexed'' column is the ideal bit-packed cost $q(32+\ell_d)$ rounded to bytes; the indexed implementation column is the ordinary contiguous-coordinate representation used by sparse tensor code. The exact timing ratios are machine dependent, but the payload relation is not. Even under ideal bit packing, indexed sparse communication remains larger than the value-only message; under ordinary integer-coordinate serialization, the Qwen full-model indexed payload is three times the ZO-COSMO payload before any network overhead is counted.

\subsection{Two-Process Message Transfer Check}
\label{appendix:two_process_transport}

We send LoRA-scale messages between two processes on the same host using PyTorch's Gloo TCP backend. The sender transmits contiguous tensors and waits for a one-byte acknowledgement. Sparse messages are repeated $300$ times and the dense message $30$ times within each trial; method order rotates across $31$ trials. Table \ref{tab:appendix_two_process_transport} reports median exchange time, including the acknowledgement.

\begin{table}[H]
    \centering
    \caption{Two-process transfer of Qwen LoRA-scale messages.}
    \label{tab:appendix_two_process_transport}
    \small
    \setlength{\tabcolsep}{5pt}
    \begin{tabular}{lcccc}
        \toprule
        Message & Tensors & Payload & Median exchange & Relative time \\
        \midrule
        ZO-COSMO value-only & $1$ & $32.0$ KiB & $127.2\,\mu$s & $1.00\times$ \\
        Indexed sparse & $2$ & $64.0$ KiB & $162.8\,\mu$s & $1.28\times$ \\
        Dense LoRA refresh & $1$ & $9.6$ MiB & $2.69$ ms & $21.16\times$ \\
        \bottomrule
    \end{tabular}
\end{table}

Doubling the sparse payload increases median transfer time by $28\%$, indicating a substantial fixed-latency component. The dense refresh carries $308\times$ as many bits and takes $21.16\times$ as long on this transport. We use payload bits for the method comparisons and report these same-host timings as a separate implementation measurement.

\subsection{Amortized Cost of the Shared Seed}
\label{appendix:seed_overhead}

The value-only cost in \eqref{eq:appendix_sparse_costs} assumes that the public random seed has already been fixed. It is the per-round accounting for a repeated optimization run; the one-time overhead can be written separately. Let $B_{\mathrm{seed}}$ be the number of bits used to initialize the public pseudorandom generator and let $B_{\mathrm{sync}}(T)$ be the total number of bits used for occasional round-counter or resynchronization control messages over $T$ rounds. The average directed-link payload of ZO-COSMO is then
\begin{equation}
    \bar{C}_{\mathrm{COSMO}}(T)
    =
    qB_{\mathrm{val}}
    +
    \frac{B_{\mathrm{seed}}+B_{\mathrm{sync}}(T)}{T}.
    \label{eq:seed_amortized_cost}
\end{equation}
This differs from \eqref{eq:appendix_sparse_costs} only by a term that vanishes with the number of optimization rounds. By contrast, an indexed sparse method pays coordinate identifiers every round:
\begin{equation}
    \bar{C}_{\mathrm{indexed}}(T)
    =
    qB_{\mathrm{val}}+q\ell_d.
    \label{eq:indexed_recurrent_cost}
\end{equation}
Seed setup is amortized across rounds, whereas coordinate lists recur in every message.

For example, suppose a run uses a 128-bit seed and one 32-bit epoch counter every 100 rounds. Then
\begin{equation}
    \frac{B_{\mathrm{seed}}+B_{\mathrm{sync}}(T)}{T}
    \le
    \frac{128}{T}
    +
    0.32
    \quad\text{bits per round}.
    \label{eq:seed_numeric_example}
\end{equation}
In the Qwen full-model setting with $q=8192$ and 32-bit values, the value payload is $262{,}144$ bits per round and the indexed coordinate payload is $270{,}336$ bits per round. The seed and control overhead is therefore negligible relative to both the value payload and the repeated coordinate-identifier payload.

The same accounting applies to balanced public schedules. An epoch seed suffices to regenerate the support without a coordinate list. After a missed round, the round id and seed recover the support, not the missed values or optimizer state; reliable delivery or explicit state recovery is still required.

\subsection{ZO-MARINA-type Baseline Definition}
\label{appendix:marina_like_definition}

ZO-MARINA-type alternates sparse ZO updates with periodic dense refreshes. The name refers to this communication pattern. MARINA itself uses compressed gradient differences and a variance-reduced estimator; the comparator averages independent two-query ZO updates and has no MARINA convergence guarantee. Algorithm \ref{alg:appendix_marina_like} gives the two-worker Qwen implementation.

\begin{algorithm}[htbp]
\caption{ZO-MARINA-type for two-worker Qwen}
\label{alg:appendix_marina_like}
\begin{algorithmic}[1]
\Require $K,q,\mu,\eta_t,\beta$; initialize $x_{1,0}=x_{2,0}$ and $m_{1,0}=m_{2,0}=0$
\For{round $t=0,1,\ldots,T-1$}
    \State Set $q_t=d$ if $t\bmod K=0$, and $q_t=q$ otherwise.
    \For{$i\in\{1,2\}$}
        \State Independently draw a size-$q_t$ support $S_{i,t}$ and Rademacher signs, defining $u_{i,t}$.
        \State $g_{i,t}=\frac{d}{q_t}\frac{f_i(x_{i,t}+\mu u_{i,t};\xi_{i,t})-f_i(x_{i,t}-\mu u_{i,t};\xi_{i,t})}{2\mu}u_{i,t}$.
        \State Set $P_{i,t}=\diag(\mathbf1_{S_{i,t}})$ and update $m_{i,t+1}$ by \eqref{eq:optional_subspace_momentum}.
        \State Send the active values of $h_{i,t}=P_{i,t}m_{i,t+1}$ to the other worker.
        \State Include the coordinate list on sparse rounds; dense coordinates have fixed order.
    \EndFor
    \State Both workers set $x_{i,t+1}=x_{i,t}-\frac{\eta_t}{2}(h_{1,t}+h_{2,t})$.
\EndFor
\end{algorithmic}
\end{algorithm}

On sparse rounds, each worker applies both decoded sparse updates, including overlaps. On dense rounds, it applies their full-vector average. Common initialization keeps the two replicas synchronized, so this is an update-aggregation comparison. Dense query signs are independent across workers. The LoRA experiment uses $K=25$, with the first dense refresh at $t=0$.

Each directed link carries $q(B_{\mathrm{val}}+\ell_d)$ bits on sparse rounds and $dB_{\mathrm{val}}$ bits on dense rounds. Over a complete refresh cycle, the average is
\begin{equation}
    C_{\mathrm{MARINA}}
    =
    \frac{K-1}{K}q(B_{\mathrm{val}}+\ell_d)
    +
    \frac{1}{K}dB_{\mathrm{val}}.
    \label{eq:appendix_marina_cost}
\end{equation}
For a finite horizon $T\ge1$, there are $\lceil T/K\rceil$ dense rounds, so the cumulative cost is
\begin{equation}
    C_T=\lceil T/K\rceil dB_{\mathrm{val}}
    +(T-\lceil T/K\rceil)q(B_{\mathrm{val}}+\ell_d).
\end{equation}
With $d=2{,}523{,}136$, $q=8192$, and 32-bit values, the first dense packet costs $80.740$ Mbit per link. Extrapolating the same packet format to full-model Qwen gives about $2.09{\times}10^{11}$ bits per refresh. Seed-aware sparse protocols are evaluated separately.

Equal-round comparisons fix the query count; fixed-budget comparisons also account for the affordable horizon. When the first refresh exceeds the budget, the reported state is the initialization.

\subsection{Com-DSZO Baseline Definition}
\label{appendix:com_dszo_definition}

Com-DSZO combines a two-point ZO estimator with error-compensated compressed consensus \citep{hua2026compressedzo}. We instantiate its published recursion with Top-$k$ innovation compression. In matrix form, let $X_t$ stack the node states, let $\widehat X_t$ be the reconstructed state, let $B_t$ accumulate disagreement, and let $Q_t$ be the compressed innovation. One round is
\begin{align}
    \widehat X_{t+1}&=\widehat X_t+\psi Q_t, \\
    B_{t+1}&=B_t+\psi(I-W)Q_t, \\
    X_{t+1}&=X_t-\gamma B_{t+1}-\eta G_t^{\mathrm{ZO}}, \\
    Q_{t+1}&=\operatorname{Top}_q(X_{t+1}-\widehat X_{t+1}).
    \label{eq:com_dszo_instantiation}
\end{align}
The Rosenbrock problem is unconstrained, so the projection is the identity. We use the same two function evaluations and dense Rademacher ZO estimator as the dense baselines, $\psi=0.5$, $\gamma=0.1$, and $\eta=2.5\times10^{-3}$. Top-$q$ selects coordinates from the innovation, giving a message cost of $q(B_{\mathrm{val}}+\ell_d)$ bits. The baseline also stores dense reconstruction and disagreement arrays, $\widehat X_t$ and $B_t$. This comparison tests adaptive compression with error compensation. The separate indexed control isolates representation cost by using the ZO-COSMO trajectory and charging its coordinate list.

\subsection{Communication-Budget Rate}
\label{appendix:comm_budget_rate}
\label{appendix:budget_rate}

Under a directed-link budget $\mathcal{B}\ge 2qB_{\mathrm{val}}$, choose a horizon $T=\lfloor\mathcal{B}/(qB_{\mathrm{val}})\rfloor$ satisfying Theorem \ref{thm:budget_convergence}, with its prescribed stepsize and smoothing radius. Let $R_{\mathcal{B}}^{\mathrm{COSMO}}$ denote the resulting averaged gradient norm. Then
\begin{equation}\label{eq:main_budget_convergence}
    R_{\mathcal{B}}^{\mathrm{COSMO}}\le
    \mathcal{O}\!\left(
    \left(\Delta_0+\frac{d\sigma^2}{N}\right)\sqrt{\frac{qB_{\mathrm{val}}}{\mathcal{B}}}
    +\frac{qB_{\mathrm{val}}}{\mathcal{B}}
    +\frac{d^2(G^2+\sigma^2)B_{\mathrm{val}}}{(1-\rho^2)^2\mathcal{B}}
    \right).
\end{equation}

The communication-budget form of Theorem \ref{thm:budget_convergence} follows by substituting the number of affordable rounds into the iteration bound. Let $\mathcal{B}$ denote the total number of transmitted bits per directed link. For ZO-COSMO,
\begin{equation}
    T_{\mathrm{COSMO}}
    =
    \left\lfloor
    \frac{\mathcal{B}}{qB_{\mathrm{val}}}
    \right\rfloor.
    \label{eq:budget_T_cosmo}
\end{equation}
Using common initialization and the canonical smoothing choice, the simplified bound \eqref{eq:simplified_bigo_rate} has the schematic form
\begin{equation}
    R_T
    \le
    \mathcal{O}\!\left(\frac{a}{\sqrt{T}}\right)
    +
    \mathcal{O}\!\left(\frac{1}{T}\right)
    +
    \mathcal{O}\!\left(\frac{b}{qT}\right),
    \label{eq:budget_schematic_iteration}
\end{equation}
where
\begin{equation}
    a:=\Delta_0+\frac{d\sigma^2}{N},
    \qquad
    b:=\frac{d^2(G^2+\sigma^2)}{(1-\rho^2)^2}.
    \label{eq:budget_ab_def}
\end{equation}
Substituting \eqref{eq:budget_T_cosmo} into \eqref{eq:budget_schematic_iteration} gives
\begin{equation}
    R_{\mathcal{B}}^{\mathrm{COSMO}}
    \le
    \mathcal{O}\!\left(
        a\sqrt{\frac{qB_{\mathrm{val}}}{\mathcal{B}}}
    \right)
    +
    \mathcal{O}\!\left(
        \frac{qB_{\mathrm{val}}}{\mathcal{B}}
    \right)
    +
    \mathcal{O}\!\left(
        \frac{d^2(G^2+\sigma^2)B_{\mathrm{val}}}
        {(1-\rho^2)^2\mathcal{B}}
    \right).
    \label{eq:budget_cosmo_rate}
\end{equation}
The network term is first-order independent of $q$ after converting rounds to bits: larger supports mix more coordinates per round and consume proportionally more communication. The other displayed terms favor the additional rounds purchased by smaller supports. Appendix \ref{appendix:parameter_choice_theory} separates this prediction from query cost and finite-precision considerations in practical parameter selection.

To isolate representation cost, consider transmitting the same $q$-sparse round with an explicit coordinate list. With the same value precision, the budget purchases
\begin{equation}
    T_{\mathrm{idx}}
    =
    \left\lfloor
    \frac{\mathcal{B}}
    {q(B_{\mathrm{val}}+\ell_d)}
    \right\rfloor
    \approx
    \frac{T_{\mathrm{COSMO}}}
    {1+\ell_d/B_{\mathrm{val}}}.
    \label{eq:budget_T_indexed}
\end{equation}
Thus index metadata directly reduces the number of rounds available under the same bit budget. If the update sequence has the same iteration-wise bound and differs only in its encoding, the leading stochastic term is inflated by
\begin{equation}
    \sqrt{1+\frac{\ell_d}{B_{\mathrm{val}}}},
    \label{eq:budget_index_sqrt_factor}
\end{equation}
and the lower-order $1/T$ terms are inflated by
\begin{equation}
    1+\frac{\ell_d}{B_{\mathrm{val}}}.
    \label{eq:budget_index_linear_factor}
\end{equation}
These factors quantify the effect of encoding the same update sequence. For full-model Qwen with 32-bit values, \eqref{eq:budget_index_linear_factor} is $2.03$ before any dense refresh is counted. Seed-aware sparse encodings provide a second value-only reference, as evaluated in the Qwen controls.

The same calculation covers a mask refreshed every $K$ rounds. A coordinate-list implementation has amortized multiplier
\begin{equation}
    m_{\mathrm{adapt}}(K)=1+\frac{\ell_d}{KB_{\mathrm{val}}},
    \label{eq:periodic_adaptive_multiplier}
\end{equation}
while any zero-error code is lower bounded by $1+H(S\mid R)/(KqB_{\mathrm{val}})$. For full-model Qwen, the coordinate-list factors are $2.03$, $1.10$, and $1.02$ for $K=1,10,50$. Long mask reuse therefore makes metadata inexpensive when one stable subspace is adequate. ZO-COSMO addresses the complementary case in which supports continue to rotate for coordinate coverage and no data-dependent mask-selection state is maintained.

For the ZO-MARINA-type baseline, combining \eqref{eq:appendix_marina_cost} with the value-only ZO-COSMO cost gives the per-round cost multiplier
\begin{equation}
    m_{\mathrm{MARINA}}
    =
    \frac{C_{\mathrm{MARINA}}}{qB_{\mathrm{val}}}
    =
    \frac{K-1}{K}
    \left(1+\frac{\ell_d}{B_{\mathrm{val}}}\right)
    +
    \frac{d}{Kq}.
    \label{eq:marina_round_multiplier}
\end{equation}
The first term is the indexed sparse overhead; the second term is the amortized dense refresh. Under a fixed bit budget, ZO-COSMO can take approximately $m_{\mathrm{MARINA}}$ times as many sparse rounds as this ZO-MARINA-type baseline.

\begin{table}[htbp]
    \centering
    \caption{Round-cost multipliers relative to ZO-COSMO.}
    \label{tab:appendix_round_multipliers}
    \small
    \setlength{\tabcolsep}{4pt}
    \begin{tabular}{lcccc}
        \toprule
        Setting & $d$ & $q$ & \makecell{Indexed sparse\\multiplier} & \makecell{ZO-MARINA-type\\multiplier, $K=25$} \\
        \midrule
        Qwen LoRA-scale & $2.52{\times}10^6$ & $8{,}192$ & $1.69{\times}$ & $13.94{\times}$ \\
        Qwen full-model & $6.53{\times}10^9$ & $8{,}192$ & $2.03{\times}$ & $\approx 3.19{\times}10^4{\times}$ \\
        \bottomrule
    \end{tabular}
\end{table}

The refresh cost in \eqref{eq:marina_round_multiplier} grows as $d/(Kq)$. The table uses the LoRA period $K=25$; its full-model row extrapolates payload cost. In Qwen LoRA, the first dense packet exceeds the fixed comparison budget, so the corresponding entry reports initialization. Equal-round SST-2 results separately evaluate the trained refresh method with its full cumulative payload.

\section{Experimental Details and Reproducibility}
\label{appendix:experimental_details}

We report all experiments using cumulative directed-link communication. For a target budget, each curve is evaluated at the latest logged checkpoint whose cumulative bit count does not exceed that budget. This avoids interpolating across different communication schedules. Sparse baselines that reveal their chosen coordinates are charged for both values and coordinate identifiers. ZO-COSMO is charged only for the transmitted values after the shared seed is fixed. For the ZO-MARINA-type baseline, we charge indexed sparse communication on ordinary rounds and the full dense vector on refresh rounds.

\subsection{Fixed-Budget Evaluation Protocol}
\label{appendix:fixed_budget_protocol}

Every run logs an iteration counter, a method name, a seed, the graph topology, the active support size $q$, the target metric, and the cumulative communication cost per node. Let $r$ index a single method/seed/topology/$q$ run and let
\begin{equation}
    \big\{
        (b_{r,0},m_{r,0}),\ldots,(b_{r,L_r},m_{r,L_r})
    \big\}
    \label{eq:logged_budget_metric_pairs}
\end{equation}
be its logged cumulative-bit and metric pairs, sorted by $b_{r,\ell}$. For a target budget $B$, the reported fixed-budget value is
\begin{equation}
    \widehat{m}_r(B)
    =
    m_{r,\ell^\star},
    \qquad
    \ell^\star
    =
    \max\{\ell:\ b_{r,\ell}\le B\}.
    \label{eq:fixed_budget_selector}
\end{equation}
If no checkpoint satisfies the budget, the run is excluded from that budget rather than extrapolated backwards. In the reported tables this case does not occur for the displayed budgets. The aggregate value is then the mean and standard deviation of $\widehat{m}_r(B)$ over seeds:
\begin{equation}
    \bar{m}(B)
    =
    \frac{1}{|\mathcal{R}|}
    \sum_{r\in\mathcal{R}}\widehat{m}_r(B),
    \qquad
    s(B)^2
    =
    \frac{1}{|\mathcal{R}|-1}
    \sum_{r\in\mathcal{R}}
    \big(\widehat{m}_r(B)-\bar{m}(B)\big)^2.
    \label{eq:fixed_budget_mean_std}
\end{equation}
The selected checkpoint lies within the budget. We use its measured value directly, without interpolating between different communication schedules.

The per-node communication counter follows the graph. A directed message with payload cost $C$ sent over every directed edge contributes approximately $\bar{d}_{\mathrm{graph}}C$ bits per node, where $\bar{d}_{\mathrm{graph}}$ is the average degree. The code records this value through the same accounting routine for every method. Thus a sparse value-only round, an indexed sparse round, and a dense refresh differ only through the payload formula:
\begin{equation}
    C_{\mathrm{value}}=qB_{\mathrm{val}},
    \qquad
    C_{\mathrm{indexed}}=q(B_{\mathrm{val}}+\ell_d),
    \qquad
    C_{\mathrm{dense}}=dB_{\mathrm{val}}.
    \label{eq:experiment_payload_rules}
\end{equation}
Fixed-budget plots use cumulative bits; equal-round controls show the corresponding optimization dynamics.

\begin{algorithm}[htbp]
\caption{Fixed-Budget Table Extraction}
\label{alg:fixed_budget_table}
\begin{algorithmic}[1]
\Require Raw logs with cumulative bits $b_{r,\ell}$ and metric $m_{r,\ell}$; budgets $\mathcal{B}$
\For{budget $B\in\mathcal{B}$}
    \For{each method/seed/topology/$q$ run $r$}
        \State Select $\ell^\star=\max\{\ell:b_{r,\ell}\le B\}$
        \State Record $\widehat{m}_r(B)=m_{r,\ell^\star}$
    \EndFor
    \State Report the seed mean and standard deviation of $\widehat{m}_r(B)$
\EndFor
\end{algorithmic}
\end{algorithm}

\begin{table}[htbp]
    \centering
    \caption{Core experimental hyperparameters.}
    \label{tab:appendix_hyperparams}
    \footnotesize
    \setlength{\tabcolsep}{3pt}
    \begin{adjustbox}{max width=\linewidth}
    \begin{tabular}{lcccccc}
        \toprule
        Experiment & $d$ & $N$ & $q$ & Steps & $\eta$ & $\mu$ \\
        \midrule
        Rosenbrock main & 20 & 10 & 1 & 3000 & $2.5{\times}10^{-3}$ & $5{\times}10^{-3}$ \\
        Rosenbrock $q$ sweep & 20 & 10 & $\{1,2,4,8,16\}$ & 3000 & $8{\times}10^{-4}$ & $5{\times}10^{-3}$ \\
        Network stress & 128 & $\{4,8,16,32,64\}$ & 16 & 1200 & $4{\times}10^{-4}$ & $5{\times}10^{-3}$ \\
        Support locality & 128 & $\{8,16,32,64\}$ & 16 & 2200 & $4{\times}10^{-4}$ & $5{\times}10^{-3}$ \\
        Qwen LoRA confirmatory & $2.52{\times}10^6$ & 2 & 8192 & $59/100$ & $5{\times}10^{-6}$--$5{\times}10^{-5}$ & $10^{-3}$ \\
        Qwen topology diagnostic & $2.52{\times}10^6$ & $\{2,4,8\}$ & 8192 & 100 & $5{\times}10^{-5}$ & $10^{-3}$ \\
        Qwen edge-local study & $2.52{\times}10^6$ & $\{4,8\}$ & 8192 & $14$--$100$ & $5{\times}10^{-5}$ & $10^{-3}$ \\
        Qwen full model & $6.53{\times}10^9$ & 1 & 8192 & 200 & $10^{-8}$ & $10^{-3}$ \\
        \bottomrule
    \end{tabular}
    \end{adjustbox}
\end{table}

Unless otherwise stated, sparse payloads use 32-bit values and momentum uses $\beta=0.9$. The Rosenbrock main comparison selects $\eta$ from $\{8\times10^{-4},1.2\times10^{-3},1.6\times10^{-3},2\times10^{-3},2.5\times10^{-3},3\times10^{-3}\}$ on pilot seed 0 and reports disjoint seeds 1--5. The support-size ablation fixes one learning rate. The network stress and support-locality tests use the core method ($\beta=0$) and heterogeneity scale $0.3$. Edge supports are generated from $(s_0,t,\min\{i,j\},\max\{i,j\})$; no edge identifier or coordinate list is transmitted.

The confirmatory QNLI study uses $\eta=5\times10^{-5}$ for both ZO-COSMO and Rand-$k$. The zero-momentum controls use either that nominal step or $5\times10^{-6}$ to match the initial scale $(1-\beta)\eta$. ZO-MARINA-type uses $2.5\times10^{-5}$. The topology and edge-local studies use five seeds and charge every repeated exchange. The fixed-budget study uses $26.214$M bits per node. The same-matching extension runs to $1540$ rounds with either a constant step or a common tenfold reduction after round $100$; all other optimization settings are unchanged. Both report full-validation accuracy and label NLL.

The Rosenbrock experiment permits direct runs of both sparse and dense baselines on the same nonconvex objective and query budget. The indexed control isolates recurring metadata cost, while Com-DSZO tests adaptive, error-compensated compression.

For node $i$, the local Rosenbrock objective is
\begin{equation}
    f_i(x)
    =
    \sum_{r=1}^{d-1}
    \left[
        \alpha
        \left(
            z_{i,r+1}-z_{i,r}^2
        \right)^2
        +
        (1-z_{i,r})^2
    \right],
    \qquad
    z_i=x-s_i,
    \label{eq:appendix_rosenbrock_objective}
\end{equation}
where $\alpha=2$ and $s_i$ is a small node-specific shift. The shift scale is $0.02$ in the reported runs. This creates mild heterogeneity without changing the qualitative geometry of the objective. We report both the objective at the network average and the average local objective in the logs; the main paper uses the objective at the network average because it is the quantity controlled by the convergence theorem.

All Rosenbrock runs use the same initialization, smoothing radius, value precision, graph, and oracle budget within a seed. In the main comparison, the indexed control copies the ZO-COSMO support and update exactly and changes only the message cost. Com-DSZO instead uses dense local ZO directions and Top-$k$ innovation compression as defined in Appendix \ref{appendix:com_dszo_definition}. The separate support-size ablation retains the earlier independent Rand-$k$ comparator, a no-momentum control, and dense ZO-DSGD.

\begin{table}[htbp]
    \centering
    \caption{Rosenbrock $q$-sensitivity at fixed bit budgets with the common step size $\eta=8\times10^{-4}$. Cells report ZO-COSMO / indexed Rand-$k$, mean $\pm$ standard deviation over five seeds; lower is better.}
    \label{tab:rosen_q_sensitivity}
    \footnotesize
    \setlength{\tabcolsep}{3pt}
    \begin{adjustbox}{width=\linewidth}
    \begin{tabular}{c c c c}
        \toprule
        $q$ & $0.3$M bits & $1$M bits & $3$M bits \\
        \midrule
        1  & $0.131\pm0.018$ / $0.145\pm0.017$ & $0.119\pm0.009$ / $0.113\pm0.006$ & $0.119\pm0.009$ / $0.113\pm0.006$ \\
        2  & $0.632\pm0.314$ / $0.866\pm0.384$ & $0.121\pm0.011$ / $0.113\pm0.008$ & $0.121\pm0.011$ / $0.113\pm0.008$ \\
        4  & $7.283\pm0.322$ / $8.054\pm0.540$ & $0.172\pm0.023$ / $0.192\pm0.023$ & $0.121\pm0.009$ / $0.115\pm0.009$ \\
        8  & $13.971\pm0.306$ / $14.689\pm0.569$ & $1.630\pm0.594$ / $2.679\pm0.947$ & $0.124\pm0.012$ / $0.122\pm0.015$ \\
        16 & $22.680\pm3.168$ / $23.303\pm0.827$ & $9.385\pm0.857$ / $10.234\pm0.643$ & $0.288\pm0.058$ / $0.401\pm0.131$ \\
        \bottomrule
    \end{tabular}
    \end{adjustbox}
\end{table}

At $q=1$ and $q=2$, both methods take many rounds at the larger budgets and approach saturation. At $q=8$ and $q=16$, larger packets leave fewer rounds and the fixed-budget gap is clearer (Table \ref{tab:rosen_q_sensitivity}). The sweep illustrates how the value of compact encoding changes with support size, complementing the identical-update control in the main experiment.

For all decentralized experiments, we use symmetric connected graphs and Metropolis mixing weights. Given an adjacency matrix $A$, the off-diagonal weights are
\begin{equation}
    W_{ij}
    =
    \frac{1}{1+\max\{\deg(i),\deg(j)\}}
    \quad\text{if } A_{ij}=1,\ i\ne j,
    \label{eq:appendix_metropolis_offdiag}
\end{equation}
and $W_{ii}=1-\sum_{j\ne i}W_{ij}$. The reported spectral parameter is computed from the centered mixing matrix:
\begin{equation}
    \rho
    =
    \rho_{\mathrm{spec}}\!\left(
        W-\frac{1}{N}\mathbf{1}\mathbf{1}^{\top}
    \right),
    \label{eq:appendix_rho_compute}
\end{equation}
where $\rho_{\mathrm{spec}}(\cdot)$ is the spectral radius. Ring, grid, and Erd\H{o}s-R\'enyi graphs therefore differ mainly through this mixing parameter and average degree. The Erd\H{o}s-R\'enyi graphs are resampled until connected, with edge probability $0.4$.

The network stress test changes both $N$ and the graph family while holding the objective dimension, support size, step size, and horizon fixed. It uses $d=128$, $q=16$, $T=1200$, five seeds, and node shifts of scale $0.3$. Complete graphs provide the full-mixing reference; grids grow in two dimensions; rings deliberately drive the spectral gap toward zero. At $N=64$, their respective $\rho$ values are $0$, $0.9677$, and $0.9968$. All three reduce the objective to about $6.1\%$ of its initial value. The network-average objective remains within $0.1\%$ across graph families, while disagreement records the predicted deterioration on the ring.

\begin{figure}[htbp]
    \centering
    \includegraphics[width=0.78\linewidth]{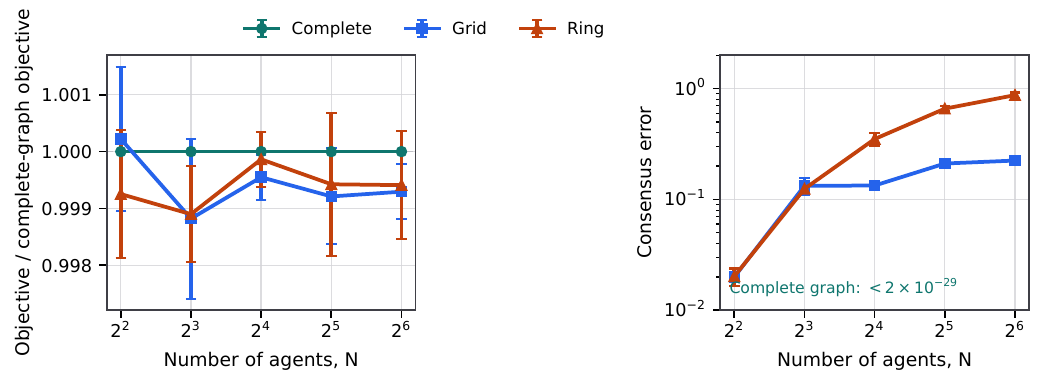}
    \caption{Five-seed topology stress test for the global-support update.}
    \label{fig:appendix_network_stress}
\end{figure}

\subsection{Support-Locality Experiment}
\label{appendix:support_locality_experiment}

This experiment compares the communication operators in Algorithm \ref{alg:edge_local_cosmo}. All methods use the same heterogeneous Rosenbrock network, $d=128$, $q=16$, two function queries per node and round, no momentum, and five seeds on even rings. Global-support ZO-COSMO sends $q$ values to both ring neighbors, costing $2qB_{\mathrm{val}}=1024$ bits per node and round. Edge-local ZO-COSMO alternates the ring's two perfect matchings and sends one $q$-value message, costing $512$ bits per node and round.

The seed--scalar control uses ideal all-gather communication. Each node forms a dense seed-reconstructible Rademacher ZO update, and every node applies the exact average after diameter-depth dissemination. We charge the spanning-tree lower bound $(N-1)(64+32)$ bits per node for a 64-bit seed and 32-bit scalar. This operator-level control uses exact aggregation and the minimum packet-edge count from Appendix \ref{appendix:public_randomness_operators}; SeedFlood's SubCGE implementation is not reproduced.

\begin{figure}[htbp]
    \centering
    \includegraphics[width=0.96\linewidth]{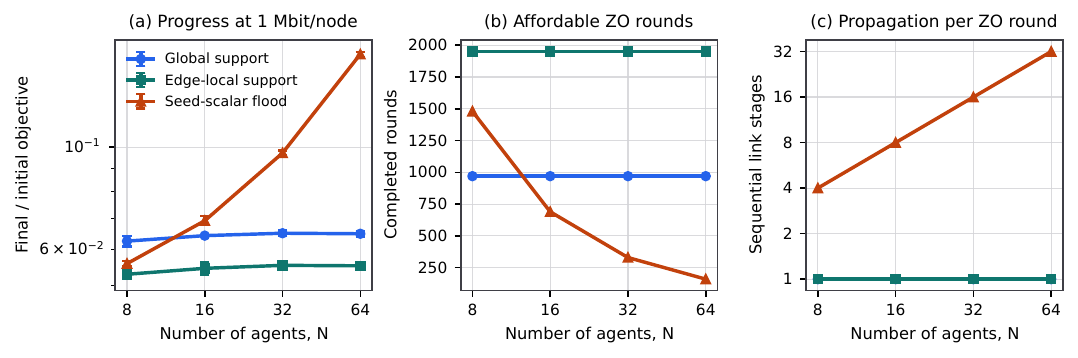}
    \caption{One-hop support locality versus graph-wide seed--scalar dissemination. Bars show one standard deviation over five seeds.}
    \label{fig:appendix_support_locality}
\end{figure}

\begin{table}[htbp]
    \centering
    \caption{Relative objective at $1$ Mbit per node. Parentheses give completed logged rounds.}
    \label{tab:appendix_support_locality}
    \footnotesize
    \setlength{\tabcolsep}{4pt}
    \begin{tabular}{cccc}
        \toprule
        $N$ & Global support & Edge-local support & Seed--scalar flood \\
        \midrule
        8  & $0.0625\pm0.0017$ (970) & $\mathbf{0.0530\pm0.0013}$ (1950) & $0.0558\pm0.0007$ (1480) \\
        16 & $0.0643\pm0.0003$ (970) & $\mathbf{0.0546\pm0.0017}$ (1950) & $0.0692\pm0.0016$ (690) \\
        32 & $0.0650\pm0.0010$ (970) & $\mathbf{0.0554\pm0.0006}$ (1950) & $0.0971\pm0.0012$ (330) \\
        64 & $0.0649\pm0.0010$ (970) & $\mathbf{0.0552\pm0.0004}$ (1950) & $0.1593\pm0.0021$ (160) \\
        \bottomrule
    \end{tabular}
\end{table}

At $N=8$, the compact flood packet offsets its four sequential ring stages and remains competitive. Its payload then grows linearly with $N$ and its propagation depth grows as $N/2$: by $N=64$, only $160$ logged rounds fit the budget. Edge-local support uses one message and one neighbor stage per active node and is best at every tested $N$. Global support is also one-hop, but spends twice the ring payload to mix both neighbors in each round. These results quantify the communication benefit of local mixing as the ring grows.

\paragraph{Same-matching support control.}
We fix the alternating matching and give every edge the same public support and signs. Both matching variants cost 512 bits per node and round, with identical query counts, initialization, and $\eta=4\times10^{-4}$. Their objectives are close at $N=8,16$, separating the scheduling gain from any additional benefit of independent supports. At $N=4$, the common-support variants are unstable on one of five seeds. Table \ref{tab:matching_synthetic_v19} includes all five seeds with unclipped updates.

\begin{table}[htbp]
\centering
\caption{Same-matching control at 1 Mbit per node, five-seed mean $\pm$ std of $F(\bar x_T)/F(\bar x_0)$.}
\label{tab:matching_synthetic_v19}
\small
\setlength{\tabcolsep}{4pt}
\begin{tabular}{rccc}
\toprule
$N$ & All-neighbor, common & Matching, common & Matching, edge-local \\
\midrule
4 & $(1.18\pm2.63)10^{86}$ & $(7.61\pm17.01)10^{106}$ & $0.0466\pm0.0017$ \\
8 & $0.0626\pm0.0017$ & $0.0520\pm0.0013$ & $0.0529\pm0.0013$ \\
16 & $0.0644\pm0.0003$ & $0.0543\pm0.0015$ & $0.0545\pm0.0017$ \\
\bottomrule
\end{tabular}
\end{table}

\subsection{Qwen2-7B Experiments}

\paragraph{Task selection.}
SST-2 starts near its accuracy ceiling, so we fixed a second-task rule before running any method. The candidate pool was QNLI, BoolQ, MRPC, and RTE: binary accuracy tasks with single-token label verbalizers and at most $6000$ validation examples. We evaluated the unmodified Qwen2-7B model on each full validation set, retained tasks with accuracy in $[55,90)\%$, and selected the eligible task with the largest validation set. Table \ref{tab:qwen_task_selection} records the complete screen; the rule selects QNLI.

\begin{table}[H]
    \centering
    \caption{Qwen2-7B task screen before method training.}
    \label{tab:qwen_task_selection}
    \small
    \setlength{\tabcolsep}{6pt}
    \begin{tabular}{lccc}
        \toprule
        Task & Validation examples & Initial accuracy & Selected \\
        \midrule
        QNLI & $5463$ & $60.79\%$ & Yes \\
        BoolQ & $3270$ & $81.01\%$ & No \\
        MRPC & $408$ & $69.12\%$ & No \\
        RTE & $277$ & $78.34\%$ & No \\
        \bottomrule
    \end{tabular}
\end{table}

The confirmatory QNLI study evaluates all $5463$ validation examples for ten paired seeds. Its initial accuracy is $60.79\%$. We fix the directed-link budget at $26.214$ Mbit, enough for $100$ value-only ZO-COSMO rounds. Indexed Rand-$k$ costs $0.442$ Mbit per round and therefore completes $59$ rounds without exceeding the same budget. Method selection and optimization settings are paired by seed; the primary endpoint is full-validation accuracy.

\paragraph{Two-worker protocol.}
The LoRA benchmark simulates two decentralized workers. Each worker has an independent LoRA state and momentum, receives a disjoint contiguous training shard, uses batch size $4$ and maximum sequence length $256$, and participates in exact complete-graph mixing. The network-average model is evaluated on all $872$ SST-2 examples or all $5463$ QNLI examples. Rank-$8$ adapters on the query and value projections give $d=2{,}523{,}136$ trainable coordinates. We use $q=8192$, $\mu=10^{-3}$, and subspace momentum $\beta=0.9$. ZO-COSMO and Indexed Rand-$k$ use $\eta=5{\times}10^{-5}$; the stable $\eta=2.5{\times}10^{-5}$ setting is used for ZO-MARINA-type, whose dense refresh period is $25$ rounds.

Each worker makes two forward queries on its local minibatch. ZO-COSMO uses one public support and exchanges ordered update values. Rand-$k$ uses independent node supports generated from $(s_0,t,i)$; its original encoding also sends their indices. Both replicas apply the same averaged sparse updates. Common initialization therefore keeps their trainable states identical, even when their momentum and queried supports differ. The $N>2$ experiments below test disagreement reduction between distinct peer states.

\paragraph{Seed-aware re-encoding.}
This post-hoc control exposes the existing support-generator seed and node identifier to the receiver. It reconstructs each independent support while preserving the directions, values, update order, and trajectory. An existing 100-round Rand-$k$ run then costs $100qB_{\mathrm{val}}=26.2144$ Mbit per link instead of $100q(B_{\mathrm{val}}+22)=44.2368$ Mbit. The analysis reuses the original training logs. A packet-level test verifies exact float32 trajectory agreement between encodings. Seed setup and packet headers are accounted for separately from recurring value fields.

\begin{table}[H]
    \centering
    \caption{QNLI at $B=26.214$ Mbit per directed link (ten seeds). Seed-aware Rand-$k$ re-encodes existing runs.}
    \label{tab:qwen_diagnostics}
    \renewcommand{\arraystretch}{1.12}
    \small
    \setlength{\tabcolsep}{3.5pt}
    \begin{tabular}{lccc}
        \toprule
        Method & Mbit/round & Rounds & Accuracy (\%) \\
        \midrule
        Explicit-index Rand-$k$ & $0.442$ & $59$ & $61.75{\pm}3.41$ \\
        Seed-aware Rand-$k$ & $0.262$ & $100$ & $66.48{\pm}7.36$ \\
        \rowcolor{gray!10}
        \textbf{ZO-COSMO} & $0.262$ & $100$ & $65.40{\pm}3.57$ \\
        \bottomrule
    \end{tabular}
\end{table}

The supplied CSVs contain individual-seed measurements; figures show mean and standard deviation.

Explicit-index Rand-$k$ leaves $0.44\%$ of the budget unused. The paired ZO-COSMO contrast is $+3.65$ points (95\% percentile-bootstrap CI $[0.49,6.86]$, $100{,}000$ resamples). Against seed-aware Rand-$k$, it is $-1.08$ points (CI $[-5.85,3.67]$), with mean query loss lower by $0.091$ (CI $[0.006,0.175]$). The accuracy interval spans zero. ZO-MARINA-type is reported at its $60.79\%$ initialization because its first refresh costs $80.7$ Mbit, above the comparison budget.

\begin{table}[H]
    \centering
    \caption{QNLI after 100 rounds (ten seeds). The first two rows are identical trajectories under different encodings.}
    \label{tab:qwen_qnli_equal_rounds}
    \small
    \setlength{\tabcolsep}{5pt}
    \begin{tabular}{lcccc}
        \toprule
        Method & Rounds & Mbit/link & Accuracy (\%) & Query loss \\
        \midrule
        Indexed Rand-$k$ & $100$ & $44.237$ & $66.48{\pm}7.36$ & $6.700{\pm}0.178$ \\
        Seed-aware Rand-$k$ & $100$ & $26.214$ & $66.48{\pm}7.36$ & $6.700{\pm}0.178$ \\
        \rowcolor{gray!10}
        \textbf{ZO-COSMO} & $100$ & $\mathbf{26.214}$ & $65.40{\pm}3.57$ & $\mathbf{6.609{\pm}0.190}$ \\
        \bottomrule
    \end{tabular}
\end{table}

\paragraph{Momentum and effective step size.}
On a coordinate with zero momentum state, the first update is $(1-\beta)g$. Thus removing $\beta=0.9$ while holding $\eta$ fixed makes the first update ten times larger. Table \ref{tab:qwen_qnli_momentum_ablation} reports both that stress test and a control with $\eta$ reduced by ten, which isolates momentum memory from this deterministic rescaling.

\begin{table}[H]
    \centering
    \caption{QNLI momentum controls after $100$ rounds (ten seeds). The matched-step comparison gives a $3.92$-point momentum gain, with 95\% paired bootstrap CI $[1.67,6.08]$.}
    \label{tab:qwen_qnli_momentum_ablation}
    \small
    \setlength{\tabcolsep}{5pt}
    \begin{tabular}{lcccc}
        \toprule
        Variant & $\beta$ & $\eta$ & $(1-\beta)\eta$ & Accuracy (\%) \\
        \midrule
        Same nominal step & $0$ & $5{\times}10^{-5}$ & $5{\times}10^{-5}$ & $50.40{\pm}0.61$ \\
        Core, matched first step & $0$ & $5{\times}10^{-6}$ & $5{\times}10^{-6}$ & $61.48{\pm}3.31$ \\
        \rowcolor{gray!10}
        \textbf{ZO-COSMO} & $0.9$ & $5{\times}10^{-5}$ & $5{\times}10^{-6}$ & $\mathbf{65.40{\pm}3.57}$ \\
        \bottomrule
    \end{tabular}
\end{table}

\begin{table}[H]
    \centering
    \caption{SST-2 after $100$ rounds, mean $\pm$ standard deviation over five seeds.}
    \label{tab:qwen_two_worker_lora}
    \renewcommand{\arraystretch}{1.12}
    \small
    \setlength{\tabcolsep}{3pt}
    \begin{tabular}{l c c c c c}
        \toprule
        \textbf{Method} & $\boldsymbol{\eta}$ & \makecell{\textbf{Final acc.}\\\textbf{(\%)}} & \makecell{\textbf{Last-50}\\\textbf{query loss}} & \makecell{\textbf{Total}\\\textbf{Mbit/link}} & \textbf{Relative} \\
        \midrule
        \rowcolor{gray!10}
        \textbf{ZO-COSMO (Ours)} & $5{\times}10^{-5}$ & $94.13{\pm}0.42$ & $0.741{\pm}0.165$ & $\mathbf{26.2}$ & $\mathbf{1.00\times}$ \\
        Indexed Rand-$k$ & $5{\times}10^{-5}$ & $94.15{\pm}0.63$ & $\mathbf{0.725{\pm}0.060}$ & $44.2$ & $1.69\times$ \\
        ZO-MARINA-type & $2.5{\times}10^{-5}$ & $\mathbf{94.84{\pm}0.24}$ & $0.976{\pm}0.301$ & $365.4$ & $13.94\times$ \\
        \bottomrule
    \end{tabular}
\end{table}

\begin{figure}[H]
    \centering
    \includegraphics[width=0.86\linewidth]{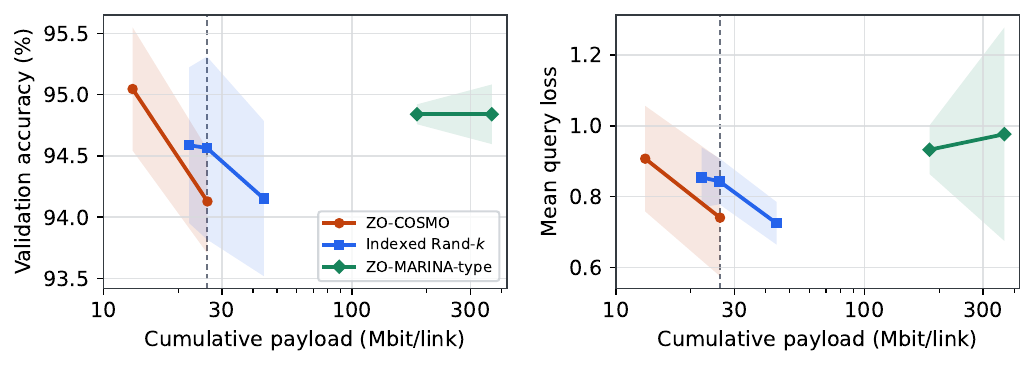}
    \caption{Two-worker Qwen2-7B SST-2 results over five seeds. The dashed line marks the $26.214$M-bit comparison budget.}
    \label{fig:qwen_two_worker_sst2}
\end{figure}

\begin{figure}[H]
    \centering
    \begin{subfigure}{0.48\textwidth}
        \centering
        \includegraphics[width=\linewidth,trim={0 0 0 24bp},clip]{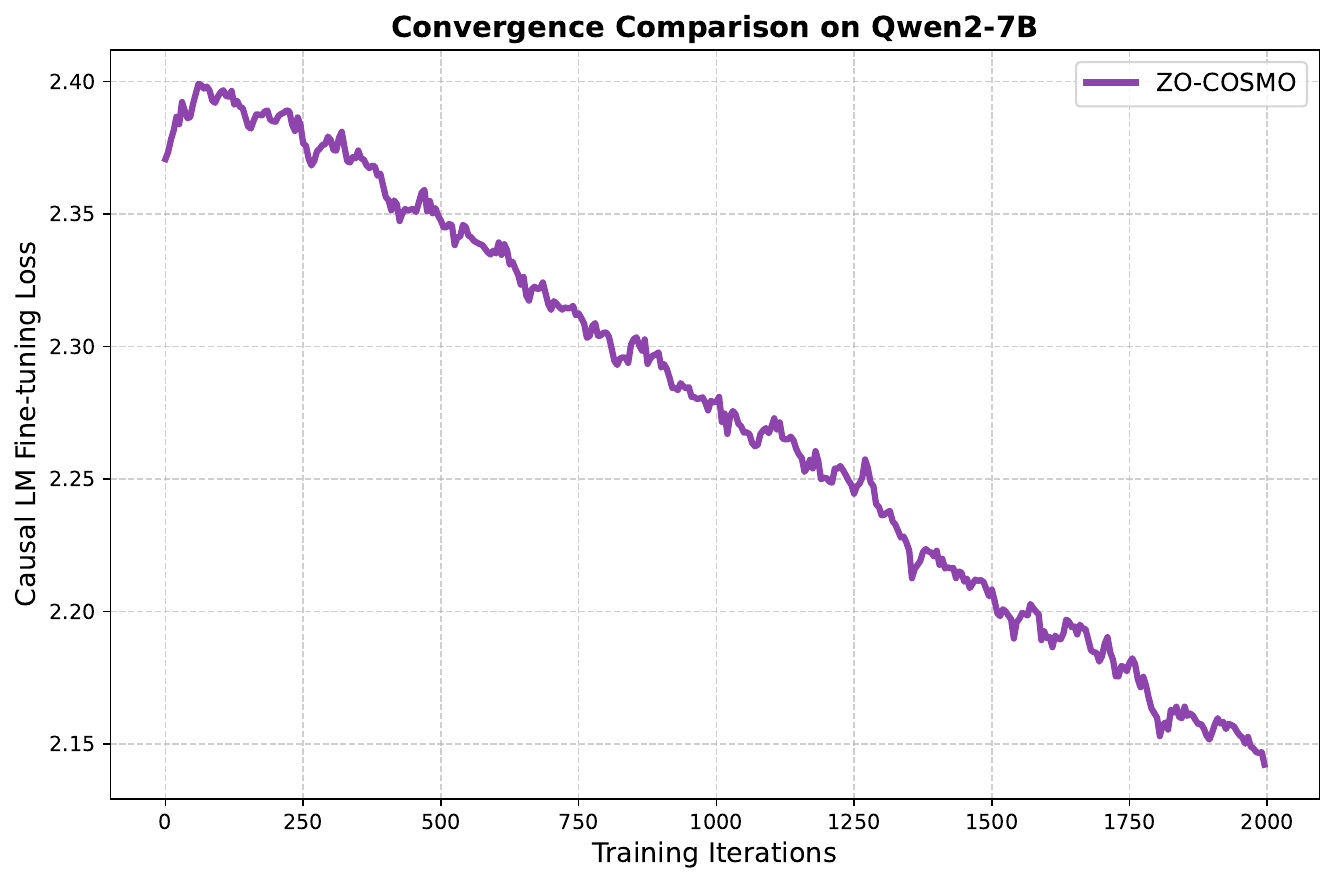}
        \caption{ZO-COSMO.}
        \label{fig:qwen_loss_zocosmo_restored}
    \end{subfigure}\hfill
    \begin{subfigure}{0.48\textwidth}
        \centering
        \includegraphics[width=\linewidth,trim={0 0 0 24bp},clip]{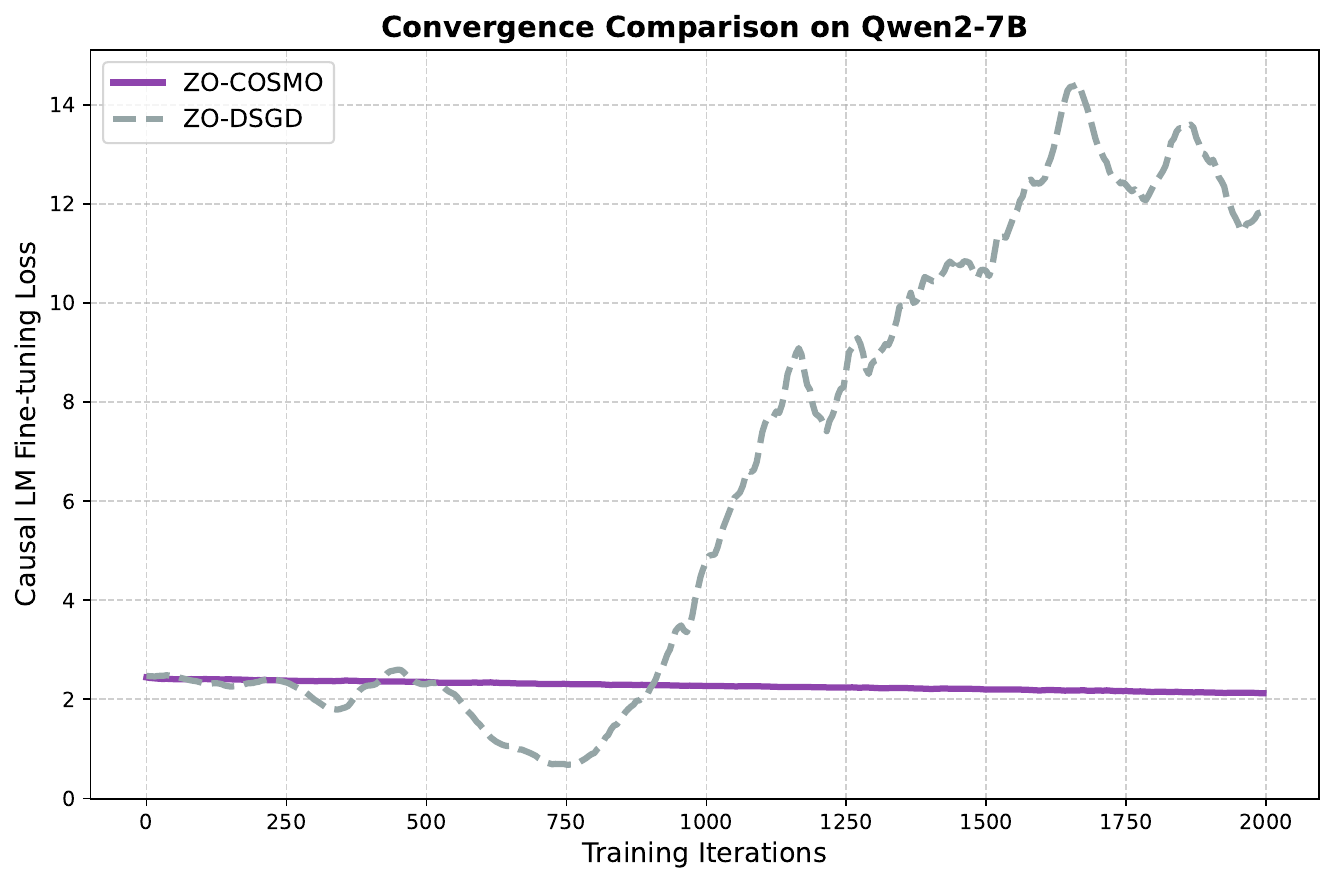}
        \caption{Comparison.}
        \label{fig:qwen_loss_comp_restored}
    \end{subfigure}
    \caption{Qwen2-7B LoRA training-loss trajectories on SST-2. ZO-COSMO reaches loss $2.12$ rather than $11.46$ for dense ZO-DSGD with $308\times$ less total payload.}
    \label{fig:qwen_training_restored}
\end{figure}

At the $26.214$M-bit budget, ZO-COSMO completes $100$ rounds and Indexed Rand-$k$ completes $59$ rounds using $26.100$M bits. Their accuracies are $94.13\pm0.42\%$ and $94.56\pm0.75\%$, respectively; both remain in the pretrained model's near-ceiling regime. ZO-COSMO has the lower interval query loss, $0.741\pm0.165$ versus $0.843\pm0.065$. ZO-MARINA-type cannot purchase its first $80.7$M-bit dense refresh. The maximum sampled discrepancy between synchronized replicas is zero in every run.

\paragraph{Logical-worker topology simulation.}
Every logical worker has its own float32 LoRA state, momentum, and disjoint QNLI shard. Logical workers may share an in-memory copy of the immutable bfloat16 Qwen2-7B backbone. A placement audit runs the same four-worker trajectories on one and two backbone copies: all query values, final adapter fingerprints, and validation metrics agree exactly. Repeated runs and an additional intermediate evaluation also reproduce the trajectory exactly; evaluation restores every worker's state. These checks verify the simulated graph updates independently of memory placement.

We run $N\in\{2,4,8\}$ logical workers for $100$ rounds over complete and ring graphs, using the same $q=8192$, $\mu=10^{-3}$, $\beta=0.9$, batch size $4$, and learning rate $5\times10^{-5}$ as the confirmatory run. The reported model is the network-average LoRA state, evaluated on all $5463$ QNLI validation examples. One gossip pass costs $26.214$ Mbit per directed link. At $N=8$, this is $1.468$ Gbit over the complete graph and $0.419$ Gbit over the ring.

\begin{figure}[H]
    \centering
    \includegraphics[width=0.88\linewidth]{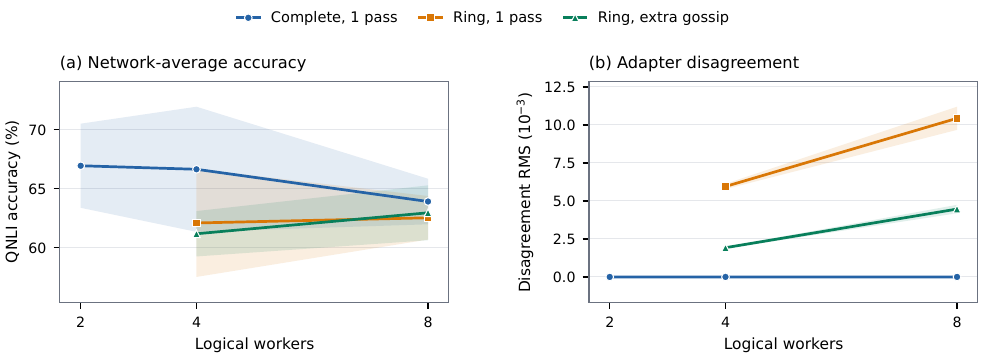}
    \caption{QNLI with $2$--$8$ logical workers after $100$ rounds (five-seed mean $\pm$ std). Extra ring passes are fully charged.}
    \label{fig:qwen_multi_agent}
\end{figure}

The complete-graph means are $66.96\pm3.57\%$, $66.66\pm5.31\%$, and $63.92\pm1.93\%$ for $N=2,4,8$. With one ring pass, $N=4$ reaches $62.10\pm4.58\%$ and $N=8$ reaches $62.54\pm1.85\%$. Their consensus RMS values, sampled on a fixed set of $2048$ LoRA coordinates, are $0.00595\pm0.00018$ and $0.01043\pm0.00076$. Repeating the exchange twice at $N=4$ and four times at $N=8$ changes the spectral factors from $0.333$ to $0.111$ and from $0.805$ to $0.419$. RMS falls to $0.00192\pm0.00005$ and $0.00446\pm0.00028$, while accuracy becomes $61.19\pm1.93\%$ and $62.97\pm2.33\%$. Extra gossip consistently reduces disagreement, but its task-metric benefit is not uniform.

\paragraph{Edge-local QNLI study.}
We next compare global support with Algorithm \ref{alg:edge_local_cosmo} on complete and ring overlays for $N\in\{4,8\}$. The budget is $26.214$ Mbit per node. A global round sends one $q$-value message on every outgoing edge, so the last feasible checkpoints are 33 and 14 rounds on the $N=4$ and $N=8$ complete overlays, and 50 rounds on either ring. A perfect-matching round sends one message per node, allowing 100 edge-local rounds. The index-charged control is the same edge-local trajectory at round 59, re-encoded with 22 coordinate bits per value. We report full-validation accuracy and label NLL over five paired seeds; the 100-round control separates communication efficiency from per-round mixing strength.

\begin{figure}[H]
    \centering
    \includegraphics[width=0.98\linewidth]{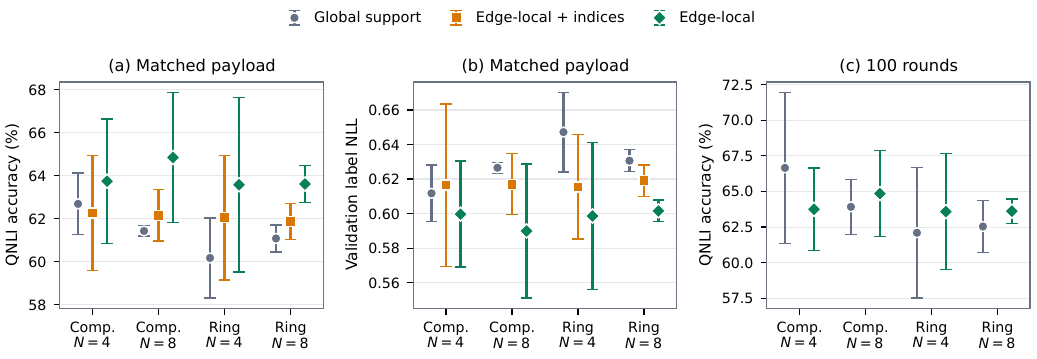}
    \caption{QNLI with global and edge-local support (five-seed mean $\pm$ std). The first two panels fix $26.214$ Mbit per node; the last fixes 100 rounds.}
    \label{fig:qwen_edge_local}
\end{figure}

On the eight-worker complete graph, global support, index-charged edge-local, and index-free edge-local updates reach $61.44\pm0.24\%$, $62.16\pm1.19\%$, and $64.85\pm3.03\%$ accuracy. The index-free gain over global support is $3.42$ points (95\% paired bootstrap CI $[1.06,5.63]$), with a label-NLL reduction of $0.0366$ (CI $[0.0030,0.0649]$). On the eight-worker ring, the endpoints are $61.09\pm0.63\%$, $61.88\pm0.83\%$, and $63.62\pm0.85\%$: an accuracy gain of $2.53$ points (CI $[1.87,3.07]$) and an NLL reduction of $0.0291$ (CI $[0.0238,0.0343]$).

At $N=4$, the corresponding accuracy gains are $1.05$ points on the complete graph (CI $[-2.00,4.10]$) and $3.40$ points on the ring (CI $[0.31,6.49]$). Equal-100-round accuracy intervals span zero in three settings; the eight-worker ring gives $+1.08$ points (CI $[0.07,2.07]$), while all four equal-round NLL intervals span zero. Matching purchases more updates per bit, whereas all-neighbor mixing reaches more peers on each selected coordinate. The alternating ring satisfies the two-round contraction in \eqref{eq:alternating_ring_block_gap}. These paired intervals describe each comparison separately, without a multiple-comparison adjustment. The next control holds scheduling fixed.

Here $d/q=308$, so the 100-round Qwen runs have exposure $\tau=qT/d\approx0.325$. Under iid uniform supports, a coordinate has been selected with probability $1-(307/308)^{100}\approx0.278$. The longer control below extends Qwen training to $\tau=5$. The quadratic mechanism study reaches $\tau=200$, covering substantially more repeated coordinate activations.

\input{v28_qwen_correlation_experiment}

The single-worker full-model study tests sparse-update stability at $6.53$B trainable coordinates and compares two encodings of the same trajectory. With $q=8192$, each update costs $262.1$K bits with public support or $532.5$K bits with explicit coordinates.

In the Qwen full-model study, each step samples a public sparse Rademacher direction supported on $q$ coordinates of the selected trainable parameter set. The two-query loss difference is
\begin{equation}
    \widehat{\partial}_t
    =
    \frac{
        \ell(\theta_t+\mu u_t;\xi_t)
        -
        \ell(\theta_t-\mu u_t;\xi_t)
    }{2\mu},
    \label{eq:appendix_qwen_directional}
\end{equation}
and the communicated sparse values correspond to
\begin{equation}
    \widehat{g}_{t,S_t}
    =
    \frac{d}{q}\widehat{\partial}_t u_{t,S_t}.
    \label{eq:appendix_qwen_sparse_grad}
\end{equation}
Momentum is stored only on coordinates that have appeared in the sparse support, so the auxiliary optimizer state scales with the visited support rather than with the full parameter dimension. In the full-model run, the model itself is loaded in bfloat16, while the sparse momentum map can remain on CPU and grow with the number of distinct sampled coordinates.

The full-model study uses the local Qwen2-7B checkpoint, SST-2 batches of size $8$, maximum sequence length $256$, and evaluation on the SST-2 validation set. The trainable coordinate set contains all model parameters except token embeddings and the language-model head, giving $6.53{\times}10^9$ coordinates. For $q=8192$, the value-only payload is
\begin{equation}
    qB_{\mathrm{val}}
    =
    8192\cdot 32
    =
    262{,}144
    \quad\text{bits per step},
    \label{eq:qwen_value_payload_numeric}
\end{equation}
whereas an indexed sparse message additionally pays
\begin{equation}
    q\lceil\log_2(6.53{\times}10^9)\rceil
    =
    8192\cdot 33
    =
    270{,}336
    \quad\text{index bits per step}.
    \label{eq:qwen_index_payload_numeric}
\end{equation}
Adding coordinate identifiers to the same updates gives $532.5$K bits per step, $2.03\times$ the value-only payload. Supports, signs, and values are unchanged, so the difference is the support-description cost.

This single-worker full-model run verifies sparse two-query updates at the advertised dimensionality, where coordinate metadata is already comparable to the value payload. The multi-worker LoRA experiments above evaluate decentralized optimization and graph topology.

\begin{figure}[H]
    \centering
    \includegraphics[width=0.72\linewidth]{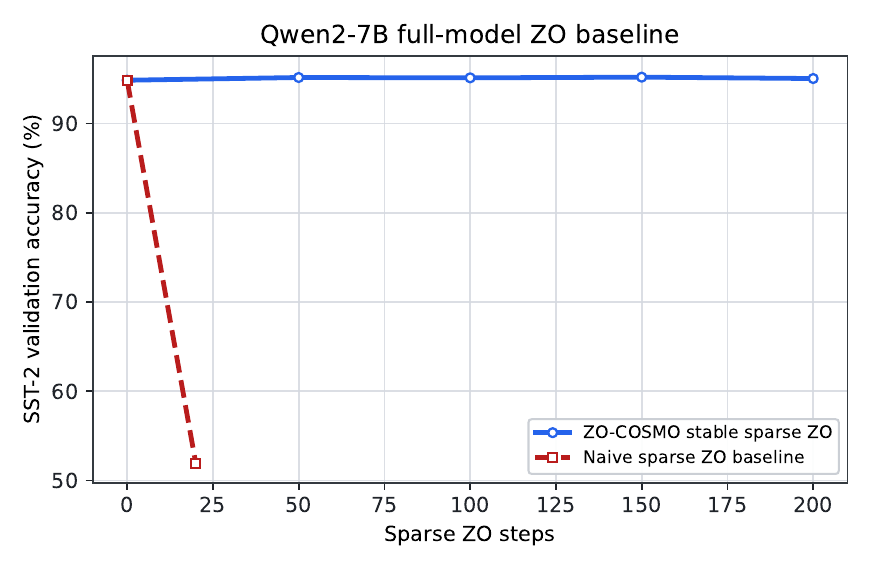}
    \caption{Qwen2-7B full-model sparse ZO study on SST-2.}
    \label{fig:qwen_stable}
\end{figure}

\begin{table}[H]
    \centering
    \caption{Qwen2-7B full-model sparse ZO communication check.}
    \label{tab:qwen_indexed_sparse}
    \renewcommand{\arraystretch}{1.12}
    \small
    \setlength{\tabcolsep}{3pt}
    \begin{tabular}{>{\raggedright\arraybackslash}p{3.45cm} c >{\raggedright\arraybackslash}p{3.1cm} c c}
        \toprule
        \textbf{Run / accounting} & \textbf{Steps} & \textbf{Outcome} & \makecell{\textbf{Bits}\\\textbf{/ Step}} & \makecell{\textbf{Relative}\\\textbf{Payload}} \\
        \midrule
        \rowcolor{gray!10}
        \textbf{ZO-COSMO (Ours)} & 200 & \makecell[l]{$\mathbf{94.84\%\rightarrow}$\\$\mathbf{95.03{\pm}0.13\%}$} & $\mathbf{262.1}$K & $\mathbf{1.00\times}$ \\
        Same ZO-COSMO trajectory + explicit indices & 200 & identical after re-encoding & $532.5$K & $2.03\times$ \\
        Naive sparse ZO, large step & 20 & $94.84\%\rightarrow51.95\%$ & $262.1$K & $1.00\times$ \\
        \bottomrule
    \end{tabular}
\end{table}

All random-support methods use 32-bit payload accounting. The index fraction $\ell_d/(B_{\mathrm{val}}+\ell_d)$ increases at smaller value widths (Table \ref{tab:appendix_index_share}); those lower-bit entries are payload calculations, not quantized training results.

\subsection{Reproducibility Artifacts}
\label{appendix:reproducibility_artifacts}

Raw logs record the method, seed, topology when applicable, support size, metrics, and cumulative communication cost. Tables use the fixed-budget rule in Appendix \ref{appendix:fixed_budget_protocol}; figures are generated from the same logs.

\begin{table}[htbp]
    \centering
    \caption{Main artifacts used for the reported experiments.}
    \label{tab:appendix_artifacts}
    \footnotesize
    \setlength{\tabcolsep}{3pt}
    \begin{tabular}{>{\raggedright\arraybackslash}p{2.3cm}>{\raggedright\arraybackslash}p{3.0cm}>{\raggedright\arraybackslash}p{3.25cm}>{\raggedright\arraybackslash}p{3.25cm}}
        \toprule
        Experiment & Config & Raw or summary logs & Main figure/table artifact \\
        \midrule
        Rosenbrock & main, strong-baseline, network-stress, and support-locality YAMLs & fixed-budget, $N$-sweep, and support-locality raw logs & composite synthetic figures and summary CSVs \\
        Correlation mechanism & original and extension YAMLs & protocols, trajectories, and frozen moments & eight-seed moment and quadratic plots; plot audits \\
        Qwen2-7B & two-worker, topology, edge-local, and full-model YAMLs & ten-seed QNLI, edge-local, SST-2, and full-model logs & LoRA communication, topology, and full-model figures \\
        Communication & serialization and transfer scripts & serialization and 31-trial transfer CSVs & appendix communication diagnostics \\
        \bottomrule
    \end{tabular}
\end{table}

The original mechanism study uses \path{configs/correlation_v21.yaml}, with raw CSVs and its source manifest in \path{results/correlation_v21/}. The extension uses \path{configs/correlation_v23.yaml} and \path{results/correlation_v23/}; its main panels were fixed before training, and the additional appendix views have a separate plot audit. Paths are relative to the experiment-code directory.

The uniform-precision Qwen topology and edge-local study uses \path{configs/qwen_precision_recheck_v25.yaml}; all 55 trials and 95 checkpoints are retained in \path{results/qwen_precision_recheck_v25/}. Aggregation checks initialization fingerprints, source hashes, checkpoint completeness, and payload formulas. The longer fixed-matching study uses \path{configs/qwen_qnli_correlation_v25_fp32.yaml} and \path{results/qwen_correlation_v25_fp32/}. Each trial records its query trace, coordinate coverage, evaluations, and final adapter states. The evaluation and placement audits are stored separately.

The common step-drop extension uses \path{configs/qwen_qnli_step_drop_v27.yaml}. All 15 runs and 75 checkpoints are in \path{results/qwen_step_drop_v27/}; \path{scripts/summarize_qwen_step_drop.py} verifies the archived reference hashes, exact round-100 matches, recorded learning rates, and payload counters before producing the figures and paired statistics. All three couplings use the same schedule.

The support generator folds the seed with the round number and, when needed, a node or unordered edge identifier. Global-support ZO-COSMO uses no node identifier; edge-local ZO-COSMO agrees within each matched pair. Independent Rand-$k$ uses node identifiers. Those supports are private only when the generator seed is withheld: making it public permits the seed-aware encoding without changing the independent-support trajectories.

The Qwen configuration specifies the HuggingFace checkpoint and trainable-coordinate filter. The topology runner records logical workers, graph, spectral factor, active links, per-node and total payload, and sampled consensus error. The full-model run excludes embeddings and the language-model head, uses bfloat16 weights, and stores sparse momentum in a CPU dictionary. These choices reduce memory use while preserving the recorded protocol and payload.

\subsection{Implementation Details}
\label{appendix:implementation_notes}

Global support is keyed by the master seed and round; edge-local support additionally uses the unordered endpoints of the active edge. The same support serves queries and message decoding. The explicit-index baselines send coordinate lists, while the seed-aware control reconstructs independent supports from node-specific public keys.

For dense baselines, communication is charged as a full vector payload even if a local implementation stores tensors in contiguous blocks. For sparse baselines, communication is charged per directed link and per sparse vector. If an undirected graph edge carries messages in both directions, the total edge traffic is twice the directed-link number. The main method comparisons use directed-link cost. The edge-local topology study instead fixes payload per node because global and matching rounds activate different numbers of links; total network payload is recorded as well.

Fine-tuning payloads are counted with 32-bit values. Table \ref{tab:appendix_index_share} gives the corresponding representation costs at other value widths.

\paragraph{Qwen LoRA-scale experiment.}
Dense ZO-DSGD communicates all LoRA parameters; ZO-COSMO sends only public sparse values. Both use forward-only loss differences. Each logical worker keeps independent float32 adapter and momentum states; the frozen backbone uses bfloat16. Adapter precision is set explicitly when each adapter is created and reset. Sharing a frozen backbone changes memory placement and scheduling, not these states or the graph update.

\paragraph{Qwen full-model experiment.}
The trainable set excludes token embeddings and the language-model head. Weights use bfloat16, and a coordinate map stores sparse momentum without a dense update buffer. The stable runs use seeds $0,1,2$ and are re-encoded with explicit indices; the oversized-step row is a one-seed failure check. Table \ref{tab:appendix_payload_examples} gives the cost of a hypothetical dense refresh.

\subsection{Ablation Design}
\label{appendix:ablation_interpretation}

The momentum ablation keeps the public support and removes the momentum buffer; its matched-first-step control accounts for the resulting scale change. The independent-support ablation changes cross-node alignment. We evaluate both its original coordinate-list encoding and the seed-aware encoding reconstructed from node-specific public seeds.

The same-update indexed control isolates representation: it uses the identical public support, estimator, momentum, and masked consensus as ZO-COSMO, then charges the coordinate list that an ordinary sparse format would carry. Com-DSZO tests a different question. Its Top-$k$ innovation is data-adaptive and error compensated, but it communicates indices and keeps two dense auxiliary arrays in addition to its compressed innovation. Comparing these methods separates an index-free representation gain from the optimization and memory benefits of a stronger compressor.

The periodic-refresh control's first dense packet exceeds the Qwen LoRA comparison budget, so that entry is marked at initialization. Its trained equal-round SST-2 result is reported separately with the full cumulative payload.

\section{Broader Impact and Ethics}
Forward-only fine-tuning reduces storage of backward-pass activations. Local data and exchanged model states still require access controls; function-value access and decentralized execution alone provide no privacy guarantee. The experiments evaluate optimization and payload use with logical workers, while hardware and execution details are reported in Appendix \ref{appendix:experimental_details}.

\end{document}

%% file: figures/one_hop_compatibility_v20.tex
\begin{tikzpicture}[
    font=\small,
    >=Stealth,
    packet/.style={draw=gray!55, fill=gray!3, line width=0.6pt,
        minimum width=2.25cm, minimum height=0.58cm, inner sep=3pt},
    receiver/.style={draw=gray!65, fill=white, line width=0.7pt,
        minimum width=1.18cm, minimum height=1.35cm, align=center},
    flow/.style={->, draw=gray!65, line width=0.85pt}
]
\node[font=\bfseries\small] at (2.85,1.35) {Public, incompatible supports};
\node[font=\bfseries\small] at (9.05,1.35) {ZO-COSMO: aligned supports};
\draw[gray!35] (0,1.02) -- (5.70,1.02);
\draw[gray!35] (6.20,1.02) -- (11.90,1.02);
\node at (0.30,0.48) {$j$};
\node at (0.30,-0.43) {$\ell$};
\node[packet] (jl) at (2.10,0.48) {$(y_j^{(1)},y_j^{(2)})$};
\node[packet] (ll) at (2.10,-0.43) {$(y_\ell^{(2)},y_\ell^{(3)})$};
\node[receiver] (il) at (4.90,0.03) {$i$\\[3pt]$K=\{1,2\}$};
\draw[flow] (jl.east) -- ($(il.west)+(0,0.34)$);
\draw[flow] (ll.east) -- ($(il.west)+(0,-0.34)$);
\node[text=red!65!black] at (2.85,-1.18) {$y_\ell^{(1)}$ missing: exact average unavailable};
\node at (6.50,0.48) {$j$};
\node at (6.50,-0.43) {$\ell$};
\node[packet,draw=green!40!black,fill=green!4] (jr) at (8.30,0.48) {$(y_j^{(1)},y_j^{(2)})$};
\node[packet,draw=green!40!black,fill=green!4] (lr) at (8.30,-0.43) {$(y_\ell^{(1)},y_\ell^{(2)})$};
\node[receiver] (ir) at (11.10,0.03) {$i$\\[3pt]$K=\{1,2\}$};
\draw[flow] (jr.east) -- ($(ir.west)+(0,0.34)$);
\draw[flow] (lr.east) -- ($(ir.west)+(0,-0.34)$);
\node[text=green!35!black] at (9.05,-1.18) {All target values present; no index field};
\end{tikzpicture}

%% file: v28_correlation_main.tex
\subsection{Direction Correlation and Heterogeneity}
\label{sec:correlation_mechanism}

Support agreement makes a message decodable; direction agreement also changes the update statistics. Hold the matching fixed, with $N=2m$. Let $a_e=(\nabla F_i(x_i)+\nabla F_j(x_j))/2$ on pair $e=\{i,j\}$, $\bar a=m^{-1}\sum_e a_e$, and $H_M^2=m^{-1}\sum_e\|a_e-\bar a\|^2$. The linearized two-query update is $v_e=D_ea_e$, where $D_e=(d/q)u_eu_e^\top$ and $\bar v=m^{-1}\sum_e v_e$.

Both endpoints share $u_e$. Across pairs, G shares the support and signs, S shares only the support, and I draws both independently. All three use the same matching and send $q$ values per node.

\begin{proposition}\label{prop:corr_moments}
Condition on the states and matching. With $V_s=\mathbb E\|\bar v-\bar a\|^2$ and $U_s=m^{-1}\sum_e\mathbb E\|v_e-\bar v\|^2$,
\begin{align}
 V_{\mathrm G}&=(d-1)\|\bar a\|^2,&
 V_{\mathrm I}&=\frac{d-1}{m}(\|\bar a\|^2+H_M^2),\label{eq:corr_main_variance}\\
 V_s+U_s&=(d-1)\|\bar a\|^2+dH_M^2,
 &&s\in\{\mathrm G,\mathrm S,\mathrm I\}.\label{eq:corr_main_energy}
\end{align}
\end{proposition}
\begin{proof}
See Appendix \ref{appendix:correlation_moments}, including the formula for S.
\end{proof}

Under G, averaging commutes with the random map: $\bar v=D\bar a$. Local gradient components that cancel in $\bar a$ therefore cancel before estimation. Independent maps trade this cancellation for averaging direction noise across pairs. For $d>1$, I has smaller $V_s$ precisely when $H_M^2<(m-1)\|\bar a\|^2$; G benefits on the other side of this threshold. Equation \eqref{eq:corr_main_energy} tracks the accompanying disagreement injection. Correlation therefore changes the update itself, even when the schedule and payload are fixed.

With unequal quadratic Hessians, state disagreement controls the residual (Appendix Proposition \ref{prop:curvature_v23}). Let $\mathcal A$ denote masked pair averaging and $V$ the applied update matrix. The two rows of $V$ on each pair agree, giving $\mathcal A V=V$ even though queries and mixing share a mask. Appendix Proposition \ref{prop:momentum_moments_v23} extends the moments to oracle noise, stored momentum, and finite differences. Theorems \ref{thm:corr_core} and \ref{thm:momentum_convergence} give, respectively, a sharper deterministic core rate for iid matchings and convergence with stochastic sparse momentum.

%% file: v23_correlation_experiment.tex
\subsection{Mechanism Experiment}
\label{sec:correlation_experiment}

We compare G, S, and I with $N=8$, $d=32$, $q=4$, and paired matchings. Local quadratics have unequal positive diagonal Hessians and centered linear terms with $N^{-1}\sum_i\|b_i\|^2=4$. Noisy queries add fresh independent $\epsilon_i^\top x$, with $\mathbb E\|\epsilon_i\|^2=0.25$, reusing each sample within its query pair.

Figure \ref{fig:correlation_mechanism} uses iid ring matchings and $\eta=0.004$. With oracle noise, G reaches $(1.06\pm0.30)\times10^{-3}$ versus $(9.60\pm1.69)\times10^{-3}$ for I; with sparse momentum $\beta=0.9$ at the same step, the corresponding values are $(1.00\pm0.13)\times10^{-3}$ and $(10.60\pm2.70)\times10^{-3}$. Noise leaves a nonzero floor.

Matching the first update by increasing the momentum step to $0.04$ reverses the comparison: G ends at $1.30\pm1.42$, versus $0.131\pm0.021$ for I. Appendix \ref{appendix:extension_experiments_v23} includes this control, zero-linear-heterogeneity cases, and all 1536 runs across two steps and topologies.

\begin{figure}[!htbp]
\centering
\includegraphics[width=\linewidth]{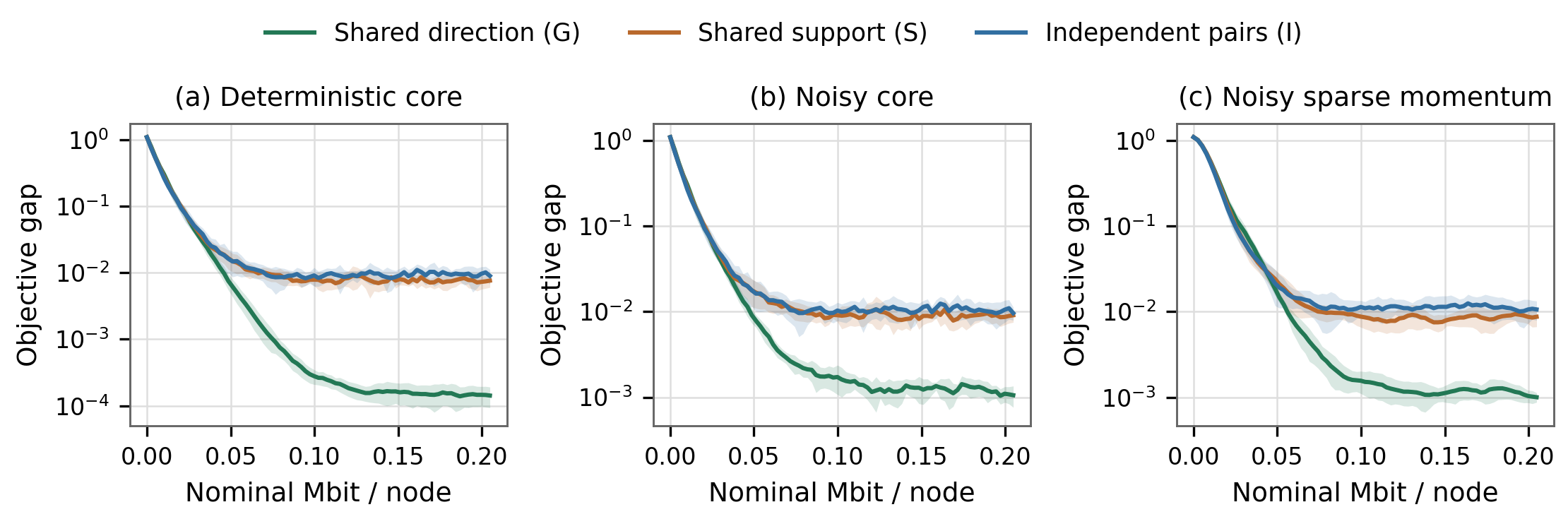}
\caption{Different local curvatures on an eight-node ring. Eight-seed mean $\pm$ std at equal queries and nominal payload; (c) uses the same step as (b).}
\label{fig:correlation_mechanism}
\end{figure}

%% file: v28_qwen_long_main.tex
Long-horizon training remains effective with a common tenfold step reduction after round $100$. At $1540$ rounds on the four-worker ring, the three couplings reach $74.07$--$77.82\%$ accuracy over five seeds. With matching and support draws fixed, independent edge signs improve accuracy by $3.75$ points (CI $[0.24,6.79]$) and lower label NLL by $0.027$ (CI $[0.009,0.049]$). The gain comes from query correlation at unchanged scheduling and support selection (Appendix \ref{appendix:qwen_long_correlation}).

%% file: v28_momentum_proof.tex
\subsection{Convergence with Sparse Momentum}
\label{appendix:momentum_convergence}

The finite difference and local update use the current mixing support. We apply contraction to the pre-query state and bound the resulting update separately. This argument covers global support, independent supports on matched edges, and a common support on all matched edges.

Write $\mathcal A_t$ for the linear mixing operator and $\Pi=I-N^{-1}\mathbf1\mathbf1^\top$. Each $\mathcal A_t$ preserves the average, commutes with $\Pi$, and is nonexpansive in Frobenius norm. Suppose that for an integer $h\ge1$ and $\gamma\in(0,1]$, every block satisfies
\begin{equation}\label{eq:momentum_block_contraction}
 \mathbb E\!\left[\|\mathcal A_{kh+h-1}\cdots\mathcal A_{kh}Z\|_F^2
       \mid\mathcal F_{kh}\right]\le(1-\gamma)\|Z\|_F^2
 \qquad (\mathbf1^\top Z=0),
\end{equation}
for every $\mathcal F_{kh}$-measurable $Z$. The block's supports and matching schedule are independent of the state at its start. Their dependence on subsequent within-block iterates is handled explicitly below. Every node is active, and its rank-$q$ projector $P_{i,t}$ is conditionally uniform given the past. Put $p=q/d$ and $a_\beta=\beta/(1-\beta)$.

\begin{theorem}\label{thm:momentum_convergence}
Under Assumptions \ref{assum:smooth_lower}, \ref{assum:oracle}, and \ref{assum:third}, consider either global or perfect-matching ZO-COSMO with \eqref{eq:optional_subspace_momentum}, common initialization, and zero initial momentum. If \eqref{eq:momentum_block_contraction} holds, then for any constant $\eta>0$, $\mu>0$, and $0\le\beta<1$,
\begin{equation}\label{eq:momentum_finite_rate}
 \frac1T\sum_{t<T}\mathbb E\|\nabla F(\bar x_t)\|^2
 \le \mathcal O\!\left(
 \frac{\Delta_0}{\eta T}+L\eta M_{\mathrm{node}}^2+B_\mu^2
 +L^2\eta^2M_{\mathrm{node}}^2
       \left[\frac{h^2}{\gamma^2}+\frac{a_\beta^2}{p}\right]\right).
\end{equation}
The hidden constant is numerical. Choosing $\eta=\Theta(T^{-1/2})$ and $\mu=\Theta((dq)^{-1/2}T^{-1/4})$ yields vanishing average stationarity for fixed problem parameters, $h,\gamma$, and $\beta$. A directed-link or per-node budget is incorporated by substituting the number of rounds affordable under that protocol's actual payload.
\end{theorem}

\begin{proof}
Let $G_t,M_t,V_t$ stack the gradient estimates, stored momentum, and applied directions. On node $i$, the update is
\begin{equation}\label{eq:momentum_node_update}
 m_{i,t+1}=(I-P_{i,t})m_{i,t}+\beta P_{i,t}m_{i,t}
             +(1-\beta)\hat g_{i,t},\qquad
 v_{i,t}=P_{i,t}m_{i,t+1}.
\end{equation}
Since $P_{i,t}\hat g_{i,t}=\hat g_{i,t}$, rearranging gives the exact identity
\begin{equation}\label{eq:sparse_momentum_identity}
 v_{i,t}=\hat g_{i,t}-a_\beta(m_{i,t+1}-m_{i,t}).
\end{equation}
This identity retains the inactive entries of $m_{i,t}$ and requires no independence between the query and the projector.

First control the moments. Orthogonality of active and inactive coordinates and convexity of the squared norm give
\begin{align}
 \mathbb E\|m_{i,t+1}\|^2
 &\le[1-(1-\beta)p]\mathbb E\|m_{i,t}\|^2
                      +(1-\beta)M_{\mathrm{node}}^2,\label{eq:momentum_storage_bound}\\
 \mathbb E\|v_{i,t}\|^2
 &\le\beta p\mathbb E\|m_{i,t}\|^2
                      +(1-\beta)M_{\mathrm{node}}^2.
\end{align}
Here only the old momentum is conditioned on the past; $\mathbb E\|P_{i,t}m_{i,t}\|^2=p\mathbb E\|m_{i,t}\|^2$. Proposition \ref{prop:estimator} bounds the fresh estimate, including its dependence on the current support. Starting from zero momentum, induction proves
\begin{equation}\label{eq:momentum_uniform_moments}
 \mathbb E\|m_{i,t}\|^2\le M_{\mathrm{node}}^2/p,
 \qquad \mathbb E\|v_{i,t}\|^2\le M_{\mathrm{node}}^2,
 \qquad \mathbb E\|\bar g_t\|^2\le M_{\mathrm{node}}^2.
\end{equation}
The last bound follows from Jensen's inequality for both common and edge-local directions. It uses the per-node second moment, without a worker-averaging factor.

Next bound disagreement. The actual recursion is
\begin{equation}
 X_{t+1}=\mathcal A_t(X_t-\eta V_t).
\end{equation}
Over a block, expand it as
\begin{equation}\label{eq:momentum_block_recursion}
 \Pi X_{kh+h}=\mathcal A_{kh+h-1}\cdots\mathcal A_{kh}\Pi X_{kh}
 -\eta\sum_{s=0}^{h-1}\mathcal A_{kh+h-1}\cdots\mathcal A_{kh+s}\Pi V_{kh+s}.
\end{equation}
Nonexpansiveness and $\|\sum_{s=0}^{h-1}b_s\|^2\le h\sum_s\|b_s\|^2$ bound the normalized expected squared norm of the second term by $\eta^2h^2M_{\mathrm{node}}^2$. This bound is pathwise before expectation and remains valid when $V_{kh+s}$ depends on any of the block's supports. Apply \eqref{eq:momentum_block_contraction} only to the first term. For $\gamma<1$, Young's inequality with $\alpha=\gamma/[2(1-\gamma)]$ yields
\begin{equation}
 \mathcal E_{kh+h}\le(1-\gamma/2)\mathcal E_{kh}
                   +\frac{3\eta^2h^2M_{\mathrm{node}}^2}{\gamma}.
\end{equation}
If $\gamma=1$, the homogeneous term has zero second moment and the same inequality follows directly. Common initialization gives
$\mathcal E_{kh}\le6\eta^2h^2M_{\mathrm{node}}^2/\gamma^2$.
Expanding a partial block of length $r<h$ and using the same norm bound gives, including the final incomplete block,
\begin{equation}\label{eq:momentum_disagreement_bound}
 \sup_{t<T}\mathcal E_t\le C\eta^2h^2M_{\mathrm{node}}^2/\gamma^2.
\end{equation}

Finally, define the virtual network average
\begin{equation}\label{eq:momentum_virtual_average}
 z_t=\bar x_t-\eta a_\beta\bar m_t.
\end{equation}
Average preservation and \eqref{eq:sparse_momentum_identity} imply
$z_{t+1}=z_t-\eta\bar g_t$ exactly, with $z_0=\bar x_0$.
The conditional mean of $\bar g_t$ is
$N^{-1}\sum_i\nabla F_i(x_{i,t})+\bar b_t$, where $\|\bar b_t\|\le B_\mu$.
Smoothness, Jensen's inequality, and \eqref{eq:momentum_uniform_moments} therefore imply
\begin{equation}\label{eq:momentum_mean_alignment}
 \mathbb E\left\|\mathbb E[\bar g_t\mid\mathcal F_t]-\nabla F(z_t)\right\|^2
 \le C L^2\mathcal E_t+C L^2\eta^2a_\beta^2M_{\mathrm{node}}^2/p+C B_\mu^2.
\end{equation}
Applying smooth descent at $z_t$, and using
$\langle a,b\rangle\ge\|a\|^2/2-\|a-b\|^2/2$, gives
\begin{align}
 \mathbb EF(z_{t+1})
 &\le\mathbb EF(z_t)-\frac\eta2\mathbb E\|\nabla F(z_t)\|^2
      +C\eta L^2\mathcal E_t+C\eta B_\mu^2\nonumber\\
 &\qquad+C L^2\eta^3a_\beta^2M_{\mathrm{node}}^2/p
                +\frac{L\eta^2}{2}M_{\mathrm{node}}^2.
\end{align}
Sum over $t<T$, use $F(z_T)\ge F^*$, and divide by $\eta T/2$.
The inequality
\begin{equation}
 \mathbb E\|\nabla F(\bar x_t)\|^2
 \le2\mathbb E\|\nabla F(z_t)\|^2
       +2L^2\eta^2a_\beta^2M_{\mathrm{node}}^2/p
\end{equation}
then transfers the result back to the reported network average. Substitution of \eqref{eq:momentum_disagreement_bound} proves \eqref{eq:momentum_finite_rate}.
\end{proof}

\paragraph{Schedules covered by the bound.}
For global support, the frozen-input calculation gives $h=1$ and
$\gamma=p(1-\rho^2)$. For independent perfect matchings with
$\lambda_2(\bar L)>0$, Theorem \ref{thm:edge_local_contraction} gives
$h=1$ and $\gamma=p\lambda_2(\bar L)/2$. A common uniform support on all edges of the matching gives the same pure-mixing expectation: independence across disjoint edge supports is not needed for that identity.

The implemented ring schedule for even $N\ge4$ alternates two perfect matchings. Each phase is disconnected, so we bound contraction over two rounds. Write $L_0,L_1$ for the matching Laplacians; $L_0+L_1=L_{\mathrm{ring}}$. For a fixed centered $Z$, let $Z'=\mathcal A_0Z$ and define the expected energy drops
\begin{equation}
 D_0=\frac p2\operatorname{tr}(Z^\top L_0Z),\qquad
 D_1=\frac p2\mathbb E\operatorname{tr}({Z'}^\top L_1Z').
\end{equation}
The second support is fresh, so the total two-round drop is $D_0+D_1$. Moreover,
$\mathbb E\|Z-Z'\|_F^2=D_0$ because the first mixing step is an orthogonal projection. Since $\|L_1\|_2=2$,
\begin{equation}
 \operatorname{tr}(Z^\top L_1Z)
 \le2\mathbb E\operatorname{tr}({Z'}^\top L_1Z')+4D_0.
\end{equation}
Consequently,
\begin{equation}
 \frac p2\operatorname{tr}(Z^\top L_{\mathrm{ring}}Z)
 \le(1+2p)D_0+2D_1\le3(D_0+D_1).
\end{equation}
Thus \eqref{eq:momentum_block_contraction} holds with
\begin{equation}\label{eq:alternating_ring_block_gap}
 h=2,\qquad \gamma=\frac{p\lambda_2(L_{\mathrm{ring}})}6
               =\frac p3\bigl(1-\cos(2\pi/N)\bigr).
\end{equation}
Substituting this block contraction into \eqref{eq:momentum_finite_rate} establishes convergence for alternating-ring momentum updates with either common or independent edge supports. Direction-dependent moment bounds are developed in Appendix \ref{appendix:correlation_moments}.

%% file: v28_correlation_appendix.tex
\section{Correlation and Optimization}
\label{appendix:correlation}

We analyze Algorithm \ref{alg:edge_local_cosmo} without momentum, holding the matching schedule fixed across couplings. Correlated compression already distinguishes individual from aggregate estimator error \citep{szlendak2022permutation}. Here the same direction also determines which peer-state coordinates are mixed, so aggregate variance must be analyzed together with the disagreement update. The identities below describe that joint operation.

In this section $\mathcal D_t=N^{-1}\|\Pi X_t\|_F^2$ is the instantaneous state disagreement; its expectation is the $\mathcal E_t$ used in the main convergence analysis.

\subsection{Protocol and Couplings}
Let $N=2m$ nodes maintain $x_i\in\R^d$. A perfect matching $M$ is sampled independently of the states. Each pair $e=\{i,j\}$ uses a uniform $q$-coordinate support $S_e$, signs $\epsilon_e$, and
\begin{equation}\label{eq:corr_operators}
 P_e=\diag(\mathbf 1_{S_e}),\qquad u_e=P_e\epsilon_e,\qquad
 D_e=\frac dq u_eu_e^\top,\qquad p=\frac qd.
\end{equation}
Both endpoints use the same $u_e$. Across distinct pairs we consider:
\begin{enumerate}
\item G: one common support and one common sign vector;
\item S: one common support, with independent sign vectors across pairs;
\item I: independent supports and sign vectors across pairs.
\end{enumerate}
The public seed, round counter, and unordered edge identifiers generate all three choices without recurring support metadata. Each node evaluates
\begin{equation}\label{eq:corr_fd}
 \hat g_i=\frac dq\frac{F_i(x_i+\mu u_e)-F_i(x_i-\mu u_e)}{2\mu}u_e,
 \qquad y_i=x_i-\eta\hat g_i,
\end{equation}
and sends $q$ ordered values of $y_i$ to its matched peer. The endpoints average those coordinates and retain their own remaining entries. Every variant therefore uses two optimization queries and $qB_{\rm val}$ outgoing value bits per node per round. Initial seed setup is separate.

Each run keeps its coupling fixed; support and sign draws are fresh at every round.

\subsection{Conditional Moments}
\label{appendix:correlation_moments}
Fix the states and matching before drawing directions. Put
\begin{equation}\label{eq:corr_pair_gradients}
 a_e=\frac{\nabla F_i(x_i)+\nabla F_j(x_j)}2,\quad
 \bar a=\frac1m\sum_e a_e=\frac1N\sum_i\nabla F_i(x_i),\quad
 H_M^2=\frac1m\sum_e\corrnorm{a_e-\bar a}^2.
\end{equation}
For the ideal directional estimates let $v_e=D_ea_e$ and $\bar v=m^{-1}\sum_e v_e$. Define
\begin{equation}\label{eq:corr_moments_def}
 V_s=\mathbb E\corrnorm{\bar v-\bar a}^2,\qquad
 U_s=\frac1m\sum_e\mathbb E\corrnorm{v_e-\bar v}^2,
 \qquad s\in\{\mathrm G,\mathrm S,\mathrm I\}.
\end{equation}
Here $U_s$ measures disagreement injected by the applied pair updates, not the current state disagreement.

For $1\le q\le d$, the three couplings satisfy
\begin{align}
 V_s&=\alpha_s\corrnorm{\bar a}^2+\tau_s H_M^2,\label{eq:corr_variance}\\
 U_s&=(d-1-\alpha_s)\corrnorm{\bar a}^2+(d-\tau_s)H_M^2,\label{eq:corr_injection}
\end{align}
where
\begin{align}
 (\alpha_{\mathrm G},\tau_{\mathrm G})&=(d-1,0),\nonumber\\
 (\alpha_{\mathrm S},\tau_{\mathrm S})&=\left(\frac dq-1+\frac{d-d/q}{m},\frac{d-d/q}{m}\right),\label{eq:corr_coefficients}\\
 (\alpha_{\mathrm I},\tau_{\mathrm I})&=\left(\frac{d-1}{m},\frac{d-1}{m}\right).\nonumber
\end{align}
In particular, $V_s+U_s=(d-1)\corrnorm{\bar a}^2+dH_M^2$ is independent of the coupling.

\begin{proof}
The sparse Rademacher identities give
\begin{equation}\label{eq:corr_D_moments}
 \mathbb E D_e=I,\qquad D_e^2=dD_e,\qquad
 \mathbb E\corrnorm{D_ea}^2=d\corrnorm a^2,
 \qquad \mathbb E\corrnorm{(D_e-I)a}^2=(d-1)\corrnorm a^2.
\end{equation}
Under G, $\bar v=D\bar a$, so $V_{\mathrm G}=(d-1)\corrnorm{\bar a}^2$. Under I, centered pair estimates are independent. Consequently
\begin{equation}
 V_{\mathrm I}=\frac{d-1}{m^2}\sum_e\corrnorm{a_e}^2
 =\frac{d-1}{m}\bigl(\corrnorm{\bar a}^2+H_M^2\bigr).
\end{equation}
For S, condition first on the common projector $P$. Independence of the signs gives $\mathbb E[D_e\mid P]=(d/q)P$. The support component of the variance is
\begin{equation}
 \mathbb E\corrnorm{((d/q)P-I)\bar a}^2=(d/q-1)\corrnorm{\bar a}^2.
\end{equation}
The conditional sign fluctuations are independent across pairs and centered given $P$. Their averaged contribution is
\begin{align}
 \frac1{m^2}\sum_e\mathbb E\corrnorm{(D_e-(d/q)P)a_e}^2
 &=\frac{d-d/q}{m^2}\sum_e\corrnorm{a_e}^2\\
 &=\frac{d-d/q}{m}\bigl(\corrnorm{\bar a}^2+H_M^2\bigr).
\end{align}
The law of total variance proves the middle line of \eqref{eq:corr_coefficients}. Finally, the pathwise identity
\begin{equation}
 \frac1m\sum_e\corrnorm{v_e-\bar v}^2
 =\frac1m\sum_e\corrnorm{v_e}^2-\corrnorm{\bar v}^2
\end{equation}
and \eqref{eq:corr_D_moments} give
$U_s=d(\corrnorm{\bar a}^2+H_M^2)-\corrnorm{\bar a}^2-V_s$.
\end{proof}

For $d>1$ and $m>1$, I has lower average-estimation error than G exactly when
\begin{equation}\label{eq:corr_crossover}
 H_M^2<(m-1)\corrnorm{\bar a}^2.
\end{equation}
Any decrease in $V_s$ is an equal increase in $U_s$. Thus a smaller average-estimator variance does not alone order the optimization trajectories. At $q=1$, G and S coincide at the level of $D_e$ because the sign cancels. At $q=d$, S and I have the same operator law because the support is fixed. For $N=2$, all variants have the same law because there is only one pair. If $d=1$, every $D_e=1$ and all $V_s$ vanish.

\subsection{The Coupled Mixing Step}
Stack states as rows of $X$ and let $\Pi=I-N^{-1}\mathbf1\mathbf1^\top$. With $L_e=(e_i-e_j)(e_i-e_j)^\top$, the mixing operator is
\begin{equation}\label{eq:corr_mixing}
 \mathcal A X=X-\frac12\sum_{e\in M}L_eXP_e.
\end{equation}
Disjoint edges imply that $\mathcal A$ is an orthogonal projection in the Frobenius inner product. It preserves the average and commutes with $\Pi$.

Let $V$ have row $v_e$ at both endpoints of $e$. The actual ideal-direction update is
\begin{equation}\label{eq:corr_pair_update}
 X^+=\mathcal A X-\eta V.
\end{equation}
This follows by mixing the two local estimates: averaging $D_e\nabla F_i(x_i)$ and $D_e\nabla F_j(x_j)$ gives $D_ea_e$. Because the two rows of $V$ agree within every pair,
\begin{equation}\label{eq:corr_pair_invariance}
 \mathcal A V=V,\qquad
 \langle\Pi\mathcal A X,\Pi V\rangle_F
 =\langle\Pi X,\Pi V\rangle_F.
\end{equation}
No independence between $V$ and the query--mixing mask is used here.

\begin{proposition}\label{prop:corr_joint}
For fixed $X,M$, define $\bar x=N^{-1}\sum_i x_i$, $\mathcal D=N^{-1}\corrnorm{\Pi X}_F^2$, and
$C_M=N^{-1}\sum_{e=\{i,j\}}\langle x_i+x_j-2\bar x,a_e\rangle$.
Under the ideal-direction update,
\begin{equation}\label{eq:corr_exact_disagreement}
 \mathbb E\mathcal D^+=\mathcal D
 -\frac{p}{2N}\sum_{\{i,j\}\in M}\corrnorm{x_i-x_j}^2
 -2\eta C_M+\eta^2U_s.
\end{equation}
\end{proposition}
\begin{proof}
Orthogonality gives the exact pure-mixing energy drop
\[
 \corrnorm{\Pi\mathcal A X}_F^2=\corrnorm{\Pi X}_F^2
 -\frac12\sum_e\corrnorm{P_e(x_i-x_j)}^2.
\]
Take its support expectation, expand \eqref{eq:corr_pair_update}, and apply \eqref{eq:corr_pair_invariance}. Since $\mathbb E v_e=a_e$, the cross term is $C_M$. The normalized update energy is precisely $U_s$.
\end{proof}

Equation \eqref{eq:corr_exact_disagreement} separates pure mixing, gradient alignment, and update disagreement. Pair constancy determines the cross term while retaining the shared query--mixing mask. Proposition \ref{prop:momentum_moments_v23} gives the moments with stored momentum, and Theorem \ref{thm:momentum_convergence} proves its convergence.

\subsection{Graph and Objective Assumptions}
\begin{assumption}\label{ass:corr_smooth}
Each deterministic objective $F_i$ is $L$-smooth, $L>0$, and $F=N^{-1}\sum_iF_i$ satisfies $F\ge F^*>-\infty$. For all $x$,
\begin{equation}\label{eq:corr_heterogeneity}
 \frac1N\sum_i\corrnorm{\nabla F_i(x)-\nabla F(x)}^2\le\zeta^2.
\end{equation}
\end{assumption}
\begin{assumption}\label{ass:corr_fd}
For every used direction and point, the central-difference remainder obeys
\begin{equation}
 |F_i(x+\mu u)-F_i(x-\mu u)-2\mu\nabla F_i(x)^\top u|
 \le (L_3/3)\mu^3\corrnorm u^3.
\end{equation}
\end{assumption}
Here the objectives are deterministic and nonconvex, with uniformly bounded gradient heterogeneity and unrestricted absolute gradients. Theorem \ref{thm:convergence} instead uses the stochastic-oracle assumptions from Section \ref{sec:theory}.

Matchings are independent draws from a fixed state-independent distribution on perfect matchings. Write
\begin{equation}\label{eq:corr_gaps}
 \bar L=\mathbb E_M\sum_{e\in M}L_e,\qquad
 \chi=1-\frac{\lambda_2(\bar L)}2<1,\qquad
 \gamma=\frac{p\lambda_2(\bar L)}2=p(1-\chi)>0.
\end{equation}
Because matching Laplacians have spectrum in $[0,2]$, $0\le\chi<1$. With $\mathcal F_t$ containing only past randomness, $X_t$ is fixed when the current matching and directions are sampled. The pure-mixing contraction therefore applies to $X_t$, not to a mask-dependent post-query state.

There is a second role for the matching spectrum. Let $G$ stack the current local gradients and $\bar a$ be their average. Pair averaging all coordinates uses the projection $T_M=I-\frac12\sum_e L_e$. Hence
\begin{align}
 \mathbb E_M H_M^2
 &=\frac1N\operatorname{tr}\bigl((\Pi G)^\top(I-\bar L/2)\Pi G\bigr)\nonumber\\
 &\le\frac\chi N\sum_i\corrnorm{\nabla F_i(x_i)-\bar a}^2
 \le2\chi(L^2\mathcal D+\zeta^2).\label{eq:corr_matching_heterogeneity}
\end{align}
The last inequality uses the network gradient as a candidate center, smoothness, and \eqref{eq:corr_heterogeneity}. Also
\begin{equation}\label{eq:corr_average_gradient_bounds}
 \corrnorm{\bar a-\nabla F(\bar x)}^2\le L^2\mathcal D,
 \qquad\corrnorm{\bar a}^2\le2\corrnorm{\nabla F(\bar x)}^2+2L^2\mathcal D.
\end{equation}
On complete graphs, uniform perfect matchings give $\chi=(N-2)/(2(N-1))$. On even rings with independently sampled matching phases, $\chi=(1+\cos(2\pi/N))/2$. The Qwen ring uses deterministic alternation, whose two-round contraction is proved in Appendix \ref{appendix:momentum_convergence}.

\subsection{A Correlation-Dependent Core Rate}
For a fixed coupling $s$, abbreviate $\alpha=\alpha_s$, $\tau=\tau_s$ and define
\begin{equation}\label{eq:corr_constants}
 r=d-1-\alpha,\quad J=\chi(d-\tau+2/\gamma),\quad K=r+J,\quad
 B_\mu=\frac{L_3\mu^2dq}{6}.
\end{equation}
\begin{theorem}\label{thm:corr_core}
Under Assumptions \ref{ass:corr_smooth}--\ref{ass:corr_fd} and \eqref{eq:corr_gaps}, run the two-query matching core with common initialization, fixed coupling, and constant $\eta,\mu>0$. If
\begin{equation}\label{eq:corr_step_condition}
 \eta\le\min\left\{\frac1{8L(1+\alpha+\tau\chi)},
                    \frac{\sqrt\gamma}{16L\sqrt K}\right\},
\end{equation}
where the second restriction is omitted for $K=0$, then with $\Delta_0=F(\bar x_0)-F^*$,
\begin{align}
 \frac1T\sum_{t<T}\mathbb E\corrnorm{\nabla F(\bar x_t)}^2
 &\le\mathcal O\!\left(\frac{\Delta_0}{\eta T}
 +L\eta\tau\chi\zeta^2
 +\frac{\eta^2L^2J}{\gamma}\zeta^2\right.\nonumber\\
 &\hspace{20mm}\left.+\left[1+\frac{\eta^2L^2}{\gamma}
                  (1+\gamma^{-1})\right]B_\mu^2\right).
 \label{eq:corr_core_rate}
\end{align}
\end{theorem}

For fixed problem and graph parameters, $\eta=\Theta(T^{-1/2})$ and $\mu=\Theta((dq)^{-1/2}T^{-1/4})$ give an $O(T^{-1/2})$ bound. G has $\tau=0$: taking $\eta=\Theta(T^{-1/3})$ and $\mu=\Theta((dq)^{-1/2}T^{-1/6})$ yields $O(T^{-2/3})$. The step condition \eqref{eq:corr_step_condition} holds for sufficiently large horizons. These are upper bounds; no matching lower bound is established for S or I.

For a nominal per-node value-payload budget $\mathcal B$, substitute $T=\lfloor\mathcal B/(qB_{\rm val})\rfloor$ whenever $T\ge1$. The couplings have identical payloads and round counts, so \eqref{eq:corr_core_rate} isolates their optimization differences. The bound uses real-valued updates, omits quantization error, and retains dependence on $d,q,N,L,\gamma,\chi,\zeta$.

\subsubsection{Proof of Theorem \ref{thm:corr_core}}
Condition on $\mathcal F_t$ and suppress $t$. Taylor's bound gives
\begin{equation}\label{eq:corr_residual}
 \hat g_i=D_e\nabla F_i(x_i)+e_i,\qquad\corrnorm{e_i}\le B_\mu.
\end{equation}
After pair averaging the applied update has identical rows $D_ea_e+r_e$ at the two endpoints, where $r_e=(e_i+e_j)/2$ and $\corrnorm{r_e}\le B_\mu$. Thus
\begin{equation}
 X^+=\mathcal A X-\eta(V+R),\quad
 \mathcal A(V+R)=V+R,\quad
 \bar x^+=\bar x-\eta(\bar v+\bar r).
\end{equation}
The pair-invariance identity holds for the residual as well, even though it depends on the direction.

Let $g=\nabla F(\bar x)$. Equations \eqref{eq:corr_average_gradient_bounds} and \eqref{eq:corr_residual} imply
\begin{equation}\label{eq:corr_alignment}
 \mathbb E\langle g,\bar v+\bar r\rangle
 \ge\frac34\corrnorm g^2-2L^2\mathcal D-2B_\mu^2.
\end{equation}
Indeed the conditional mean error from $g$ has squared norm at most $2L^2\mathcal D+2B_\mu^2$, and $\langle g,e\rangle\ge-\corrnorm g^2/4-\corrnorm e^2$.
By Proposition \ref{prop:corr_moments}, $\corrnorm{\bar r}\le B_\mu$, and \eqref{eq:corr_matching_heterogeneity},
\begin{align}
 \mathbb E\corrnorm{\bar v+\bar r}^2
 &\le4(1+\alpha)\corrnorm g^2
 +4L^2(1+\alpha+\tau\chi)\mathcal D
 +4\tau\chi\zeta^2+2B_\mu^2.\label{eq:corr_mean_moment_bound}
\end{align}
Smooth descent and the first condition in \eqref{eq:corr_step_condition} now give
\begin{equation}\label{eq:corr_function_recursion}
 \mathbb E F(\bar x^+)\le F(\bar x)-\frac\eta2\corrnorm g^2
 +\frac94\eta L^2\mathcal D
 +2L\eta^2\tau\chi\zeta^2+3\eta B_\mu^2.
\end{equation}
Here $L\eta\le1/8$, the gradient coefficient is at most $-3\eta/4+2L\eta^2(1+\alpha)\le-\eta/2$, and the disagreement coefficient is at most $2\eta L^2(1+1/8)=9\eta L^2/4$.

For disagreement, expand the squared norm of the actual update. Apply pure-mixing contraction only to $X$, then apply pair invariance to the cross term. The ideal cross term has magnitude at most $\sqrt{\mathcal D}\sqrt{\mathbb E_M H_M^2}$; the residual contributes at most $B_\mu\sqrt{\mathcal D}$. Young's inequality separately bounds each term by $\gamma\mathcal D/4$ plus $4\eta^2/\gamma$ times its corresponding squared magnitude. Also
$\mathbb E\corrnorm{\Pi(V+R)}_F^2/N\le2\mathbb E_M U_s+2B_\mu^2$.
Consequently,
\begin{align}
 \mathbb E\mathcal D^+
 &\le(1-\gamma/2)\mathcal D
 +2\eta^2r\corrnorm{\bar a}^2
 +2\eta^2(d-\tau+2/\gamma)\mathbb E_M H_M^2
 +\eta^2(2+4/\gamma)B_\mu^2\nonumber\\
 &\le(1-\gamma/2+4\eta^2L^2K)\mathcal D
 +4\eta^2r\corrnorm g^2+4\eta^2J\zeta^2
 +\eta^2(2+4/\gamma)B_\mu^2.\label{eq:corr_consensus_recursion}
\end{align}
The second condition in \eqref{eq:corr_step_condition} implies $\eta^2L^2K\le\gamma/256$; in particular the coefficient on $\mathcal D$ is at most $1-\gamma/4$.

Use the potential
\begin{equation}\label{eq:corr_potential}
 \Phi_t=F(\bar x_t)+\frac{12\eta L^2}{\gamma}\mathcal D_t.
\end{equation}
Adding the weighted form of \eqref{eq:corr_consensus_recursion} to \eqref{eq:corr_function_recursion} absorbs the disagreement term. Since $r\le K$,
$48\eta^3L^2r/\gamma\le3\eta/16\le\eta/4$.
Dropping the remaining negative disagreement term gives
\begin{align}
 \mathbb E[\Phi_{t+1}\mid\mathcal F_t]
 &\le\Phi_t-\frac\eta4\corrnorm{\nabla F(\bar x_t)}^2
 +2L\eta^2\tau\chi\zeta^2
 +\frac{48\eta^3L^2J}{\gamma}\zeta^2\nonumber\\
 &\quad+\left[3\eta+\frac{12\eta^3L^2}{\gamma}
                              (2+4/\gamma)\right]B_\mu^2.
 \label{eq:corr_potential_descent}
\end{align}
Take expectations, sum over $t<T$, and use $\Phi_0=F(\bar x_0)$ and $\Phi_T\ge F^*$. Division by $\eta T/4$ gives the explicit bound
\begin{align}
 \frac1T\sum_{t<T}\mathbb E\corrnorm{\nabla F(\bar x_t)}^2
 &\le\frac{4\Delta_0}{\eta T}
 +8L\eta\tau\chi\zeta^2
 +\frac{192\eta^2L^2J}{\gamma}\zeta^2\nonumber\\
 &\quad+\left[12+\frac{48\eta^2L^2}{\gamma}
                         \left(2+\frac4\gamma\right)\right]B_\mu^2.
 \label{eq:corr_explicit_core_rate}
\end{align}
This proves \eqref{eq:corr_core_rate}. No uniform bound on $\corrnorm{\nabla F_i(x_i)}$ was used.
\qed

\subsection{A Quadratic Cancellation Example}
The trajectory experiment uses
\begin{equation}\label{eq:corr_quadratic}
 F_i(x)=\tfrac12x^\top Hx+b_i^\top x,\qquad
 H=\diag(h_1,\ldots,h_d)\succ0,\qquad\sum_i b_i=0.
\end{equation}
Then $F(x)=x^\top Hx/2$ and $\zeta^2=N^{-1}\sum_i\corrnorm{b_i}^2$. Central differences are exact for every $\mu>0$ in ideal arithmetic.

\begin{proposition}\label{prop:corr_quadratic}
For G on \eqref{eq:corr_quadratic}, the network average satisfies
\begin{equation}\label{eq:corr_cancellation}
 \bar x_{t+1}=\bar x_t-\eta D_tH\bar x_t
\end{equation}
regardless of the matching sequence and the individual linear terms. At the consensus optimum $x_i=0$, a single step of coupling $s$ instead has
$\mathbb E\corrnorm{\bar x^+}^2=\eta^2\tau_s\mathbb E_M H_M^2$.
\end{proposition}
\begin{proof}
Average preservation and a common $D_t$ give
$N^{-1}\sum_i\hat g_i=D_t(H\bar x_t+N^{-1}\sum_i b_i)=D_tH\bar x_t$.
At consensus zero, $\bar a=0$, so Proposition \ref{prop:corr_moments} gives the second assertion.
\end{proof}
Shared curvature makes the G average trajectory identical across these heterogeneity levels and topologies. Proposition \ref{prop:curvature_v23} below derives the residual when local Hessians differ, and Proposition \ref{prop:momentum_moments_v23} adds oracle noise and momentum.

\subsection{Mechanism Experiments}
\label{appendix:correlation_experiments}
The configuration was fixed before training and archived with a timestamp and source hashes. There are $N=8$ nodes, $d=32$, $q=4$, eight seeds, three couplings, two topologies, three heterogeneity levels, and two step sizes: 288 runs. Each uses 1600 two-query rounds and 3200 optimization queries per node. The bit axis charges a nominal 32 bits per value, or 204800 bits per node. Computation uses NumPy double precision without packet quantization or physical transfer. All runs use the deterministic core, fixed parameters, and unclipped updates.

The eigenvalues of $H$ in \eqref{eq:corr_quadratic} range geometrically from 1 to 4. Gaussian $b_i$ are centered and scaled to $\zeta\in\{0,0.5,2\}$, and all nodes start from $d^{-1/2}\mathbf1$. Public directions and matchings are generated independently of these linear terms. Ring phases are independent fair draws at each round, rather than deterministic alternation. Within every comparison, supports/signs are the only protocol change. Objective and gradient diagnostics are evaluation-only and do not enter the update.

The step $\eta=10^{-4}$ satisfies \eqref{eq:corr_step_condition} for every reported coupling and topology. The larger $\eta=0.004$ tests behavior outside that sufficient condition. All runs at both steps are finite and included in the results.

For the frozen-state study, $\corrnorm{\bar a}=1$ and $H_M^2$ takes six specified values. Each seed averages 1500 direction draws for each coupling; gradients are held fixed, so these are random-matrix moment diagnostics rather than optimization-query comparisons. Figure \ref{fig:corr_moments} overlays the exact formulas and measured means. At $H_M^2=0$, the predicted G/S/I average variances are $31$, $13$, and $7.75$; at $H_M^2=15$ they are $31$, $103$, and $124$. All three intersect at $H_M^2=m-1=3$, with the complementary ordering in disagreement injection.
\begin{figure}[!htbp]
 \centering\includegraphics[width=\linewidth]{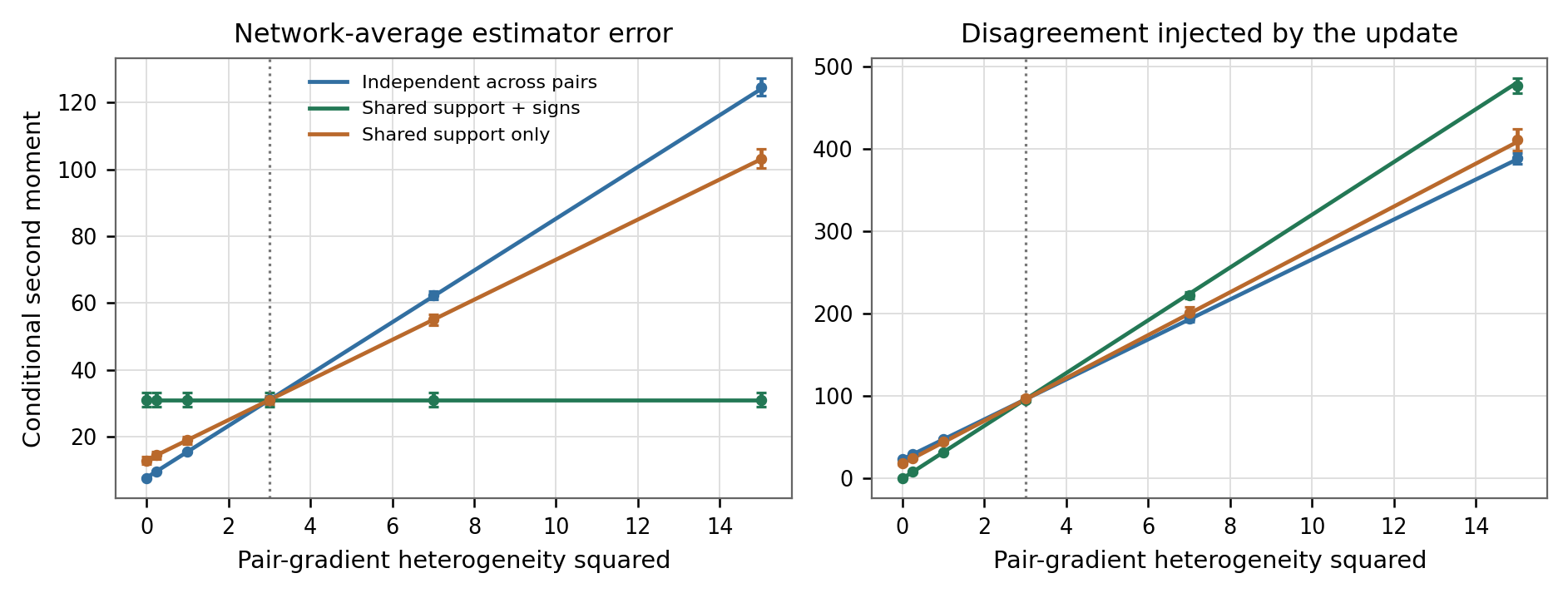}
 \caption{Exact moment curves and eight-seed Monte Carlo means $\pm$ standard deviation. The dotted line is the predicted crossover.}
 \label{fig:corr_moments}
\end{figure}

At $\eta=0.004$, nonzero linear heterogeneity produces an optimization floor for S and I, while G retains the cancellation in \eqref{eq:corr_cancellation}. For $\zeta=2$ on the complete graph, the final objective is $(1.160\pm0.421)\times10^{-7}$ for G and $0.01118\pm0.00309$ for I. At zero heterogeneity, the final means are close despite the different frozen-state variances. The separation in this common-curvature example is driven by heterogeneity cancellation.
\begin{figure}[!htbp]
 \centering\includegraphics[width=\linewidth]{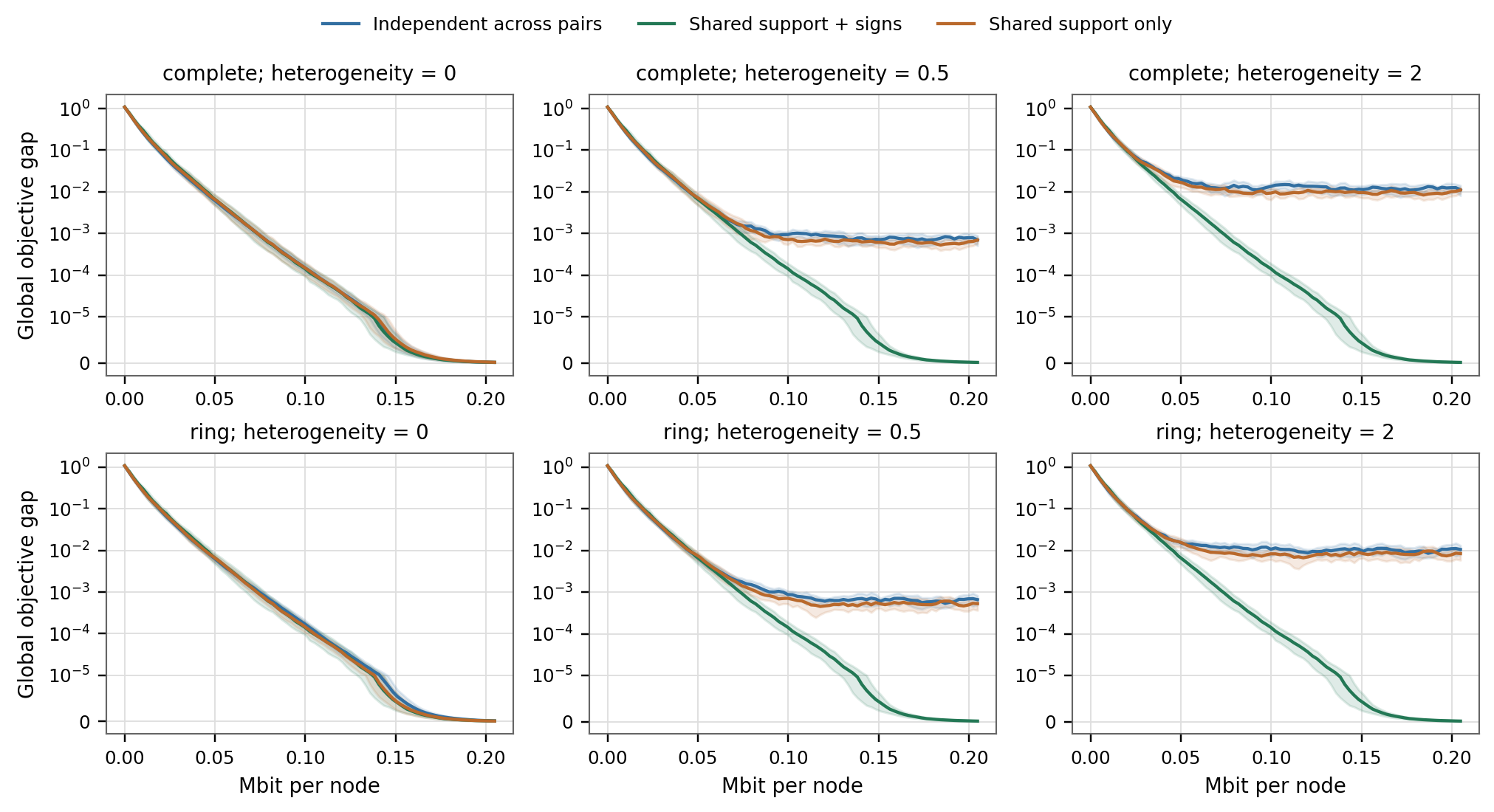}
 \caption{Objective gap at $\eta=0.004$ (outside the sufficient step certificate). Curves are eight-seed means and bands are one standard deviation.}
 \label{fig:corr_objective_practical}
\end{figure}
\begin{figure}[!htbp]
 \centering\includegraphics[width=\linewidth]{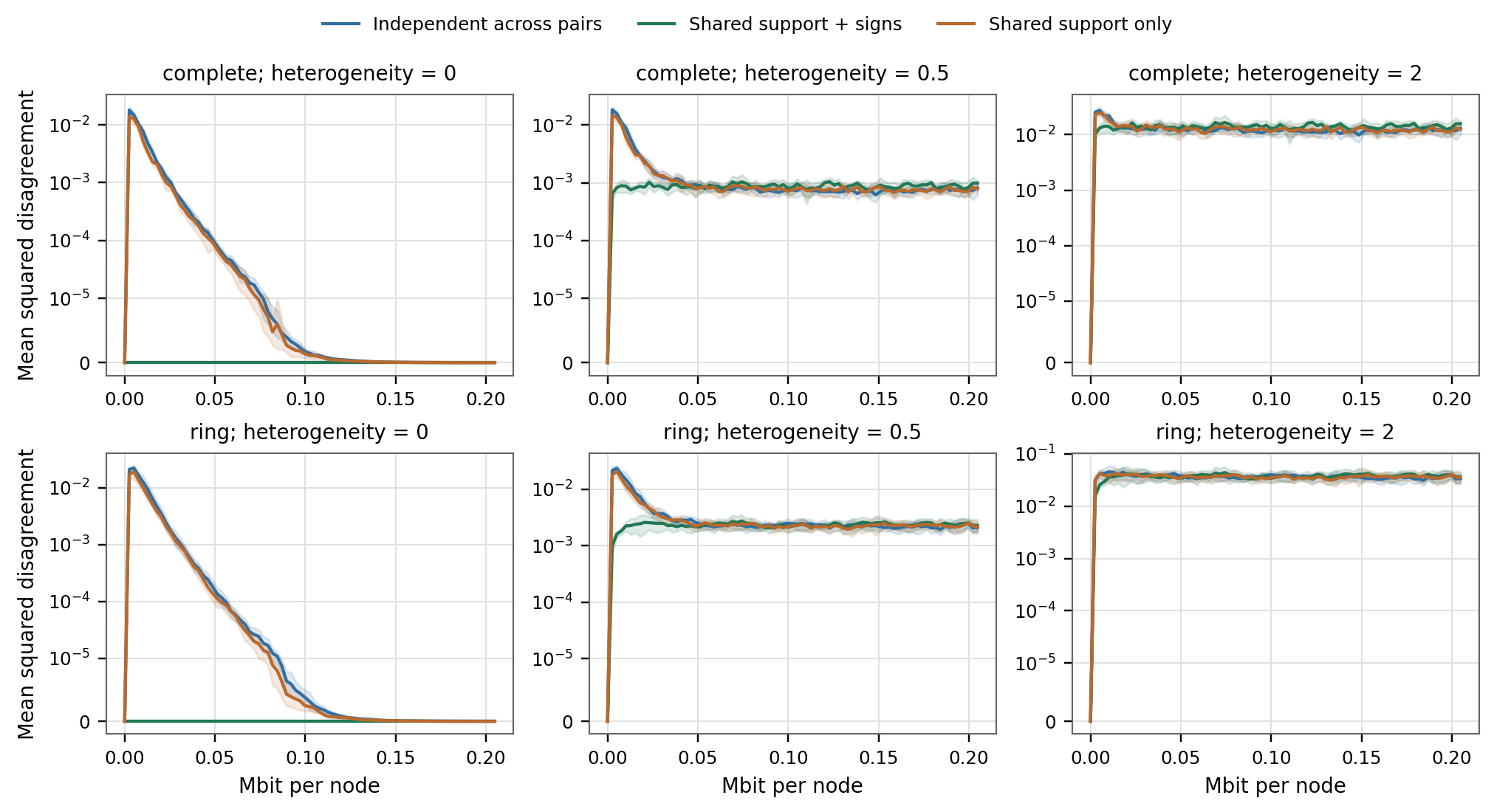}
 \caption{State disagreement for the same eight-seed runs.}
\end{figure}

\begin{figure}[!htbp]
 \centering\includegraphics[width=\linewidth]{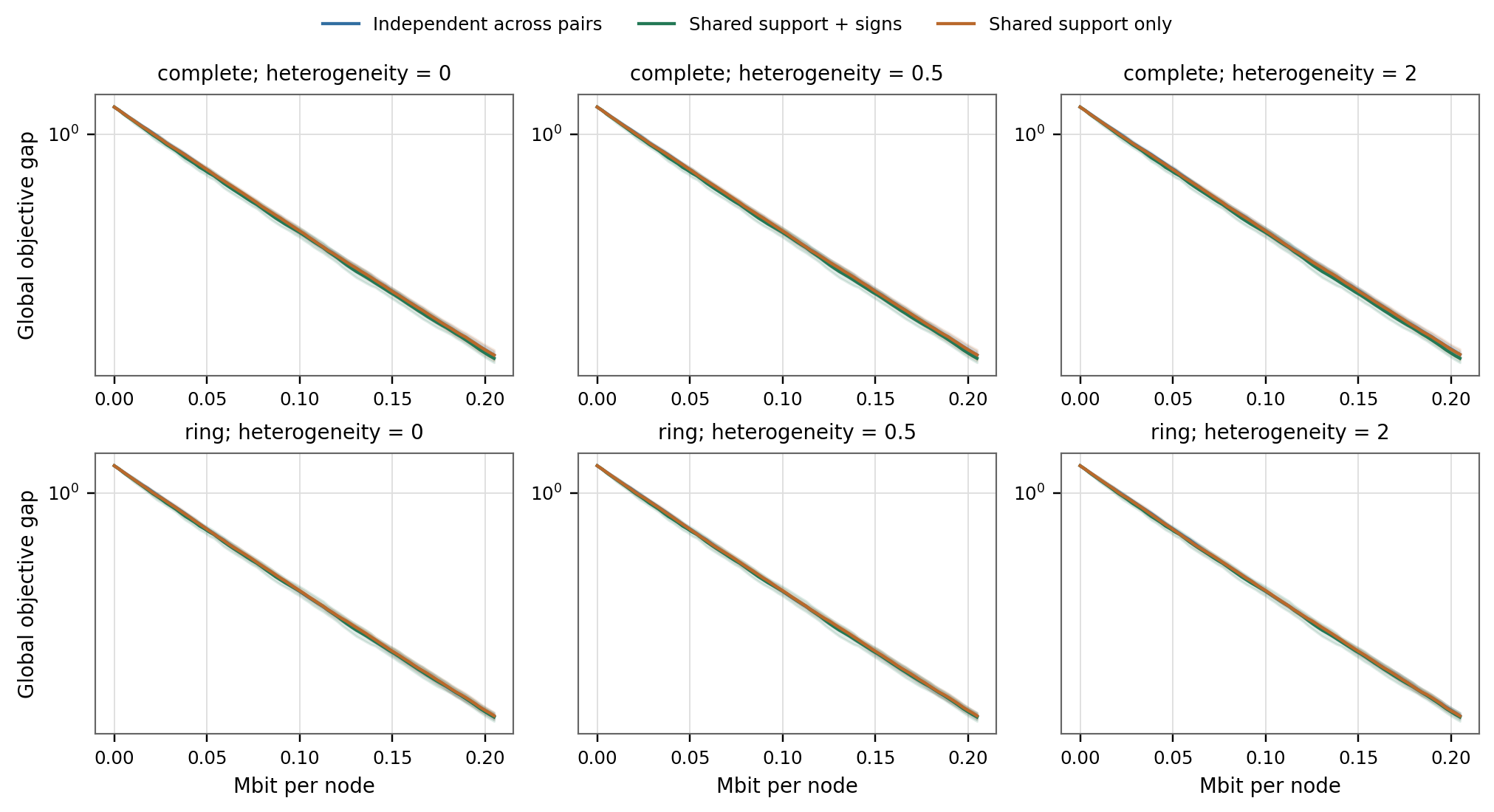}
 \caption{Objective gap at $\eta=10^{-4}$, which satisfies the sufficient step condition for all displayed settings.}
\end{figure}
\begin{figure}[!htbp]
 \centering\includegraphics[width=\linewidth]{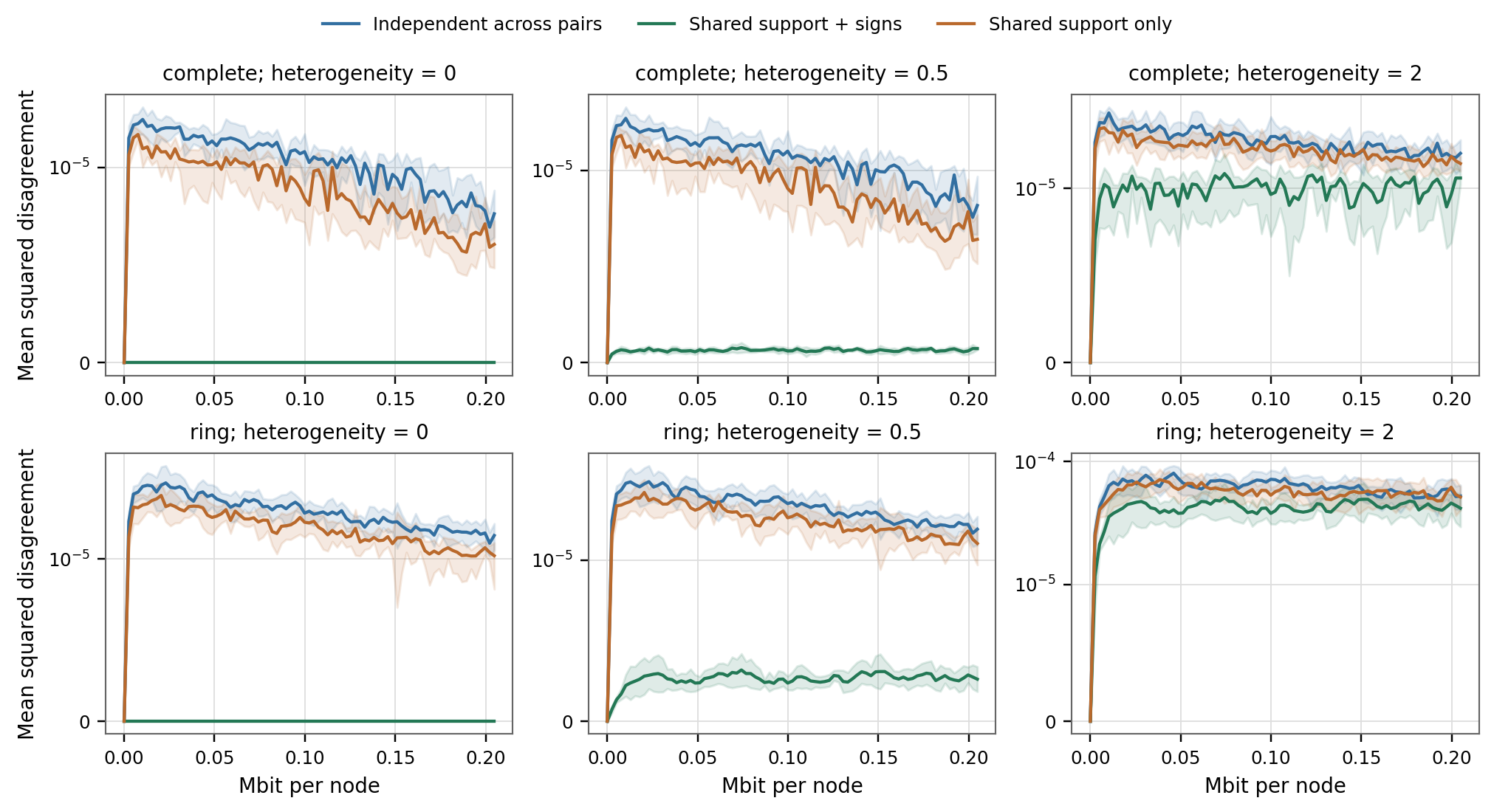}
 \caption{State disagreement at the certified step. Training plots use a symmetric-log scale, linear below $10^{-5}$; standard-deviation bands are truncated at zero for nonnegative metrics.}
\end{figure}

%% file: v28_correlation_extension.tex
\subsection{Curvature Heterogeneity}
\label{appendix:curvature_v23}

The common-Hessian example isolates exact cancellation. With different local curvatures, the uncancelled part can be identified rather than absorbed into an absolute gradient bound. Let
\[
 F_i(x)=\tfrac12x^\top H_ix+b_i^\top x,\qquad
 H=\frac1N\sum_iH_i,\quad b=\frac1N\sum_i b_i,\quad
 \delta_H^2=\frac1N\sum_i\|H_i-H\|_{\mathrm{op}}^2,
\]
where the Hessians are symmetric. Write $\mathcal D=N^{-1}\sum_i\|x_i-\bar x\|^2$.

\begin{proposition}\label{prop:curvature_v23}
For arbitrary node states,
\begin{equation}\label{eq:curvature_residual_v23}
 \bar a=\nabla F(\bar x)+r_X,\qquad
 r_X=\frac1N\sum_i(H_i-H)(x_i-\bar x),\qquad
 \|r_X\|^2\le\delta_H^2\mathcal D.
\end{equation}
Under G without momentum or oracle noise, the network-average step is
$\bar x^+=\bar x-\eta D(\nabla F(\bar x)+r_X)$ and
\begin{equation}\label{eq:curvature_direction_v23}
 \mathbb E\|Dr_X\|^2=d\|r_X\|^2\le d\delta_H^2\mathcal D.
\end{equation}
\end{proposition}
\begin{proof}
Expand $H_ix_i$ around $\bar x$. Since $\sum_i(x_i-\bar x)=0$,
\[
 \frac1N\sum_i(H_ix_i+b_i)
 =H\bar x+b+\frac1N\sum_i(H_i-H)(x_i-\bar x).
\]
The triangle and Cauchy--Schwarz inequalities give
\[
 \|r_X\|
 \le\frac1N\sum_i\|H_i-H\|_{\mathrm{op}}\|x_i-\bar x\|
 \le\delta_H\sqrt{\mathcal D}.
\]
The common direction commutes with averaging, and $\mathbb E D^2=dI$ gives
\eqref{eq:curvature_direction_v23}.
\end{proof}

The averaging identity holds with different Hessians, with a residual controlled by state disagreement. Shared curvature or consensus gives $r_X=0$. This identity applies to arbitrary quadratic states; applying Theorem \ref{thm:corr_core} additionally requires its uniform gradient-heterogeneity assumption over $\mathbb R^d$.

\subsection{Oracle Noise and Stored Momentum}
\label{appendix:momentum_correlation_v23}

Condition on the node states, stored local momenta, and current matching before drawing directions and oracle samples. For $e=\{i,j\}$, retain $a_e$ from \eqref{eq:corr_pair_gradients} and put $c_e=(m_i+m_j)/2$, $\bar c=m^{-1}\sum_e c_e$, and $\theta=1-\beta$, with $0\le\beta<1$. Suppose local gradient errors $\epsilon_i$ are centered, mutually independent, and independent of the current directions under this conditioning. Define
\[
 \sigma_X^2=\frac1N\sum_i\mathbb E\|\epsilon_i\|^2,\qquad
 \epsilon_e=(\epsilon_i+\epsilon_j)/2,\qquad p=q/d.
\]
Neither equal node variances nor isotropic noise is needed. The ideal applied pair update and its network-average conditional mean are
\begin{equation}\label{eq:momentum_pair_v23}
 w_e=\beta P_ec_e+\theta D_e(a_e+\epsilon_e),\qquad
 \nu=\mathbb E\bar w=\beta p\bar c+\theta\bar a.
\end{equation}
This is the sparse-momentum rule \eqref{eq:optional_subspace_momentum} after mixing. Inactive memory entries are retained, not decayed.

For $a,c\in\mathbb R^d$, define the nonnegative quadratic form
\begin{equation}\label{eq:psi_v23}
 \Psi_\beta(a,c)=\beta^2p(1-p)\|c\|^2
 +\theta^2(d-1)\|a\|^2
 +2\beta\theta(1-p)\langle c,a\rangle.
\end{equation}
The cross term is essential: the stored direction and fresh query use the same support.

\begin{proposition}\label{prop:momentum_moments_v23}
With $V_s^{\beta,\sigma}=\mathbb E\|\bar w-\nu\|^2$, the three couplings satisfy
\begin{align}
 V_{\mathrm G}^{\beta,\sigma}
 &=\Psi_\beta(\bar a,\bar c)+\theta^2d\sigma_X^2/N,\label{eq:momentum_variance_G_v23}\\
 V_{\mathrm I}^{\beta,\sigma}
 &=\frac1{m^2}\sum_e\Psi_\beta(a_e,c_e)+\theta^2d\sigma_X^2/N,\label{eq:momentum_variance_I_v23}\\
 V_{\mathrm S}^{\beta,\sigma}
 &=p(1-p)\|\beta\bar c+\theta\bar a/p\|^2
 +\frac{\theta^2(d-1/p)}{m^2}\sum_e\|a_e\|^2
 +\theta^2d\sigma_X^2/N.\label{eq:momentum_variance_S_v23}
\end{align}
For $U_s^{\beta,\sigma}=m^{-1}\sum_e\mathbb E\|w_e-\bar w\|^2$,
\begin{equation}\label{eq:momentum_energy_v23}
 V_s^{\beta,\sigma}+U_s^{\beta,\sigma}
 =\frac1m\sum_e\!\left[\beta^2p\|c_e\|^2+\theta^2d\|a_e\|^2
       +2\beta\theta\langle c_e,a_e\rangle\right]
 +\frac{\theta^2d\sigma_X^2}{2}-\|\nu\|^2.
\end{equation}
\end{proposition}

\begin{proof}
Because $PD=D$, $\mathbb E P=pI$, and $\mathbb E D=I$, we have
\[
 \mathbb E(P-pI)^2=p(1-p)I,\qquad
 \mathbb E(D-I)^2=(d-1)I,\qquad
 \mathbb E[(P-pI)(D-I)]=(1-p)I.
\]
Expanding the squared norm of
$\beta(P-pI)c+\theta(D-I)a$ proves \eqref{eq:psi_v23} and its nonnegativity.
For G, both random maps are common, so the noiseless average is
$\beta P\bar c+\theta D\bar a$. For I, centered pair updates are independent;
the variance of their average is the sum of variances divided by $m^2$.
These observations give the noiseless parts of
\eqref{eq:momentum_variance_G_v23}--\eqref{eq:momentum_variance_I_v23}.

For S, condition first on the common $P$. The conditional mean of the noiseless
average is $P(\beta\bar c+\theta\bar a/p)$. Its variance is the first term of
\eqref{eq:momentum_variance_S_v23}. Conditional sign fluctuations are independent
across pairs. Their variance is
$\theta^2(d-1/p)m^{-2}\sum_e\|a_e\|^2$, by the same conditional calculation as in
Appendix \ref{appendix:correlation_moments}.

Oracle cross terms vanish when conditioned on all directions. In every coupling,
\[
 \mathbb E\left\|\frac{\theta}{N}\sum_i D_{e(i)}\epsilon_i\right\|^2
 =\frac{\theta^2}{N^2}\sum_i
     \mathbb E[\epsilon_i^\top D_{e(i)}^2\epsilon_i]
 =\frac{\theta^2d\sigma_X^2}{N}.
\]
The first equality uses independence across nodes even when their directions
coincide; the second uses independence from directions and $\mathbb E D_e^2=dI$.
The oracle component is also orthogonal in expectation to the noiseless update.

Finally,
\[
 \frac1m\sum_e\mathbb E\|w_e\|^2
 =\frac1m\sum_e[\beta^2p\|c_e\|^2+\theta^2d\|a_e\|^2
                   +2\beta\theta\langle c_e,a_e\rangle]
   +\frac{\theta^2d\sigma_X^2}{2}.
\]
Subtracting $\|\nu\|^2$ and using the pathwise average/disagreement identity
proves \eqref{eq:momentum_energy_v23}.
\end{proof}

In particular, at $\beta=0$, oracle noise adds $d\sigma_X^2/N$ to each average
variance and $d\sigma_X^2(1/2-1/N)$ to each injection term. Sharing directions
preserves cancellation of deterministic heterogeneity, not cancellation of
independent oracle errors.

The momentum comparison is equally explicit. If
$\mathcal H_\Psi=m^{-1}\sum_e\Psi_\beta(a_e-\bar a,c_e-\bar c)$, then
\begin{equation}\label{eq:momentum_crossover_v23}
 V_{\mathrm I}^{\beta,\sigma}-V_{\mathrm G}^{\beta,\sigma}
 =\frac{\mathcal H_\Psi-(m-1)\Psi_\beta(\bar a,\bar c)}{m}.
\end{equation}
The comparison conditions on the same states and stored memories; both evolve
differently along separate training trajectories.

Finite differences introduce a controlled perturbation. Let $\widetilde w_e$
be the actual applied pair update under Assumption \ref{assum:third}, with the
same oracle sample used in the two queries. Let
$\widetilde V_s=\mathbb E\|\overline{\widetilde w}
-\mathbb E\overline{\widetilde w}\|^2$. Then
\begin{equation}\label{eq:finite_difference_momentum_variance_v23}
 \|\mathbb E\overline{\widetilde w}-\nu\|\le\theta B_\mu,\qquad
 \left|\sqrt{\widetilde V_s}-\sqrt{V_s^{\beta,\sigma}}\right|
 \le\theta B_\mu.
\end{equation}
Indeed, Taylor's remainder bounds each local estimation error by $B_\mu$.
Pair and network averaging therefore bound the extra applied vector by
$\theta B_\mu$. Centering cannot increase its second moment, and the reverse
triangle inequality in $L^2$ gives the second inequality.
Stacking the actual applied pair vectors as rows of $\widetilde W$, pair constancy gives
$\mathcal A\widetilde W=\widetilde W$ also holds.

These formulas give the conditional moments for noisy sparse momentum with
finite differences. Theorem \ref{thm:momentum_convergence} supplies the
corresponding convergence guarantee.

\subsection{Extended Mechanism Experiments}
\label{appendix:extension_experiments_v23}

The follow-up study changes local curvature, oracle noise, and sparse momentum while preserving the matching and payload comparison. Its configuration and main-panel selection were fixed before training and archived with a timestamp and source hashes in \path{results/correlation_v23/protocol.json}. All 1536 runs are retained.

There are eight nodes and $d=32$ coordinates, with $q=4$, $\mu=10^{-3}$, and 1600 rounds. Write $h=\operatorname{geomspace}(1,4,d)$. For each seed, a Gaussian matrix is centered by coordinate and divided by its largest absolute entry to obtain $R$, with $N^{-1}\sum_iR_i=0$ and $|R_{ij}|\le1$. We set
\[
 H_i=\operatorname{diag}(h_j(1+\omega R_{ij}))_{j=1}^d,\qquad
 \omega\in\{0,0.6\}.
\]
Thus all Hessians are positive, their average is fixed, and $\omega=0.6$ introduces local curvature differences without changing the global objective. Gaussian linear terms are centered and scaled to $N^{-1}\sum_i\|b_i\|^2=\zeta^2$, with $\zeta\in\{0,2\}$. Here $\zeta$ describes the linear terms, not a uniform gradient-heterogeneity bound for unequal Hessians.

The oracle is $f_i(x;\epsilon_i)=F_i(x)+\epsilon_i^\top x$, with
$\epsilon_i\sim\mathcal N(0,\sigma^2 I_d/d)$. Samples are independent across nodes and rounds and reused within the two-point difference. They are paired across the compared couplings and parameters. The updates use the two function values, not analytic gradients; analytic quantities are recorded only for evaluation. Central differences are exact for these quadratic-plus-linear samples up to floating-point error.

For each coupling, we test complete-graph uniform perfect matchings and independently sampled ring phases, eight seeds, both curvature settings, both linear-heterogeneity settings, and base steps $\eta_0\in\{0.001,0.004\}$. Four update settings complete the grid: deterministic core $(\beta,\sigma,\eta)=(0,0,\eta_0)$; noisy core $(0,0.5,\eta_0)$; noisy sparse momentum $(0.9,0.5,\eta_0)$; and first-step-matched noisy momentum $(0.9,0.5,10\eta_0)$. The last choice makes $(1-\beta)\eta=\eta_0$, so its first update agrees with the noisy core for each realization when memory starts at zero. Subsequent updates need not agree.

Every run uses 3200 optimization queries and a nominal 204800 outgoing bits per node. Computation is float64; values are charged at 32 bits without quantizing or transmitting packets. All runs remained finite, with no clipping, early stopping, or seed replacement. The measured residual satisfied \eqref{eq:curvature_residual_v23} within numerical precision at every logged state.

For $\omega=0.6$, $\zeta=2$, and $\eta_0=0.004$, G remains below S and I at the same step under deterministic, noisy, and sparse-momentum updates (Figures \ref{fig:extension_complete_v23}--\ref{fig:extension_ring_v23}). On the complete graph, the deterministic final objective is $(2.96\pm1.13)\times10^{-5}$ for G and $(1.34\pm0.48)\times10^{-2}$ for I. On the ring the G residual is larger, $(1.43\pm0.49)\times10^{-4}$, consistent with the loss of topology-independent cancellation when curvatures differ.

The ordering also depends on step size. In the first-step-matched control at $\eta=0.04$, G has a much larger final objective than I on both graphs. At the smaller base step, its first-step-matched $\eta=0.01$ retains the advantage when $\zeta=2$. This dependence is compatible with \eqref{eq:momentum_crossover_v23}, which compares conditional moments at a common state and memory, rather than complete trajectories.

When $\zeta=0$, the large deterministic advantage disappears (Figures \ref{fig:extension_zero_complete_v23}--\ref{fig:extension_zero_ring_v23}), identifying heterogeneity as the source of the separation. Oracle noise raises the floor reached by G above its deterministic value.

The eight-seed means and sample standard deviations for every configuration are in \path{results/correlation_v23/final_summary.csv}, alongside raw trajectories, per-seed endpoints, and a validation audit. The archive contains objective and state-disagreement plots for all 16 combinations of graph, curvature, linear heterogeneity, and base step. The figures below include both momentum step-size choices.

\begin{figure}[!htbp]
 \centering\includegraphics[width=\linewidth]{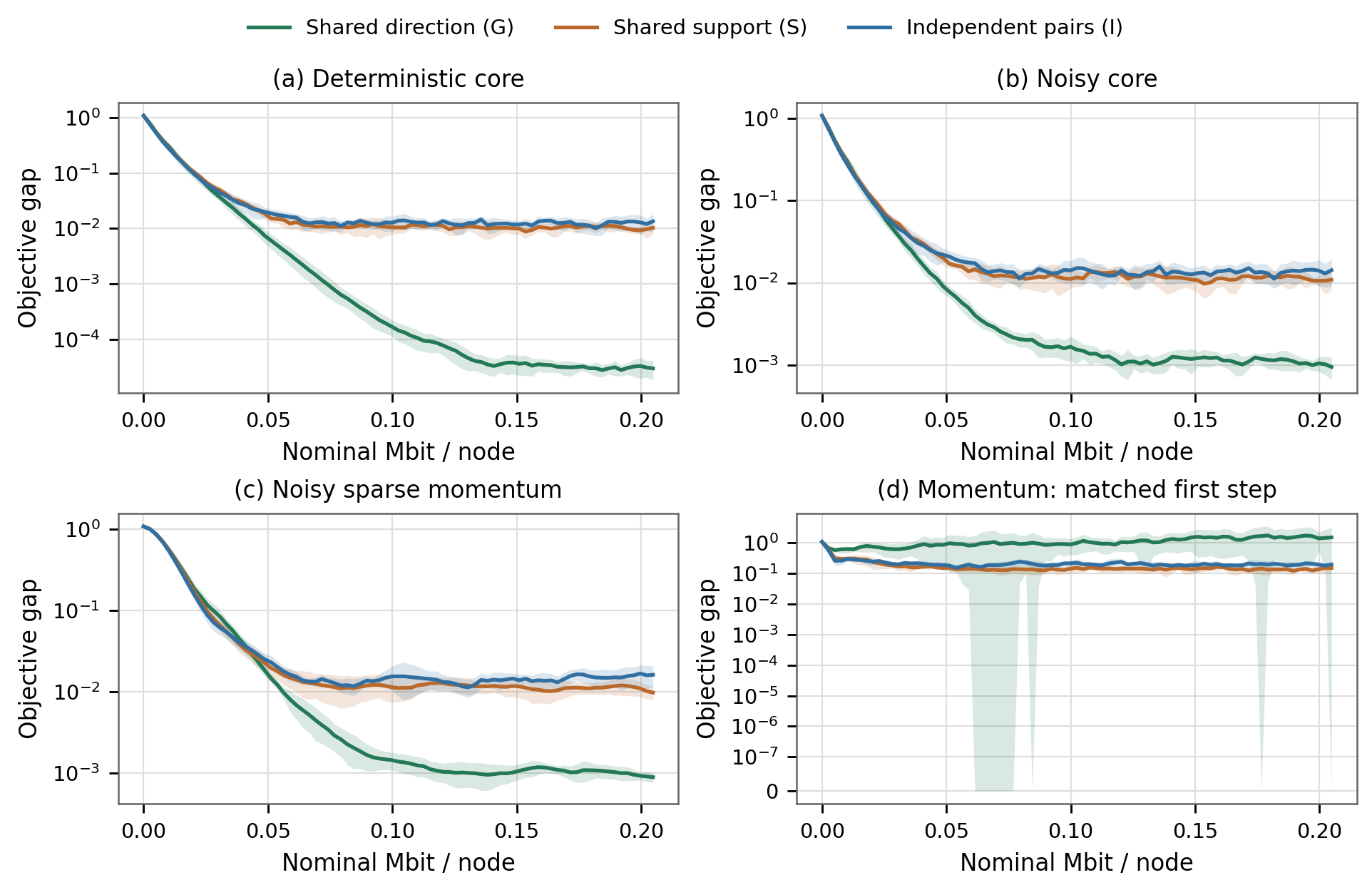}
 \caption{Unequal curvatures, $\zeta=2$, complete graph. All panels use eight seeds; (d) increases the momentum step from $0.004$ to $0.04$.}
 \label{fig:extension_complete_v23}
\end{figure}
\begin{figure}[!htbp]
 \centering\includegraphics[width=\linewidth]{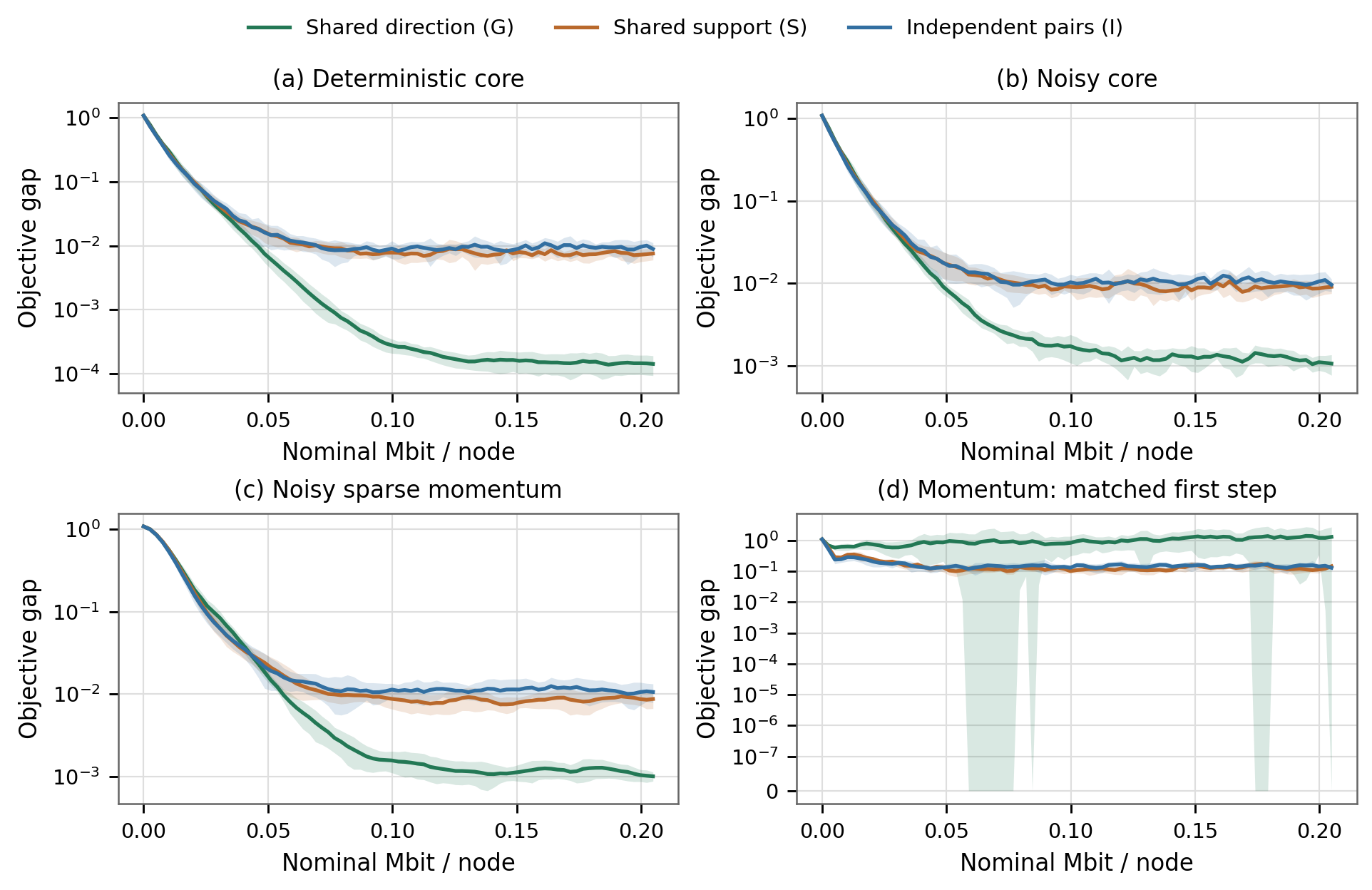}
 \caption{The same settings on iid ring matchings; (d) matches the first-update scale.}
 \label{fig:extension_ring_v23}
\end{figure}
\begin{figure}[!htbp]
 \centering\includegraphics[width=\linewidth]{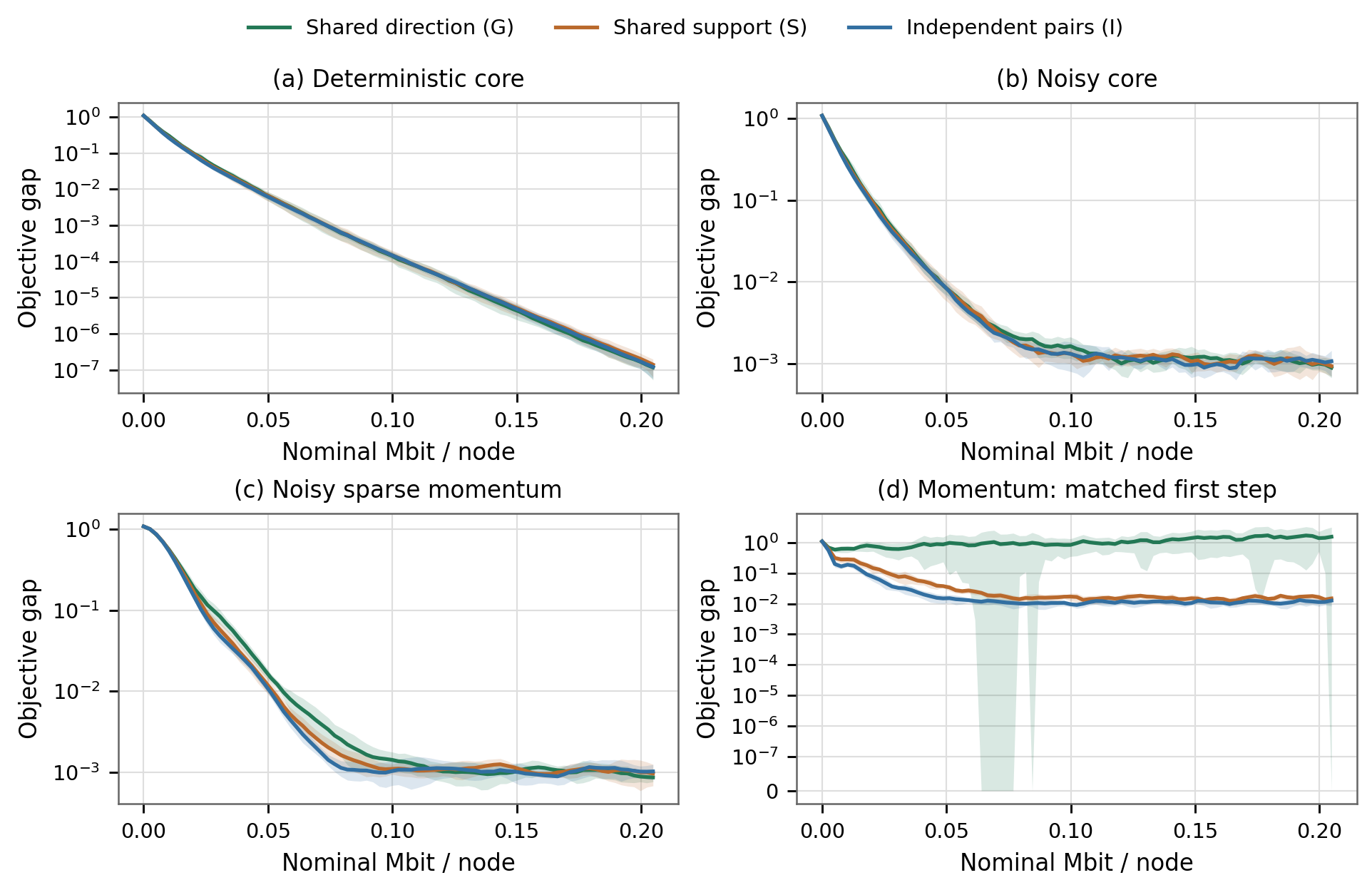}
 \caption{Unequal curvatures with zero linear heterogeneity on the complete graph. The pronounced heterogeneity-driven separation is absent.}
 \label{fig:extension_zero_complete_v23}
\end{figure}
\begin{figure}[!htbp]
 \centering\includegraphics[width=\linewidth]{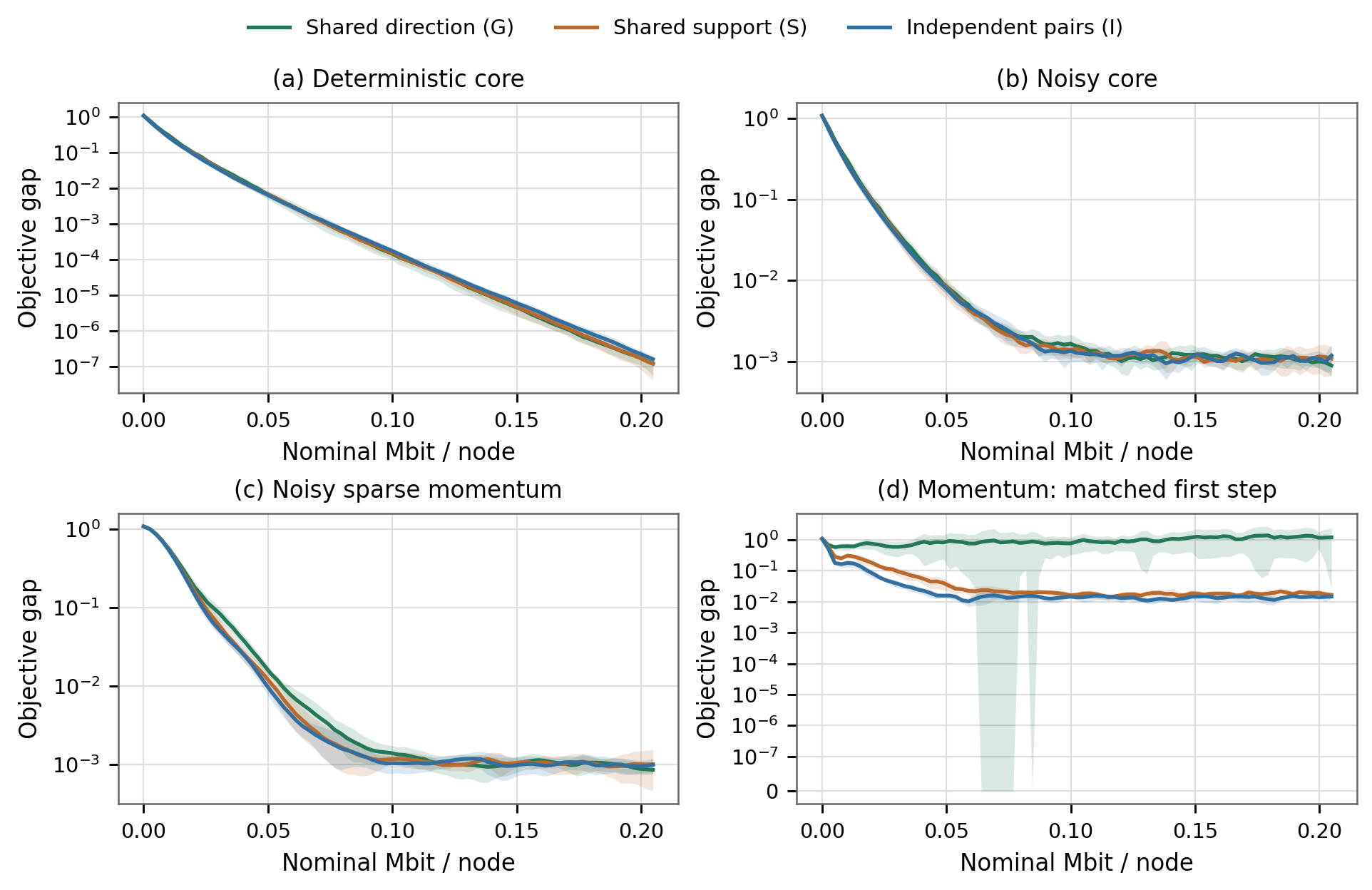}
 \caption{Zero linear heterogeneity on the ring, eight-seed mean $\pm$ std. Bands are truncated at zero for the nonnegative objective gap.}
 \label{fig:extension_zero_ring_v23}
\end{figure}

%% file: v28_qwen_correlation_experiment.tex
\paragraph{Fixed-matching controls.}
\label{appendix:qwen_long_correlation}
On the four-worker alternating ring, G shares support and signs across matched pairs, S shares support with independent edge signs, and I uses independent edge supports and signs. G/S retain identical support draws; S/I retain identical ordered edge-sign draws. All arms use the same five seeds, initialization, shards, $q=8192$, $\mu=10^{-3}$, and $\beta=0.9$. Each worker maintains float32 adapters on the frozen bfloat16 backbone. The terminal cost at $T=1540$ is $403.702$ Mbit and $3080$ optimization queries per worker, with exposure $qT/d=5$.

\paragraph{Scheduling at fixed query correlation.}
The secondary scheduling control compares matching G at round $100$ with all-neighbor global-support mixing at round $50$, both at $26.214$ Mbit per worker. They use the same query-direction generator and initialization. Label NLL falls from $0.6472\pm0.0231$ to $0.5902\pm0.0474$, a reduction of $0.0570$ (paired 95\% CI $[0.0172,0.1050]$). Accuracy is $64.58\pm5.40\%$ versus $60.18\pm1.86\%$ (paired $+4.40$ points, CI $[-0.01,9.42]$). At equal $100$ rounds, both intervals span zero. The same-budget NLL gain therefore persists with query correlation unchanged.

\paragraph{Long-horizon adaptation.}
All three couplings use $\eta=5\times10^{-5}$ for the first $100$ updates and $\eta=5\times10^{-6}$ thereafter, with momentum carried through the switch. The common schedule, five seeds, and terminal accuracy endpoint were fixed before training. We evaluate the full validation set at $100$, $308$, $912$, $924$, and $1540$ rounds and reuse the archived constant-step study as the paired reference. All $15$ new runs complete. Their round-$100$ accuracy, NLL, query statistics, disagreement, coverage, and payload reproduce the reference records exactly.

At $1540$ rounds, G, S, and I reach $74.07\pm2.58\%$, $77.82\pm2.80\%$, and $75.97\pm2.75\%$ accuracy, respectively. Their gains over round $100$ are $9.49$, $12.56$, and $12.38$ points, with paired intervals $[6.79,11.84]$, $[7.56,17.08]$, and $[7.87,16.89]$. Label NLL falls to $0.528\pm0.022$, $0.501\pm0.031$, and $0.528\pm0.028$. Mean disagreement RMS is $0.00826$, $0.00804$, and $0.00728$, while visited-coordinate coverage reaches $99.33$--$99.34\%$. Training thus remains effective after repeated coordinate activations, with lower disagreement than at round $100$.

The archived constant-step counterparts end at $50.12\pm0.69\%$, $49.72\pm0.23\%$, and $50.00\pm0.96\%$, with disagreement RMS $1.65$--$1.71$. The shared step reduction improves terminal accuracy by $23.95$, $28.10$, and $25.96$ points. Figure \ref{fig:qwen_matching_v27} shows all scheduled checkpoints of both studies.

\begin{figure}[htbp]
    \centering
    \includegraphics[width=0.98\linewidth]{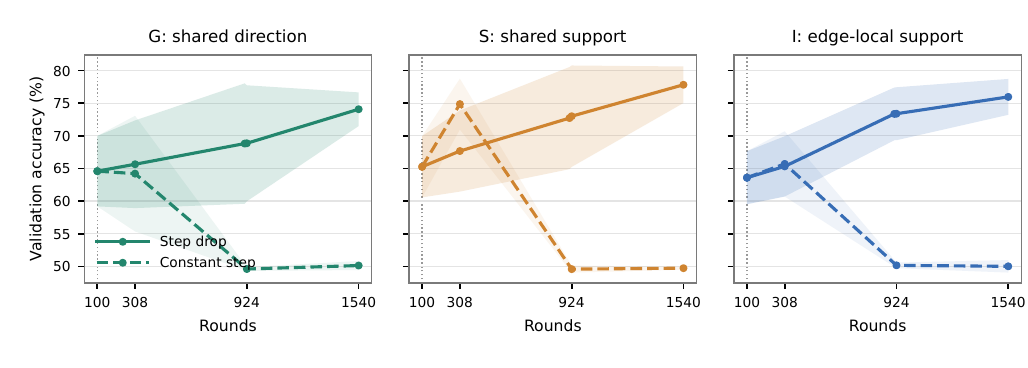}
    \caption{Long-horizon QNLI, five-seed mean $\pm$ std. Solid: tenfold step reduction after round $100$; dashed: constant step. Matching and random streams are paired.}
    \label{fig:qwen_matching_v27}
\end{figure}

\paragraph{Direction correlation at fixed support.}
With matching, support draws, query counts, and payload fixed, independent edge signs raise S above G by $3.75$ accuracy points (95\% paired CI $[0.24,6.79]$) and lower label NLL by $0.0272$ (CI $[0.0090,0.0485]$). The gain comes from cross-edge sign correlation under the same training schedule. Moving from S to I lowers disagreement RMS by $0.00077$ (CI $[0.00022,0.00131]$), while the accuracy difference is $-1.86$ points (CI $[-5.39,2.21]$). Thus the additional independent supports reduce disagreement without an established accuracy gain in this comparison. Intervals are descriptive paired estimates over five seeds, without multiple-comparison adjustment.

\paragraph{Long-horizon encoding comparison.}
At $403.702$ Mbit per worker, value-only messages afford $1540$ rounds, whereas $32$-bit values with $22$-bit coordinate indices afford $912$. Re-encoding the same edge-local trajectory gives $73.36\pm4.04\%$ accuracy at the indexed endpoint and $75.97\pm2.75\%$ at the value-only endpoint: $+2.61$ points (paired CI $[0.23,6.32]$). For G and S the corresponding differences are $+5.26$ (CI $[-0.05,10.57]$) and $+5.11$ points (CI $[-1.27,12.66]$). All three NLL intervals span zero. The optimizer trajectory is fixed; removing coordinate lists buys the additional rounds.